\documentclass[]{article}
\usepackage{amsmath,amssymb,amsthm,graphicx,subfigure,cite,rotating,multirow}
\pdfoutput=1
\usepackage[]{hyperref}
\hypersetup{pdftitle={Adaptive diffusion constrained total variation scheme with application to `cartoon + texture + edge' image decomposition},
colorlinks=true,citecolor=blue,linkcolor=red}

\newtheorem{definition}{Definition}
\newtheorem{theorem}{Theorem}
\newtheorem{lemma}{Lemma}
\newtheorem{remark}{Remark}

 \newcommand{\abs}[1]{\left\vert #1 \right\vert}
 
  \DeclareMathOperator{\divo}{div}
  \def\be{\begin{equation}}
\def\ee{\end{equation}}
\newcommand {\R} {\mathbb{R}}

 \graphicspath{ {./finalfigures_compressed/} } 
 
\begin{document}

\title{Adaptive diffusion constrained total variation scheme with application to `cartoon + texture + edge' image decomposition\thanks{An earlier version available as Preprint Number 13-54, Department of Mathematics, Universidade de Coimbra, Portugal. \href{http://www.mat.uc.pt/preprints/ps/p1354.pdf}{http://www.mat.uc.pt/preprints/ps/p1354.pdf}}}

\author{Juan C. Moreno\thanks{Corresponding author. IT, Department of Computer Science, University of Beira Interior, 6201--001, Covilh\~{a}, Portugal. E-mail: jmoreno@ubi.pt.} \and V. B. Surya Prasath\thanks{Department of Computer Science, University of Missouri-Columbia, MO 65211 USA. E-mail: prasaths@missouri.edu}\and D. Vorotnikov\thanks{Department of Mathematics, University of Coimbra, Portugal} \and  Hugo Proen\c{c}a\thanks{IT, Department of Computer Science, University of Beira Interior, 6201--001, Covilh\~{a}, Portugal.} \and K. Palaniappan\thanks{Department of Computer Science, University of Missouri-Columbia, MO 65211 USA.}} 

\date{}

\maketitle
\begin{abstract}
We consider an image decomposition model involving a variational (minimization) problem and an evolutionary partial differential equation (PDE). We utilize a linear inhomogenuous diffusion constrained and weighted total variation (TV) scheme for image adaptive decomposition. An adaptive weight along with TV regularization splits a given image into three components representing the geometrical (cartoon), textural (small scale - microtextures), and edges (big scale - macrotextures). We study the wellposedness of the coupled variational-PDE scheme along with an efficient numerical scheme based on Chambolle's dual minimization~\cite{Ch04}.  We provide extensive experimental results in cartoon-texture-edges decomposition, and denoising as well compare with other related variational, coupled anisotropic diffusion PDE based methods.

\end{abstract}

\section{Introduction}\label{sec:intro}

Decomposing an image into meaningful components is an important and challenging inverse problem in image processing. Image denoising is a very well known example of image decomposition. In such a decomposition, the given image is assumed to be under the influence of noise, and the main purpose is to remove noise without destroying edges. This denoising task can be regarded as a decomposition of the image into noise-free signal and noise part. There exist various methodologies for image restoration, where variational minimization and partial differential equation (PDE) are two of the most popular ones~\cite{AK06}. 

Another important example of image decomposition is based on its smooth and texture components using the total variation (TV) regularization method. This was first studied by Rudin et al~\cite{RO92} for image restoration. The TV regularization can be written as an unconstrained minimization,
\begin{eqnarray}\label{E:tvreg}
\min_u E_{TV}(u) = \int_{\Omega} \abs{\nabla u}\,dx +
\frac{1}{2\alpha}\,\int_{\Omega}\abs{u-f}^2\,dx.
\end{eqnarray}
The parameter $\alpha>0$ balances the fidelity term with respect to the TV regularization. Let the given image be written as $f = u+v$, where the function $u$ models well structured homogeneous regions (cartoon) and $v$ defines oscillating patterns such as noise and texture. 
~\cite{Meyer01} established the scale connection property of TV for image
decomposition, namely that the $\alpha$  parameter is related to the scale of objects in the image. In particular, Meyer proved that if $f$ is a characteristic function and is sufficiently small with respect to a norm ($\abs{f}_* \leq 1/2\alpha$), then the minimization of the TV regularization~\eqref{E:tvreg} gives $u=0$ and $f=0$, which is counter-intuitive since one expects $u=f$ and $v=0$. Thus, Meyer proposed to use dual of the closure in the bounded variation ($BV$) space of the Schwartz class for non-trivial solutions, we refer to~\cite{Meyer01} for more details. This crucial fact has been exploited by~\cite{VOsher03} to obtain numerical approximations to the Meyer decomposition model, see also related early works~\cite{AABChambolle05,AujolGilboa06}.  

Image denoising methods in general and variational PDE models in particular provide natural decomposition of images, see~\cite{AK06} for a review. Image smoothing and denoising literature is rich in terms of different methodologies utilized and current state of the art techniques~\cite{BuadesColl06,DEgiazarian07,PizarroMrazek10} have pushed the envelope of optimal improvements~\cite{ChatterjeeTIP2010}. Next we provide a brief overview different approaches utilized in image decomposition.

\begin{figure*}
\centering
	\subfigure[Input]{\includegraphics[width=3.8cm, height=3.8cm]{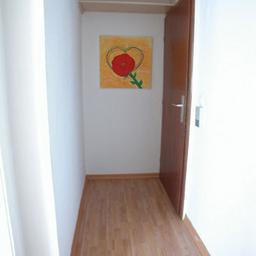}}
	\subfigure[Cartoon ($u$)]{\includegraphics[width=3.8cm, height=3.8cm]{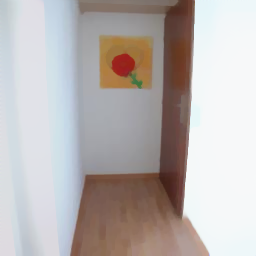}}
	\subfigure[Texture ($v$)]{\includegraphics[width=3.8cm, height=3.8cm]{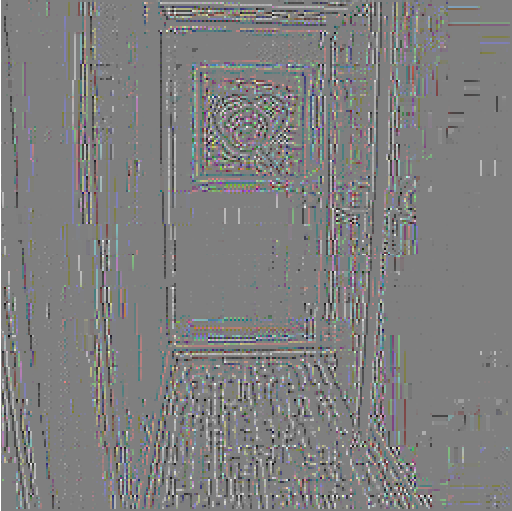}}
	\subfigure[(Pseudo) Edges ($w$)]{\includegraphics[width=3.8cm, height=3.8cm]{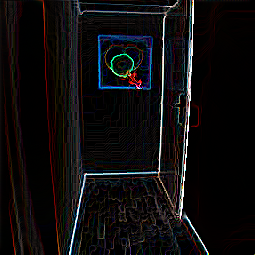}}

	\caption{MIT LabelMe decomposition result on an indoor image. Our scheme can obtain piecewise constant cartoon component along with texture and edges for better separation of basic shape elements present in a scene. Better viewed online and zoomed in.}\label{fig:label}
\end{figure*}

\subsection{Literature}\label{ssec:lit}

While the literature on cartoon and texture decomposition is extensive by now, three lines of inquiries are closely related to the work presented here.
\begin{itemize}
\item \textbf{Different function spaces for modeling the textures and discrete approximations:}
Following Meyer's seminal work~\cite{Meyer01}, various authors have considered different function spaces to model textures accurately~\cite{AujolChambolle05,LVese05,AujolGilboa06,OSVese03,GLMVese07,Haddad07,LLVese08,LLVese09,THe13}. 

\item \textbf{$L^1$ fidelity based TV regularization models:}
In another related direction the fidelity term can be made $L^1$ and is proven to provide contrast preserving restorations. We refer to~\cite{CE05} for the geometric motivation and fine properties of $L^1$ fidelity term, see also~\cite{DAGousseau09}. Applications of $L^1$-TV models for cartoon plus texture decomposition are considered as well~\cite{Ni04b,YGOsher05,YGOsher06,Haddad07,DAVese10,PTadmor11}.

\item \textbf{Different Regularizers instead of TV:}
The well-known staircasing property of the TV regularization has been analyzed by many in the past~\cite{NI04,CasellesC10,CasellesC11} and various modifications have also been studied. To avoid staircasing and blocky restoration results from the classical TV model there have been a variety of methods studied previously. Weighted and adaptive~\cite{SC96,St97,PSc10,PSd10,PSe12}, nonlocal-TV~\cite{LMVese10}, and higher order~\cite{CMMulet00,BP10,HJacob12,PSchonlieb13}.
\end{itemize}

Out of other related decomposition models we mention multi-scale parameter based models~\cite{TadmorNezzar04,Gilles12,THe13} which progressively capture texture. Extension to multichannel images in general, RGB images in particular, is also an important area of research~\cite{BC08}. Note that all of the above mentioned methods obtain cartoon and texture decomposition and in this work we obtain edges as part of our scheme. In Figure~\ref{fig:label} we show an example decomposition for an indoor image taken from the MIT LabelMe~\cite{Labelme07} open annotation. Our advanced decomposition provides basic scene elements, since it is based on total variation (edge preserving) smoothing, the cartoon component reflects coarse shapes. This feature is useful in obtaining segmentations and annotations as the edges can provide guideline for separating salient objects~\cite{Scene_labelme01}. 

\subsection{Contributions}\label{ssec:contrib}

In this paper, we propose a new image decomposition scheme which splits a given image into its geometrical and textural part along with an edge capturing diffusion constraint. Following the success of weighted and adaptive TV models, our scheme is based on a weighted TV regularization where the edge-variable-based weight is computed in a data-adaptive way. The scheme is implemented using the splitting method of~\cite{BE07} along with dual minimization scheme for the weighted TV regularization~\cite{Ch04}. As a by-product of the implementation via dual minimization we obtain an auxiliary variable which is akin to textural component of the given image. Thus, the scheme studied here provides a cartoon, texture, edge (CTE) decomposition for digital images, see Figure~\ref{fig:teaser} for some examples\footnote{We use the words edges or pseudo-edges interchangeably in the text, since the edge component is computed with gradient maps without traditional edge refinement procedures such as maxima suppression~\cite{CA86}.}. 

We consider the color image decomposition using the dual minimization based color TV model. Multi-scale decomposition following the recent work of~\cite{THe13} is also given. Moreover, we provide theoretical analysis of the model with a priori estimates and prove its solvability. Extensive experimental results on synthetic, textured real images are given. Further, illustrative comparison with other state-of-the-art models are undertaken and the results indicate superior performance of our scheme with respect to cartoon, texture separation as well as denoising with edge preservation.

\begin{figure*}
\centering
	\includegraphics[width=3.cm]{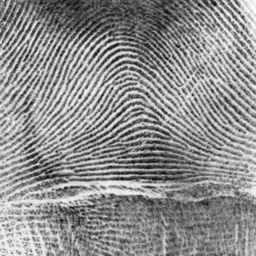}
	\includegraphics[width=3.cm]{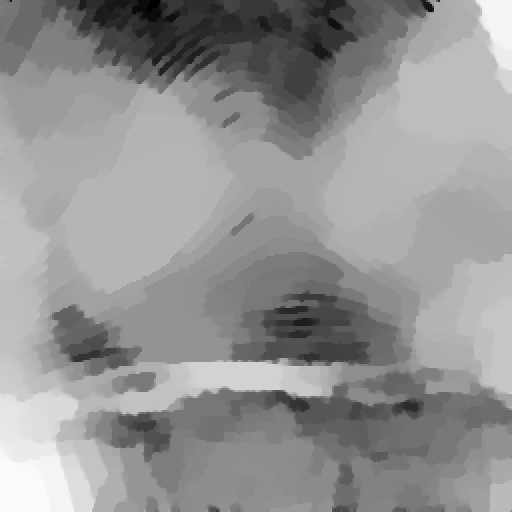}
	\includegraphics[width=3.cm]{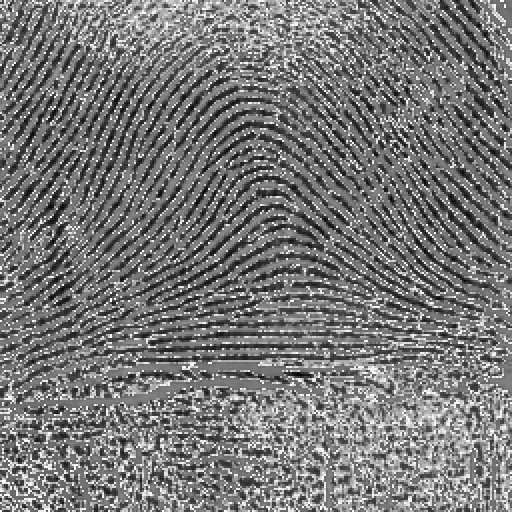}
	\includegraphics[width=3.cm]{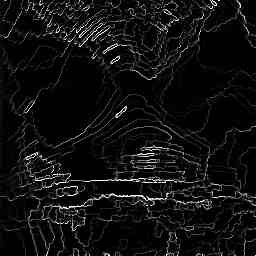}
	\includegraphics[width=3.cm]{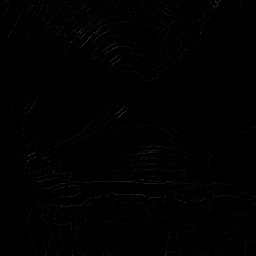}\\
	\includegraphics[width=3.cm, height=3.cm]{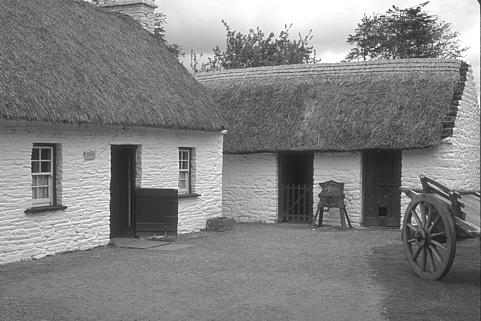}
	\includegraphics[width=3.cm, height=3.cm]{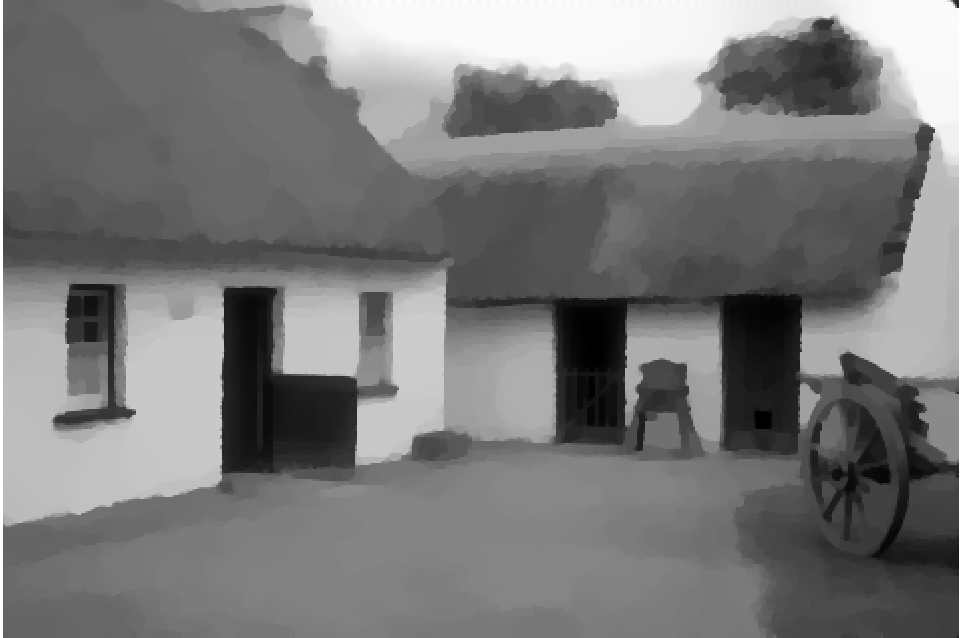}
	\includegraphics[width=3.cm, height=3.cm]{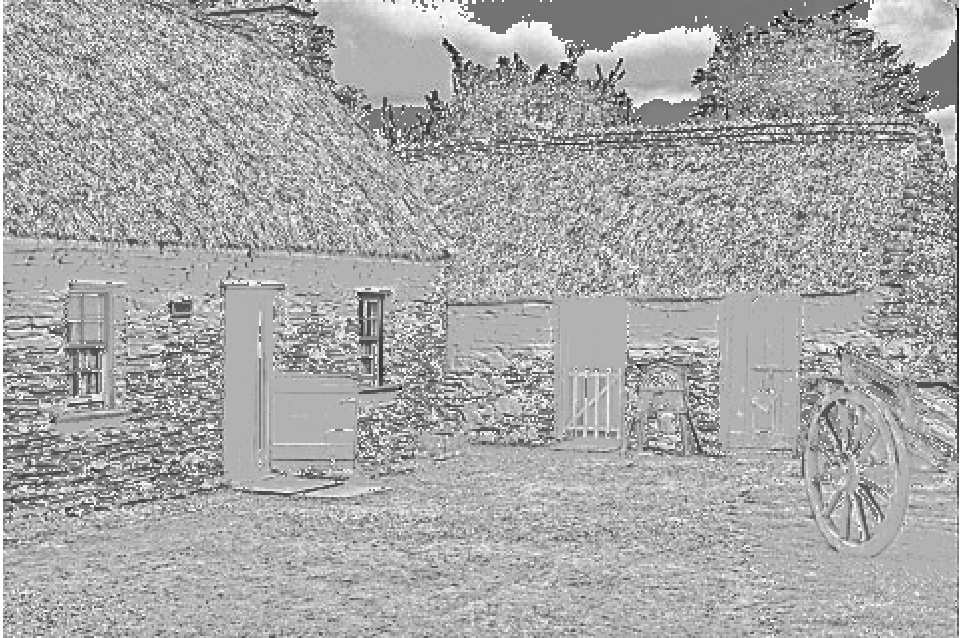}
	\includegraphics[width=3.cm, height=3.cm]{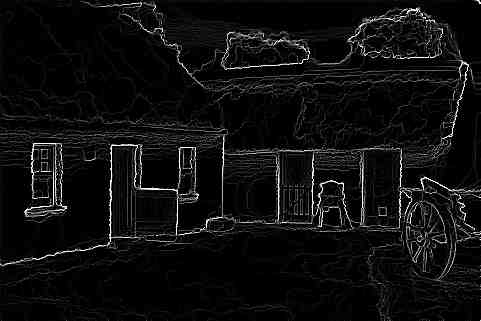}
	\includegraphics[width=3.cm,height=3.cm]{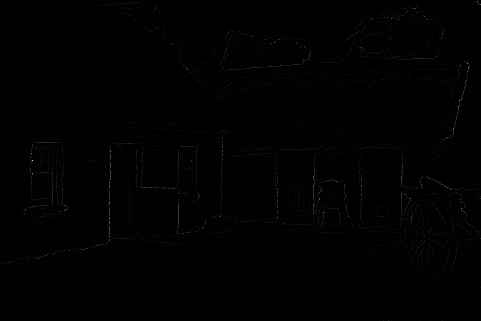}\\
	\includegraphics[width=3.cm]{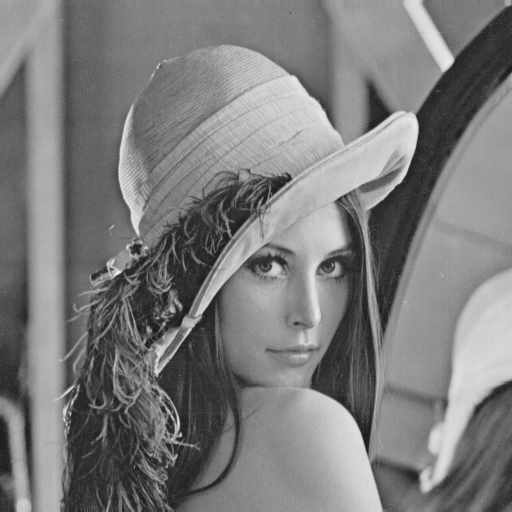}
	\includegraphics[width=3.cm]{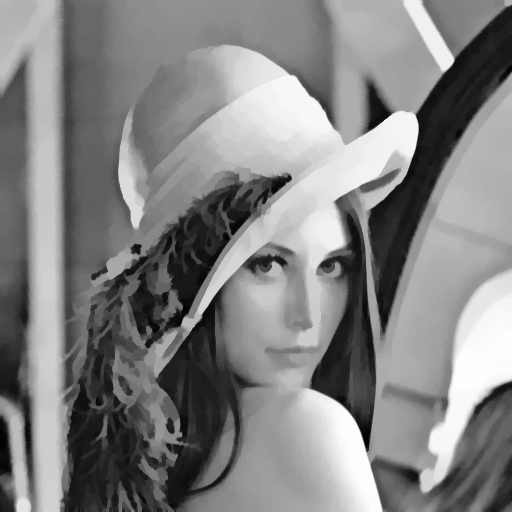}
	\includegraphics[width=3.cm]{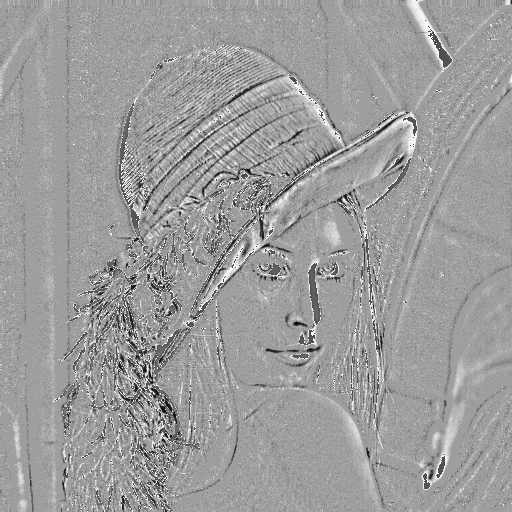}
	\includegraphics[width=3.cm]{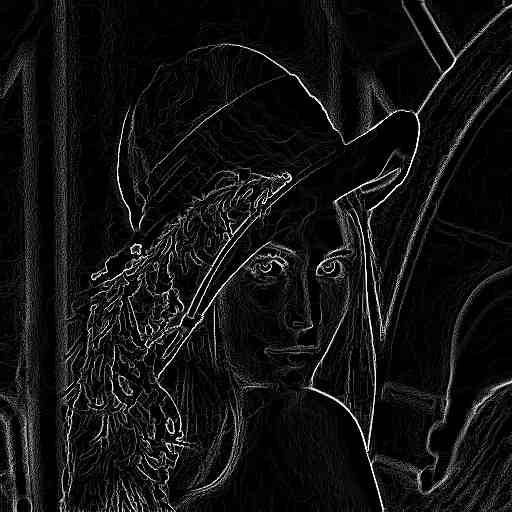}
	\includegraphics[width=3.cm]{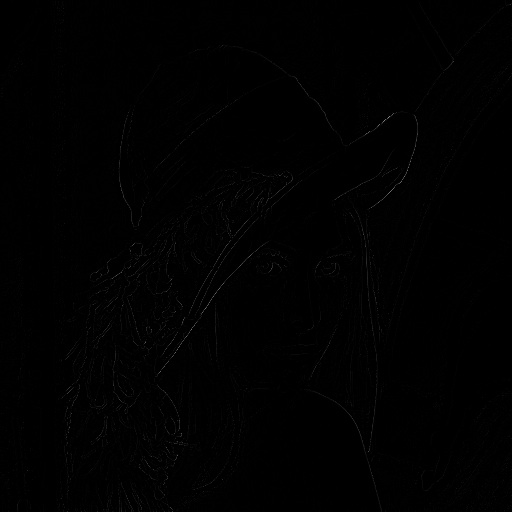}\\
	\includegraphics[width=3.cm]{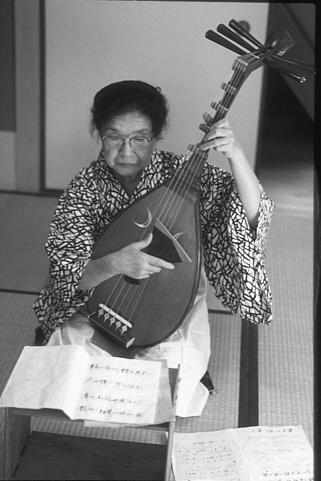}
	\includegraphics[width=3.cm]{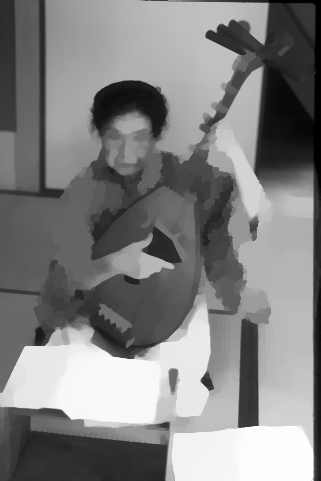}
	\includegraphics[width=3.cm]{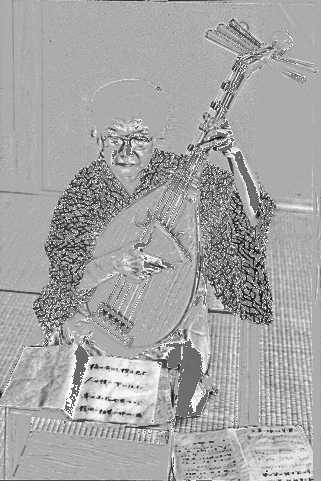}
	\includegraphics[width=3.cm]{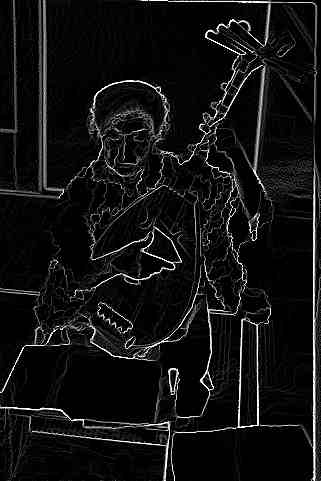}
	\includegraphics[width=3.cm]{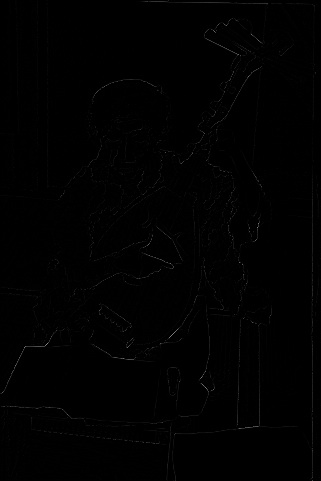}
	\subfigure[Input ($f$)]{\includegraphics[width=3.cm]{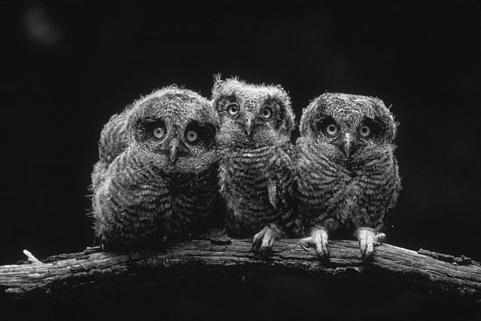}}
	\subfigure[Cartoon ($u$)]{\includegraphics[width=3.cm]{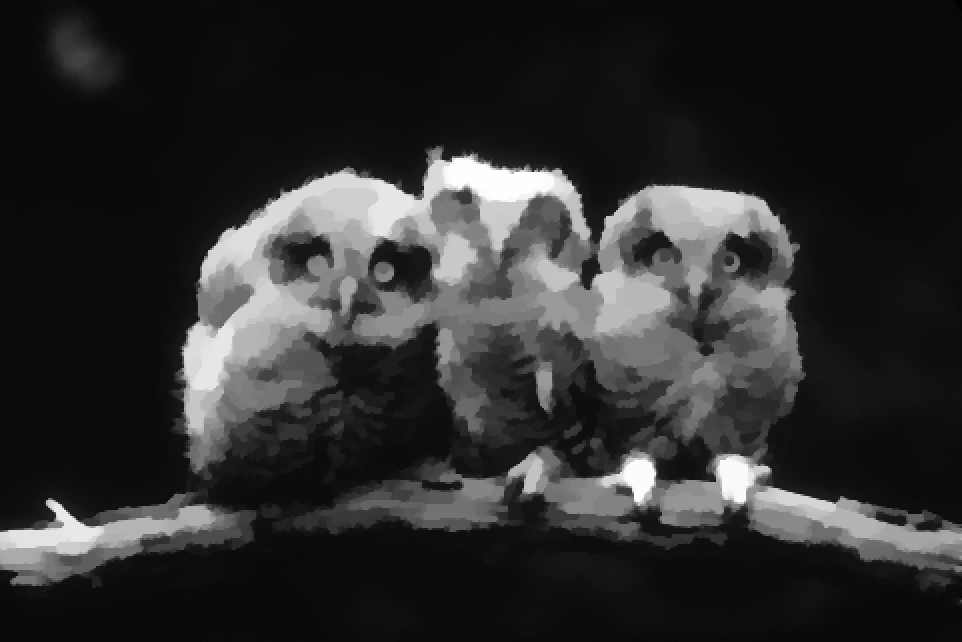}}
	\subfigure[Texture $v$]{\includegraphics[width=3.cm]{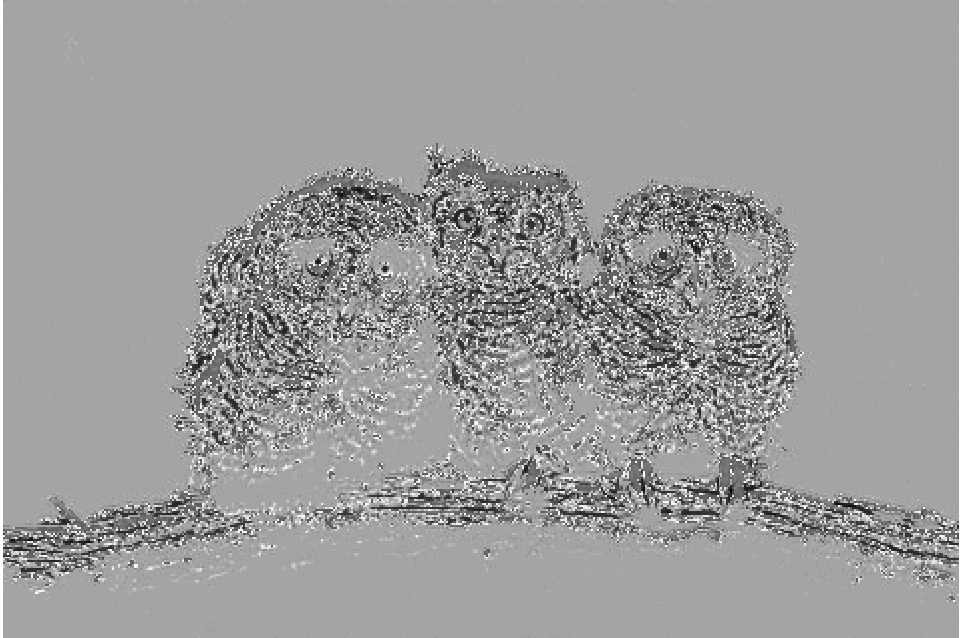}}
	\subfigure[(Pseudo) Edges ($w$)]{\includegraphics[width=3.cm]{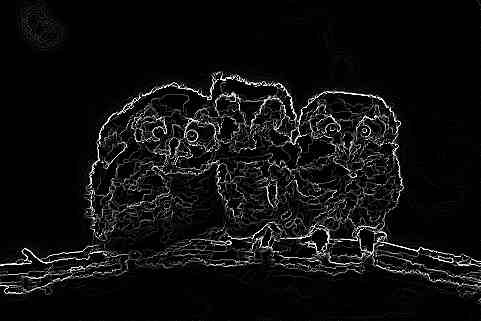}}
	\subfigure[$f-(u+v+w)$]{\includegraphics[width=3.cm]{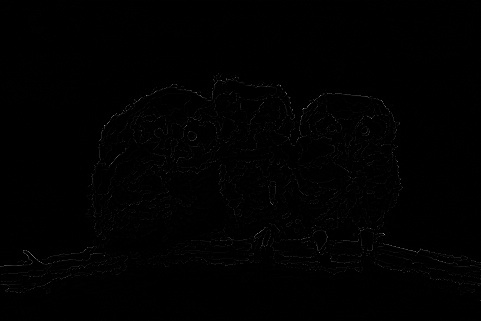}}
	\caption{The proposed coupled TV regularization with linear diffusion PDE model provides cartoon, texture and edge decomposition of images (stopping parameter $\epsilon=10^{-4}$). Better viewed online and zoomed in.}\label{fig:teaser}
\end{figure*}

The rest of the paper is organized as follows. Section~\ref{sec:diff} introduces the adaptive TV regularization coupled with a diffusion PDE.
Section~\ref{sec:theo} provides wellposedness results for our coupled model. In Section~\ref{sec:exper} we provide the experimental results conducted on real and synthetic images and comparison with related schemes. Finally, Section~\ref{sec:disc} concludes the paper.

\section{Diffusion constrained regularization}\label{sec:diff}

\subsection{Weighted total variation minimization}\label{ssec:wtv}

The total variation based regularization~\cite{RO92} given in Eqn.~\eqref{E:tvreg} is well-known in edge preserving image restoration, we rewrite it as follows,
\begin{eqnarray}\label{E:tv}
\min_u E^{\mu}_{TV}(u) = \int_{\Omega} \abs{\nabla u}\,dx + \mu\,\int_{\Omega}\abs{u-f}^2\,dx
\end{eqnarray}
where now $\mu>0$ is the image fidelity parameter which is important in obtaining results in denoising and cartoon+texture decomposition. A related approach is to consider a weighted total variation,
\begin{eqnarray}\label{E:wtvreg}
\min_u E^{\mu}_{gTV}(u)= \int_{\Omega} g(x,u,\nabla u)\abs{\nabla u}\,dx + \mu\,\int_{\Omega}\abs{u-f}^2\,dx
\end{eqnarray}
where $g(x,u,\nabla u)$ represents the generalized weight function. For example, 
~\cite{BE07} have considered a convex regularization
of the variational model
\begin{equation*}
\min\limits_{u} \left\{ \int_{\Omega}g(x)|\nabla u|\,dx + \mu\int_{\Omega}|u-f|\,dx \right\}
\end{equation*}
using a fast minimization based on a dual formulation to get a
partition in the geometric and texture information. Note that the image fidelity is changed to $L^1$ norm, we refer to~\cite{CE05} for more details. The convex regularization version is considered in~\cite{BE07},
\begin{eqnarray}\label{E:bresuv}
\min\limits_{u,v}\Bigg\{\int_{\Omega}g(x)|\nabla u|\,dx
+ \frac{1}{2\theta}\int_{\Omega}(u+v-f)^2\,dx+ \mu\int_{\Omega}|v|\,dx\Bigg\},
\end{eqnarray}
where the parameter $\theta>0$ is chosen to be small so that $f$ almost satisfies
$f\sim u+v$, with the function $u$ representing geometric information, {\it i.e.} the piecewise-smooth regions, and  function $v$ captures texture information lying in the given image. The function $g$ is an edge indicator function that vanishes at object boundaries, for example,
\begin{equation*}
g(x):= \displaystyle\frac{1}{1+\beta \left|\nabla f(x)\right|^{2}},
\end{equation*}
where $f$ is the original image and $\beta$ is an arbitrary
positive constant. Thus, we see that TV based minimization models naturally lead to cartoon and texture decomposition of digital images. The image fidelity parameter $\mu$ can be made data adaptive to obtain texture preserving restorations, see~\cite{GiS06}.
\subsection{Description of the model}\label{sec:desc}
\begin{figure*}[t]
\centering
\[\begin{array}{c}
\includegraphics[width=2.4cm]{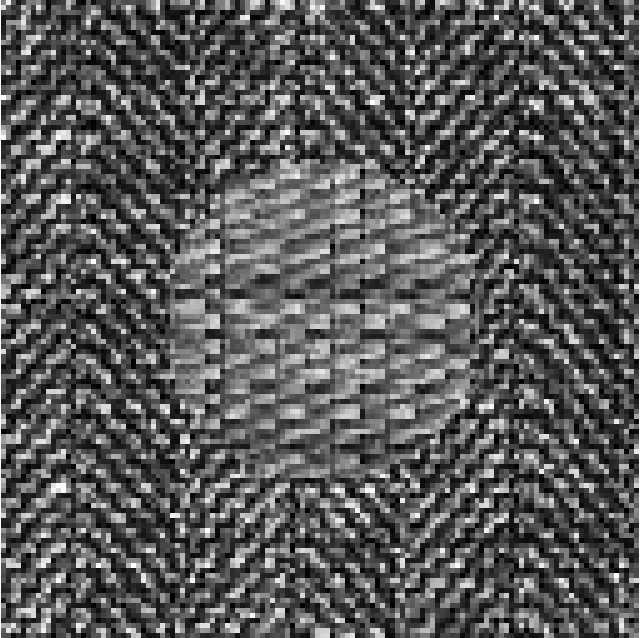}\quad
\includegraphics[width=3cm]{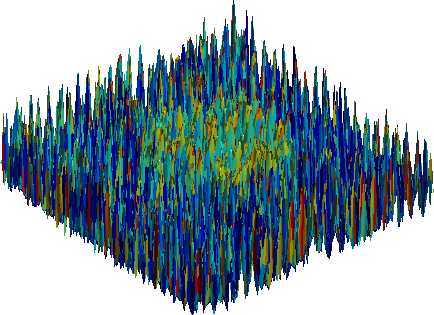}\\
\mbox{\scriptsize{(a) Original Image }}\\
\includegraphics[width=3cm]{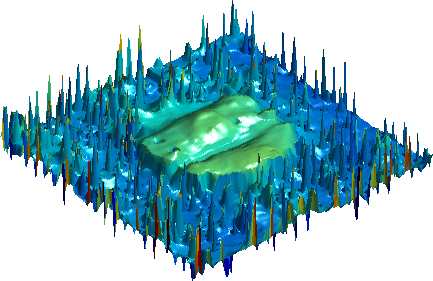}
\includegraphics[width=3cm]{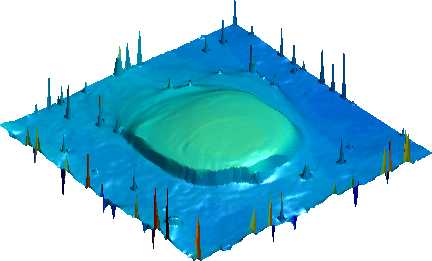}
\includegraphics[width=3cm]{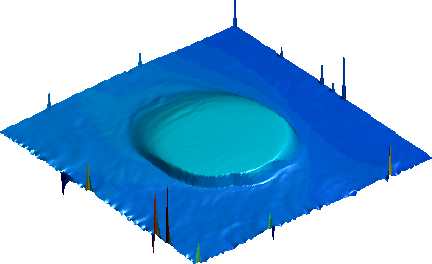}
\includegraphics[width=3cm]{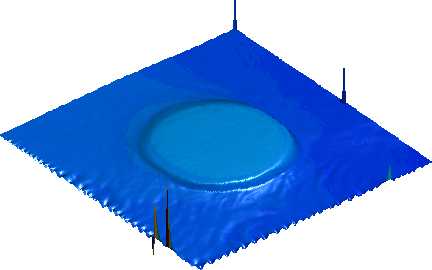}\\
\mbox{\scriptsize{(b) Cartoon component $u$ with stooping parameters $\epsilon=10^{-4}$, $\epsilon=10^{-5}$, $\epsilon=10^{-6}$ and $\epsilon=10^{-7}$ using the proposed model }}\\
\includegraphics[width=3cm]{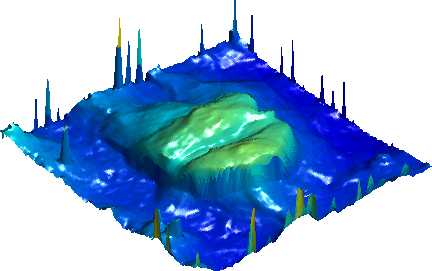}
\includegraphics[width=3cm]{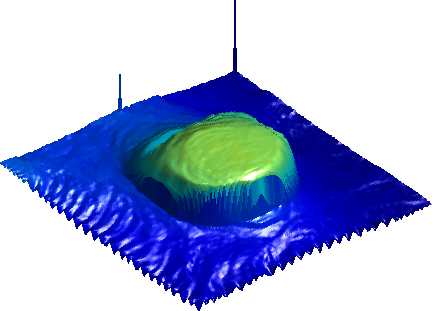}
\includegraphics[width=3cm]{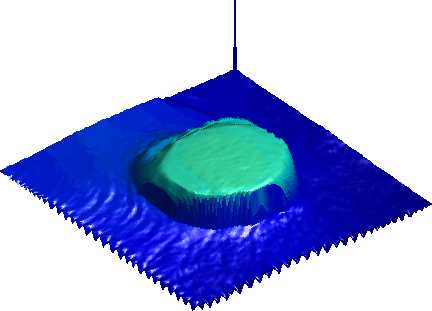}
\includegraphics[width=3cm]{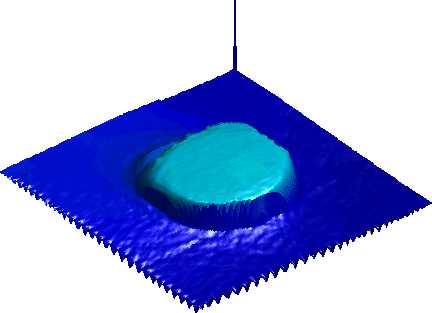}\\
\mbox{\scriptsize{(c) Cartoon component $u$ with stooping parameters $\epsilon=10^{-4}$, $\epsilon=10^{-5}$, $\epsilon=10^{-6}$ and $\epsilon=10^{-7}$ using the model in~\cite{BE07}}} 
\end{array}\]
\caption{(Color online) Our adaptive diffusion constrained total variation scheme (see Eqn.~\eqref{E:wTVsplit}) with constant $\mu$, $\lambda$ (second row) provides better edge preserving image decomposition when compared to the traditional TV regularization model (Eqn.~\eqref{E:bresuv}) of~\cite{BE07} (third row) as stoping parameters increase from $\epsilon=10^{-4}$ to $\epsilon=10^{-7}$. The proposed scheme keeps the structure without diffusing its boundary with the background.}\label{fig:CTE_Vs_Bresson_Synthetic}
\end{figure*}

In our work, we consider the following regularization model which was motivated by a coupled PDE modeling done in~\cite{PVorotnikov12} for image restoration,
\begin{eqnarray}\label{E:wTVpfidelity}
\min\limits_{u}\left\{\int_{\Omega}\phi(x,u,|\nabla u|)\,dx + \mu\int_{\Omega}|u-f|\,dx\right\},
\end{eqnarray}
\begin{eqnarray}\label{E:constr}
\frac{\partial w}{\partial t} = \lambda div(\nabla w) +
(1-\lambda)(|\nabla u|-w).
\end{eqnarray}
The choice of regularizer $\phi$ depends on an application area and among a plethora of convex and non-convex functions available with the classical TV~\cite{RO92} and the non-local TV~\cite{LMVese10}.
Motivated from the above discussions in Section~\ref{ssec:wtv}, and success enjoyed by the weighted $L^1$-TV
regularization model in image denoising and segmentation, we use $L^1$-TV regularizer model as a prime example to illustrate our model here. The proposed CTE model thus consists of a minimization along with a non-homogeneous diffusion equation,
\begin{eqnarray}\label{E:wTV}
\min\limits_{u}\left\{\int_{\Omega}g(w)|\nabla u|\,dx + \int_{\Omega}\mu(x)\,|u-f|\,dx\right\},
\end{eqnarray}
\begin{eqnarray}\label{E:constraint}
\frac{\partial w}{\partial t} = \lambda(x) div(\nabla w) + (1-\lambda(x))(|\nabla u|-w),
\end{eqnarray}
where $g(w)=\frac{1}{1+w^{2}}$, or $g(w) = exp(-w^2)$ (Perona-Malik type
diffusion functions~\cite{PM90}), or $g(w)=\abs{w}^{-1}$, or $g(w) = \frac{1}{\sqrt{\epsilon^2+\abs{w}^2}}$ (total
variation diffusion function~\cite{RO92}). That is we solve adaptive data fidelity based weighted total
variation minimization for the smooth part $u$ using Eqn.~\eqref{E:wTV} along with a linear
non-homogenous diffusion constraint on $w$ by solving Eqn.~\eqref{E:constraint}. Note that the balancing parameter $\lambda$ and image fidelity $\mu$ taking values in $[0,1]$ are important in our experimental results. Adaptive ways of choosing these parameters are explained below in Section~\ref{ssec:adap}. 
Following~\cite{BE07} we use a splitting with an auxiliary variable $v$ to obtain
\begin{eqnarray}\label{E:wTVsplit}
\min\limits_{u,v}\Bigg\{\int_{\Omega}g(w)|\nabla u|\,dx +
\frac{1}{2\theta}\int_{\Omega}(u+v-f)^2\,dx + \int_{\Omega}\mu(x)|v|\,dx\Bigg\},
\end{eqnarray}
\begin{eqnarray*}
\frac{\partial w}{\partial t} = \lambda(x) div(\nabla w) + (1-\lambda(x))(|\nabla u|-w).
\end{eqnarray*}
Thus, the computed solution of these equations provides a  representation $(u, v, w)$, where the function $u$ represents the geometric information, the function $v$ captures the texture information, and the function $w$ represents the edges lying in the given image. Figure~\ref{fig:CTE_Vs_Bresson_Synthetic} shows a comparison of our scheme ($\mu=1$ and $\lambda=0.5$) and Bresson et al~\cite{BE07} scheme Eqn.~\eqref{E:bresuv} for a synthetic texture image which contains two different texture patterns. As can be seen, our scheme (Figure~\ref{fig:CTE_Vs_Bresson_Synthetic}(b)) retains the cartoon edges better without diffusing the boundary and the shape is preserved in contrast to Bresson et al's result (Figure~\ref{fig:CTE_Vs_Bresson_Synthetic}(c)). 

The above coupled system is solved in an alternating iterative way for all the variables ($u,v,w$) involved and Chambolle's dual minimization scheme~\cite{Ch04} is used for the weighted TV minimization step. We start with the initial conditions \[(u,v,w)|_{n=0} = (f,\mathbf{0},\mathbf{1})\] and use the following steps to compute CTE components:
\begin{enumerate}
\item Solving the linear diffusion PDE~\eqref{E:constraint} for $w$ with $(u,v)$ fixed:
\begin{eqnarray}\label{E:wieqn}
w^{n+1}
= w^{n} + \frac{\delta t}{(\delta x)^2}(\lambda(x)\tilde \Delta w^{n} + (1-\lambda(x))\left(\abs{\nabla u} - w^{n}\right)),
\end{eqnarray}
where $\delta x$ is spatial discretization step (natural pixel grid), $\tilde \Delta$ is the standard finite difference discretization for the Laplacian and $\delta t$ is the step size.

\item Solving for the cartoon component $u$ with ($v,w$) fixed:

The minimization problem in $u$ is given by (see Eqn.~\eqref{E:wTVsplit}),
\begin{equation}\label{E:uieqn}
\min_{u}\left\{\int_{\Omega}\,g(w)|\nabla u |\, dx+
\frac{1}{2\theta}\int_{\Omega}(u+v-f)^2\,dx\right\}.
\end{equation}
The solution of \eqref{E:uieqn} is given by 
\begin{equation}\label{E:uisoln}
u=f-v - \theta div\, \mathbf{p},
\end{equation} 
where $\mathbf{p}=(p_{1},p_{2})$ satisfies $g(w) \nabla(\theta\, div\, \mathbf{p} -(f- v))-
|\nabla(\theta div\, \mathbf{p} - (f- v)) |\mathbf{p}=0$, which is solved using a fixed point method: $\mathbf{p}^{0} = 0$ and
\begin{equation*}\label{E:psoln}
\mathbf{p}^{n+1} = \frac{\mathbf{p}^{n}+\delta t \,\nabla(div(\mathbf{p}^{n}) -
(f-v)/\theta) }{1+\displaystyle\frac{\delta t}{g(w)} \,|\nabla(div(\mathbf{p}^{n}) - (f-v)/\theta)|}.
\end{equation*}

\item Solving for the texture component $v$ with $(u,w)$ fixed:
\begin{equation}\label{E:vieqn}
\min_{v}\left\{\frac{1}{2\theta}\int_{\Omega}(u+v-f)^2\,dx+ \int_{\Omega}\mu(x)|v|\,dx  \right\},
\end{equation}
and the solution is found as 
\begin{equation}\label{E:visoln}
v=  
\begin{cases} 
f-u-\theta\mu(x) & \text{if $f-u\geq \theta\mu(x)$},\\
f-u+\theta\mu(x) & \text{if $f-u\leq -\theta\mu(x)$},\\
0 & \text{if $|f-u|\leq \theta\mu(x)$}.
\end{cases}
\end{equation}

\end{enumerate}
Next we describe a data adaptive way for choosing the fidelity parameter $\mu$ using the cartoon component at a previous iteration $u^{n}$.
\begin{remark}
We interchangeably use edges and pseudo-edges as the $w$ component provides an edge like features from a given image. The definition of edges in a digital image depends on the context and many traditional definitions depend on the magnitude of gradients (i.e., $\abs{\nabla I}$), hence a solution of the PDE~\eqref{E:constraint} provides a pseudo-edge map, see Figure~\ref{fig:teaser}(d).
\end{remark}
\subsection{Adaptive fidelity parameter}\label{ssec:adap}
\begin{figure*} 
	\centering

\subfigure[Original image]{\includegraphics[width=4.5cm, height=3.5cm]{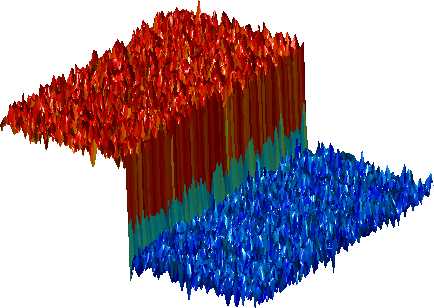}}
\subfigure[$\mu\gets$adaptive, $\lambda\gets$constant]{\includegraphics[width=4.5cm, height=3.5cm]{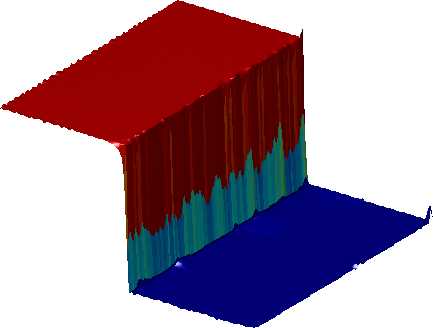}}
\subfigure[$\mu\gets$constant, $\lambda\gets$adaptive]{\includegraphics[width=4.5cm, height=3.5cm]{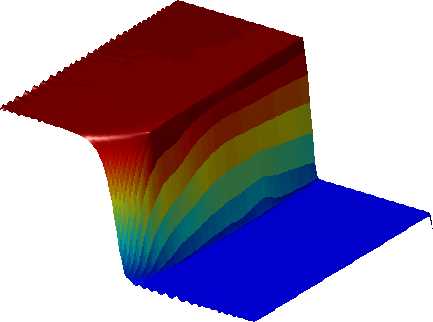}}

	\caption{(Color Online) The proposed model (stopping parameter $\epsilon=10^{-4}$) with adaptive $\mu_1$ (see Eqn.~\eqref{E:CTElambda}) and constant $\lambda$ ($=1$) provides better edge preservation in the cartoon component and captures small scale oscillations in the texture component against constant $\mu$ ($=1$) and adaptive $\lambda$ (using the definition of $\mu_1$ from Eqn.~\eqref{E:CTElambda}).}\label{fig:mulambdas}
\end{figure*}

Here, we consider the data adaptive parameters selection strategies which can provide a balanced approach in obtaining better CTE decomposition results. For the image fidelity parameter $\mu$ in Eqn.~\eqref{E:wTV} we utilize a local histogram measure which was motivated from the image segmentation model of~\cite{NB09}. For a given gray-scale image $I :\overline{\Omega}\longrightarrow [0,L]$, let
$\mathcal{N}_{{x},r}$ be the local region centered at
${x}$ with radius $r$. We compute the local histogram of
the pixel $x\in\Omega $ and its corresponding cumulative
distribution function
\begin{eqnarray}
P_{x}(y) = \frac{|\{z\in\mathcal{N}_{{x},r}\cap\Omega\,|\,
I(z)=y\}|}{\mathcal{N}_{x,r}\cap\Omega}\\
F_{x}(y) = \frac{|\{z\in\mathcal{N}_{{x},r}\cap\Omega\,|\, I(z)\leq
y\}|}{\mathcal{N}_{{x},r}\cap\Omega}
\end{eqnarray}
for $0\leq y\leq L$, respectively. This allows us to define the
following measurable function
$\mu:\Omega\rightarrow\mathbb{R}$, such that for each
${x}\in\Omega$,
\begin{equation}\label{E:CTElambda}
\mu(x) = \mu_1(x) = \frac{\int_{0}^{L} F_{x}(y)\,dy}{\max\limits_{{x}\in\Omega}\int_{0}^{L} F_{x}(y)\,dy},
\end{equation}
allowing us to get a weight of how much nonhomogeneous intensity
is present in a local region $\mathcal{N}_{{x},r}$ of a
given pixel ${x}$. This new feature of the image does not
depend on the pixel properties instead provides regional
properties, see~\cite{NB09} for more details. Thus, we see that the $\mu$ is chosen according to local histogram information and is computed in an image adaptive way using the cartoon $u^{n}$ in the iterative scheme~\eqref{E:vieqn}. We compare our approach with two related adaptive functions:
\begin{enumerate} 

\item The adaptive formulation of~\cite{PVorotnikov12} which uses a summation of cartoon components up-to iteration $n$.  
\begin{eqnarray}\label{E:VBSlambda}
\mu_2(x) = \sum_{i=0}^{n} G_{\rho_i}\star u^{n} (x)
\end{eqnarray}
with $\rho_i = 1/i^2 $ and at $n=0$ the $\lambda_2(x) = 0.05$. 

\item
Relative reduction rate based parameter proposed in local TV based scheme~\cite{LMVese10}. 
\begin{eqnarray}\label{E:IPOLlambda}
\mu_3(x)  = \frac{G_\rho\star \abs{\nabla f(x)} - G_\rho\star \abs{L_\sigma\star\nabla f(x)}}{G_\rho\star \abs{\nabla f(x)}},
\end{eqnarray}
where $L_\sigma$ is a low pass filter. Note that this adaptive parameter uses only the initial input image $f$ whereas the previous choices use $u$ computed at a previous (Eqn.~\eqref{E:CTElambda}) or every (Eqn.~\eqref{E:VBSlambda}) iteration.

\end{enumerate}
\begin{remark}
Figure~\ref{fig:mulambdas} explains the choice of adaptiveness in our coupled model~(\ref{E:wTV}-\ref{E:constraint}) for a synthetic image with different texture patterns. The first case with $\mu$ adaptive, $\lambda$ constant provides persisting cartoon component whereas the second case with $\mu$ constant, $\lambda$ adaptive (same local histogram based measure, Eqn.~\eqref{E:CTElambda} is used for defining $\lambda(x)$) blurs the boundaries in the final result. Thus, in what follows, we use only $\mu$ adaptive parameter to illustrate our decomposition results.
\end{remark}
Figure~\ref{fig:lambdas} shows a comparison of different adaptive $\mu$ functions for a synthetic texture image. We see that the local histogram based $\mu_1$ captures the texture components from all the quadrants. Moreover,  Figure~\ref{fig:lambdas}(e) shows that the energy value decreases similarly for different $\mu$ functions as the iteration increases.  Figure~\ref{fig:lambda_CTE} shows the usage of different $\mu$ function when we apply our model~(\ref{E:wTV}-\ref{E:constraint}) for the same synthetic image to obtain cartoon ($u$) + texture ($v$) + pseudo-edges (w) decomposition. Note that the texture image $v$ is obtained by linearly transforming its range to $[0,255]$ for visualization purposes. Differences outside this range are saturated to $0$ and $255$ respectively. A similar transformation is applied to the edges ($w$) component as well.
Next, we study the wellposedness of the model~(\ref{E:wTV}-\ref{E:constraint}) using weak solutions concept and prove some a priori estimates and solvability of the proposed adaptive coupled model. 

\begin{figure*}
\centering
	\subfigure[Input ($f$)]{\includegraphics[width=2.75cm, height=2.75cm]{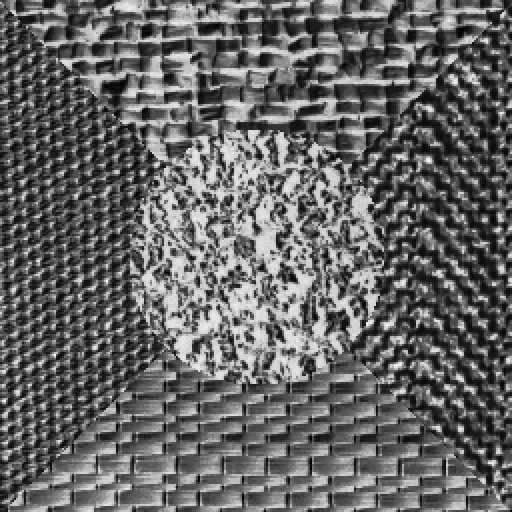}}
	\subfigure[$\mu_1$]{\includegraphics[width=2.75cm, height=2.75cm]{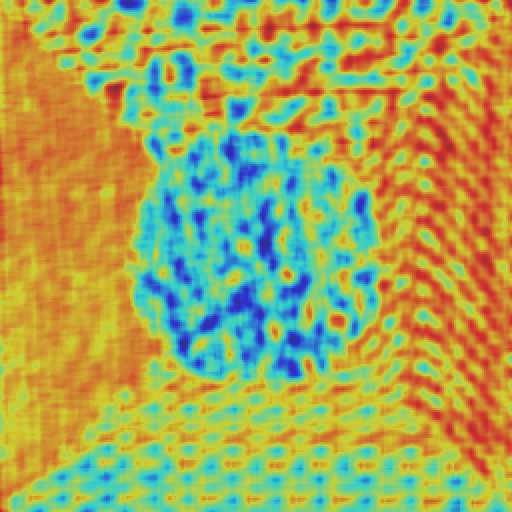}}
	\subfigure[$\mu_2$]{\includegraphics[width=2.75cm, height=2.75cm]{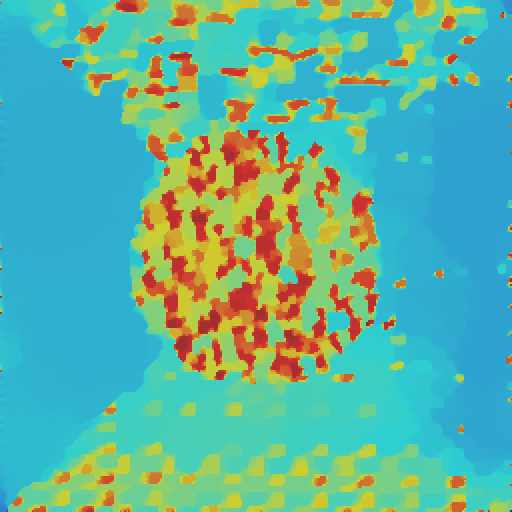}}
	\subfigure[$\mu_3$]{\includegraphics[width=2.75cm, height=2.75cm]{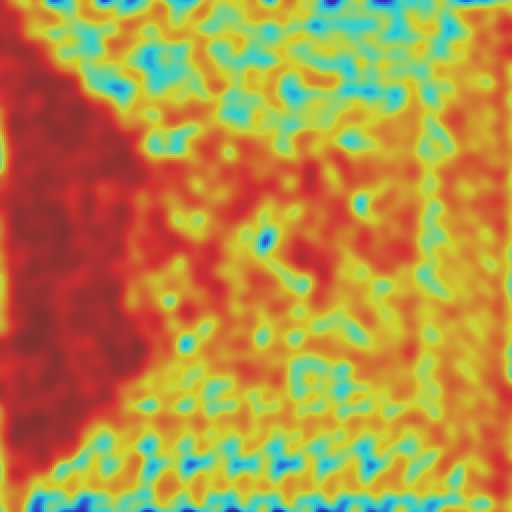}}
	\subfigure[Energy]{\includegraphics[width=4cm, height=2.75cm]{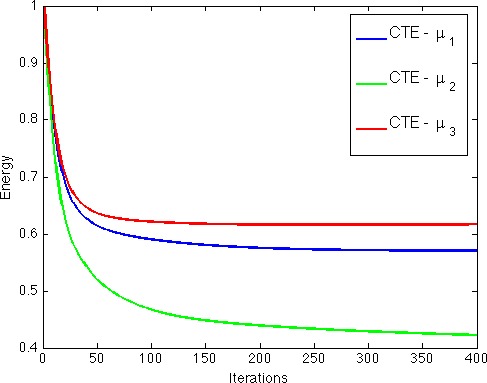}}

	\caption{(Color Online) Comparison of different $\mu$ functions computed using the given input image.
	(a) Original image.
	(b) $\mu_1$ based on local histograms  Eqn.~\eqref{E:CTElambda}.
	(c) $\mu_2$ based on the work of~\cite{PVorotnikov12} Eqn.~\eqref{E:VBSlambda}.
	(d) $\mu_3$ base on the work of~\cite{LMVese10} Eqn.~\eqref{E:IPOLlambda}.
	(e) Energy versus iteration for different adaptive $\mu$ functions based energy minimization scheme~(\ref{E:wTV}-\ref{E:constraint}) with stooping  parameter $\epsilon=10^{-4}$.}\label{fig:lambdas}
\end{figure*}
\begin{figure*}
\centering
	\includegraphics[width=2.75cm, height=2.5cm]{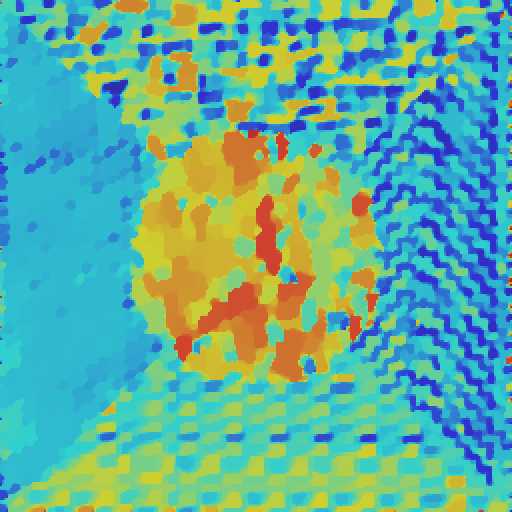}
	\includegraphics[width=2.75cm, height=2.5cm]{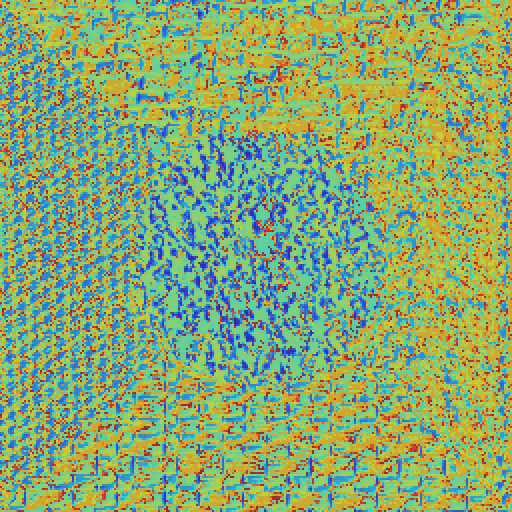}
	\includegraphics[width=2.75cm, height=2.5cm]{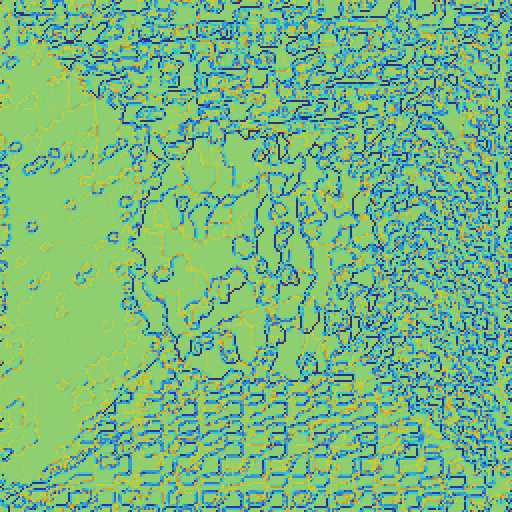}
	\includegraphics[width=2.75cm, height=2.5cm]{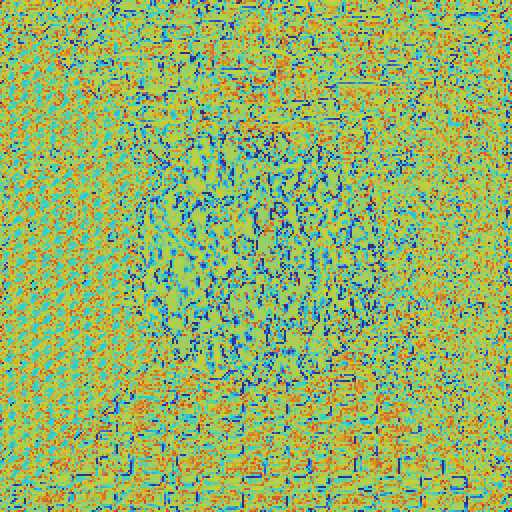}

	\includegraphics[width=2.75cm, height=2.5cm]{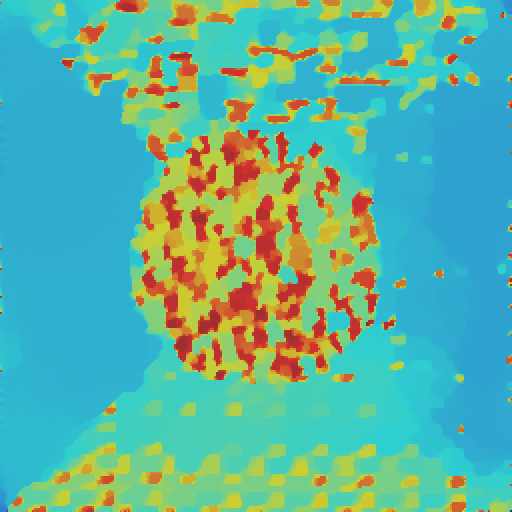}
	\includegraphics[width=2.75cm, height=2.5cm]{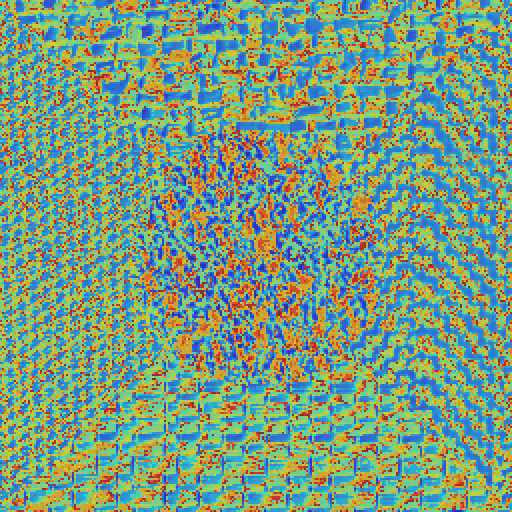}
	\includegraphics[width=2.75cm, height=2.5cm]{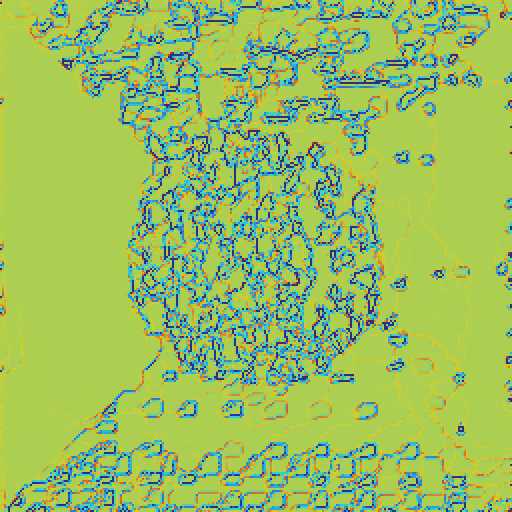}
	\includegraphics[width=2.75cm, height=2.5cm]{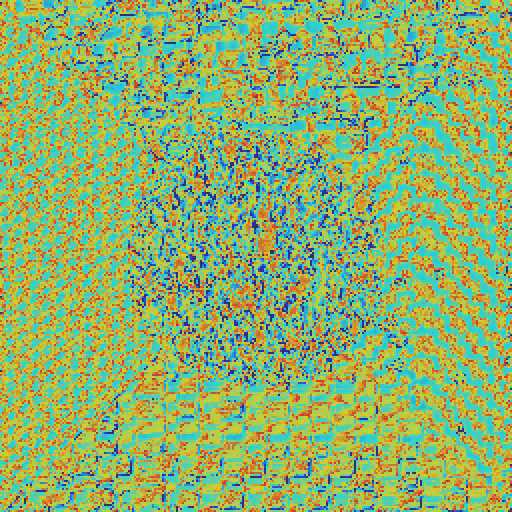}

	\subfigure[$u$]{\includegraphics[width=2.75cm, height=2.5cm]{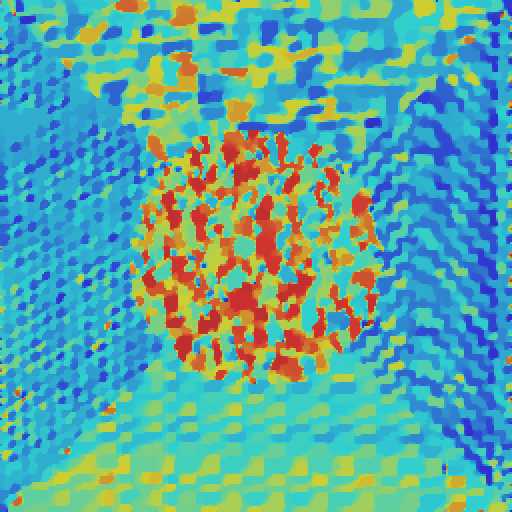}}
	\subfigure[$v$]{\includegraphics[width=2.75cm, height=2.5cm]{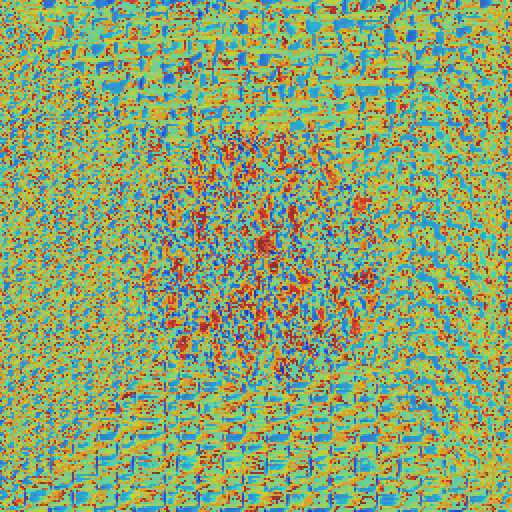}}
	\subfigure[$w$]{\includegraphics[width=2.75cm, height=2.5cm]{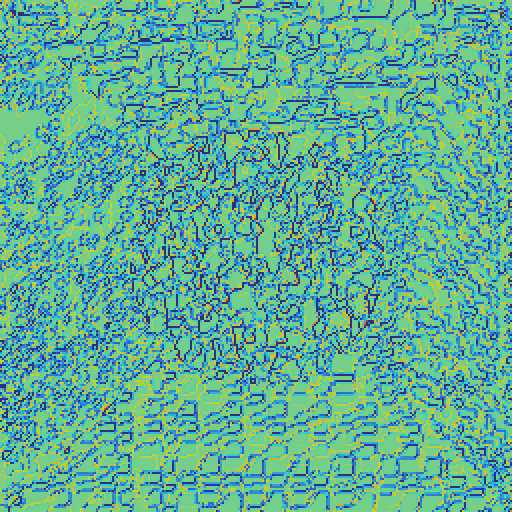}}
	\subfigure[$v+w$]{\includegraphics[width=2.75cm, height=2.5cm]{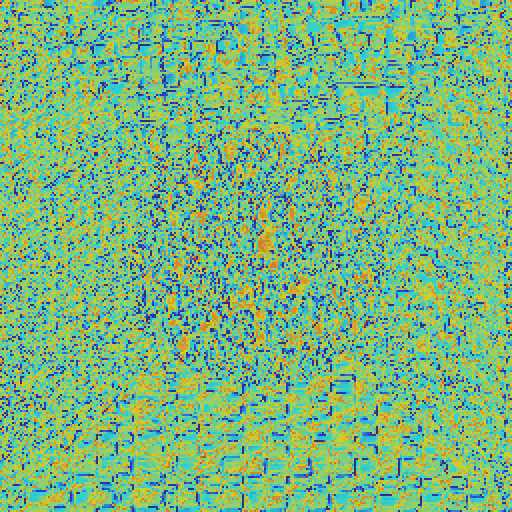}}

	\caption{(Color Online) Different $\mu$ functions based CTE scheme (with $\lambda$ constant) results (stopping parameter $\epsilon=10^{-4}$).
	Top: $\mu_1$ (local histogram) based result.
	Middle: $\mu_2$  result.
	Bottom: $\mu_3$  result. 
	(a) Cartoon $u$. 
	(b) Texture $v$.
	(c) Pseudo-edges $w$.
	(d) $v+w$. Best viewed electronically, zoomed in.}\label{fig:lambda_CTE}
\end{figure*}

\section{Wellposedness}\label{sec:theo}

\subsection{Preliminaries}\label{ssec:pl}

In the section, $\Omega$ is considered to be a bounded domain (i.e. an open set in $\R^2$) possessing the \emph{cone property}. We recall that this means that each point $x\in\Omega$ is a vertex of a \emph{finite cone} $C_x$ contained in $\Omega$, and all these cones $C_x$ are congruent \cite{Ad75}.
Obviously, rectangular domains have this property. Fix also a time interval $[0,T]$, $T>0$. 

We study wellposedness of the weighted TV regularization,
\begin{eqnarray}\label{E:wTVx}
u(t,x)=\mathbf{u}(x):\,\min\limits_{\mathbf{u}:\Omega\to \R}\Bigg\{\int_{\Omega} g(w(t,x))|\nabla \mathbf{u}(x)|\,dx + \int_{\Omega}\mu(x)|\mathbf{u}(x)-f(x)|\,dx \Bigg\},
\end{eqnarray}
with the diffusion constraint,
\begin{align}\label{E:constraintp}
\frac{\partial w(t,x)}{\partial t} &= \Delta_{p,\lambda}  w(t,x) + (1-\lambda(x))(|\nabla u(t,x)|-w(t,x)),\\ \label{E:constraintp1}
w(t,x)&=0,\ x\in\partial\Omega,\\ \label{E:constraintp2}
w(0,x)&=F(x),
\end{align}
where $p\geq 2,$ and $f:\Omega\to \R,$ $F:\Omega\to [0,+\infty),$ $\lambda:\Omega\to (0,1],$ $\mu:\Omega\to (0,+\infty),$ $g:[0,+\infty)\to (0,+\infty)$ are prescribed functions. 
The operator $\Delta_{p,\lambda}$ is a weighted $p$-Laplacian: \be\Delta_{p,\lambda} v=\lambda\divo(|\nabla v|^{p-2}\nabla v)-(1-|\nabla v|^{p-2})\nabla v\cdot \nabla\lambda.\ee In particular, for $p=2$ we recover the linear diffusion case. In this section, for the sake of generality, we admit adaptive $\lambda$ and generic $g$. Note that $w$ is non-negative by the maximum principle.
We consider the Dirichlet boundary condition for $w$, but other boundary conditions can also be handled.

We use the brief notations $L_q$ ($q\geq 1$), $W_q^{m}$  ($m\in\R$), $W^{m,q}_{0}$ ($m>0$) for the Lebesgue
and Sobolev spaces on $\Omega$ with values in $\R \ or\ \R^2$. 
Parentheses denote the bilinear form
$$ (u, v) = \int\limits_\Omega u (x)\cdot v (x)\, dx.$$ The norm in $L_2$ is $\|u\|=\sqrt{(u,u)}$.

The symbols $S (\mathcal{J}; E) $, $C (\mathcal{J}; E) $, $L_1 (\mathcal{J}; E) $ denote the spaces of Bochner measurable, continuous, Bochner integrable functions on an interval
$\mathcal{J}\subset \mathbb {R} $ with values in a Banach space $E$, respectively.

Let $\mathcal M$ be the Banach space of finite Radon measures on $\Omega$. It is the dual of the space $C_0(\Omega)$ (the space of continuous functions on $\Omega$ that vanish at $\partial\Omega$, see e.g. \cite{fonleo}). 

Let $BV$ be the space of functions of bounded variation on $\Omega$. For $v\in BV$, and $\phi\in C(\overline\Omega)$, $\phi\geq 0$, the \emph{weighted total variation} of $v$ is 
\be\label{E:tvf}TV_\phi(v)=\sup\limits_{\psi\in C_0^\infty(\Omega):\, |\psi|\leq \phi}(v,\divo \psi).\ee
In particular, the \emph{total variation} of $u$ is \be\label{E:tvf1}TV(v)=TV_1(v).\ee

Due to lower semicontinuity of suprema, for every non-negative $\phi\in C(\overline\Omega)$ and a weakly-* converging sequence $\{v_m\}\subset BV$, we have \be\label{E:twin} TV_{\phi}(v)\leq\lim_{m\to +\infty} \inf TV_{\phi}(v_m).\ee
A more refined argument of the same nature proves
\begin{lemma} \label{L:twin} For any $\varphi\in S(0,T;C(\overline\Omega))$, $\varphi\geq 0$ for a.a. $t\in(0,T)$, and a weakly-* converging sequence $\{v_m\}\subset L_q(0,T;BV)$, $q>1$, one has \be \label{E:twinlem} TV_{\varphi(t)}(v(t))\leq \lim_{m\to +\infty} \inf TV_{\varphi(t)}(v_m(t))\ee for a.a. $t\in(0,T)$. \end{lemma}

For $v\in BV$, $|\nabla v|$ will denote the corresponding total variation measure. The operator \be|\nabla (\cdot)|: BV \to \mathcal M\ee is bounded. We recall the duality relation \be TV_{\phi}(v)=\langle |\nabla v|, \phi\rangle_{\mathcal M\times C_0(\Omega)}.\ee

The symbol $C$ will stand for a generic positive constant that can take different values in different lines. We sometimes write $C(\dots)$ to specify that the constant depends on a certain parameter or value.

We will use the embeddings \be\label{E:embbv} BV\subset L_{q},\ q\leq 2,\ee \be \label{E:w2} W^1_2\subset L_{q},\ q<+\infty,\ee \be \label{E:wp} W^1_p\subset C(\overline{\Omega}),\ p> 2,\ee and
\be \label{E:wp-}\mathcal{M} \subset W^{-1}_q,\ q< 2,\ee
and the Poincar\'{e} inequality  \be \label{E:ineqg}  \|v\|_{W^1_p}\leq C\|\nabla v\|_{L_p},\ p\geq 1,\ v\Big|_{\partial \Omega}=0.\ee 
Embeddings \eqref{E:embbv}--\eqref{E:wp-}  are compact (except for \eqref{E:embbv} with $q=2$).

We assume that $\lambda$, $\nabla \lambda$ and $g$ are Lipschitz functions, \[\lambda_0=\inf_{x\in\Omega}\lambda(x)>0,\quad\quad g_0=\sup_{y\geq 0} g(y)< +\infty,\]
and there exists a constant $C_g$ so that 
\be\label{E:cg}
\Big|\frac {d (\log g(y))}{dy}\Big|\leq \frac {C_g}{1+y}\ \,   \mathrm{for\ a.a.}\ y\geq 0. 
\ee
The last condition means that $g$ can have at most polynomial decay at infinity.

We assume that $\mu\in L_\infty(\Omega),$ and 
\[ 0<\mu_1=\mathrm{ess}\inf_{x\in\Omega} \mu(x)\leq \mu_2=\mathrm{ess}\sup_{x\in\Omega} \mu(x)<+\infty.\]

Finally, we assume that \be\label{E:inff} F\in L_2,\ee and \textbf{at least one} of the following three conditions holds: 
\be\label{E:cd1}
p>2,\, f\in BV,
\ee
\be\label{E:cd2}
p=2,\, \exists\, q>1:\, f\in W^1_{q}, 
\ee
\be\label{E:cd3}
p=2,\, f\in BV\cap L_\infty,\,\exists\, c_g>0:\, \frac 1{g(y)}\leq c_g(1+y),\ y\geq 0.
\ee


\subsection{A priori estimates}\label{ssec:ape}

Before specifying the underlying function spaces and defining the notion of solution, let us derive a formal a priori estimate for problem \eqref{E:wTVx}--\eqref{E:constraintp2}. 

The Euler-Lagrange equation for \eqref{E:wTVx} is 
\be\label{E:el}
-\divo \left(g(w)\frac{\nabla u}{|\nabla u|} \right)+\mu\frac {u-f}{|u-f|}=0,\ \frac {\partial u/\partial \nu}{|\nabla u|}\Bigg|_{\partial \Omega} = 0.
\ee
For each $t\in[0,T]$, multiplying \eqref{E:el} by $\frac{w (u-f)}{g(w)}$, and integrating over $\Omega$, we get 
\be\label{E:el1}
\left(g(w)\frac{\nabla u}{|\nabla u|},\nabla\left(\frac{w (u-f)}{g(w)}\right) \right)+ \left(\frac{\mu w}{g(w)},|u-f|\right)=0.
\ee
Thus,
\begin{multline}\label{E:el2}
\left(\frac{\nabla u}{|\nabla u|},\nabla w(u-f)\right)-\left(\frac{\nabla u}{|\nabla u|},\frac {g'(w)w}{g(w)}\nabla w(u-f)\right)\\+\left(w,|\nabla u|\right)-\left(w\frac{\nabla u}{|\nabla u|},\nabla f\right) + \left(\frac{\mu w}{g(w)},|u-f|\right)=0.
\end{multline}

Multiplying \eqref{E:el} by $\frac{u-f}{g(w)}$, and integrating over $\Omega$, we find
\begin{eqnarray}\label{E:el4}
-\left(\frac{\nabla u}{|\nabla u|},\frac {g'(w)}{g(w)}\nabla w(u-f)\right)
+\left(1,|\nabla u|\right)-\left(\frac{\nabla u}{|\nabla u|},\nabla f\right) + \left(\frac{\mu}{g(w)},|u-f|\right)=0.
\end{eqnarray}
Since the last term is non-negative, we conclude that \be\label{E:el5}
TV(u)\leq TV(f)+C_g\|\nabla w\|\|u-f\|.
\ee

Multiplying \eqref{E:el} by $u-f$, and integrating over $\Omega$, we derive
\be\label{E:tvi}
(g(w),|\nabla u|)+\|\mu(u-f)\|_{L_1}\leq g_0\, TV(f)
\ee

Multiplying \eqref{E:el} by $u|u|$, and integrating over $\Omega$, we deduce
\be\label{E:estu}
2(g(w)|u|,|\nabla u|)+\left(\frac {u-f}{|u-f|},\mu u|u|\right).
\ee
It is not difficult to obtain the following scalar inequality
\be\label{E:sc1} |a-b|^3\leq 2(a-b)a|a|+2b^2|a-b|,\ a,b\in\R,\ee
which enables to conclude from \eqref{E:estu} that 
\be\label{E:estu1}
\|\sqrt{\mu}(u-f)\|^2+4(g(w)|u|,|\nabla u|)\leq 2\|\sqrt{\mu}f\|^2.
\ee
Hence, due to \eqref{E:embbv}, 
\be\label{E:estu2}
\|u-f\|\leq C(\|f\|)\leq C(\|f\|_{BV}).
\ee
By \eqref{E:el5} and \eqref{E:estu2}, 
\be\label{E:el6}
TV(u)\leq C(1+\|\nabla w\|).
\ee

Multiplying \eqref{E:constraintp} by $w$, and integrating over $\Omega$, we get
\begin{eqnarray}\label{E:next0}
\frac 12 \frac {d \|w\|^2}{dt} +(\lambda, |\nabla w|^p)+((1-\lambda)w,w)
=((1-\lambda)w,|\nabla u|)	-(w\nabla\lambda,\nabla w).
\end{eqnarray}

Using H\"{o}lder's inequality and \eqref{E:cg}, we deduce from \eqref{E:el2} that
\be\label{E:el3}
\left(w,|\nabla u|\right)\leq (1+C_g)\|\nabla w\|_{L_p}\|u-f\|_{L_{p/p-1}} +(w,|\nabla f|).
\ee
From \eqref{E:next0}, \eqref{E:el3} and Young's inequality we infer 
\begin{eqnarray}\label{E:next1}
\frac 12 \frac {d \|w\|^2}{dt} +\frac {\lambda_0} 2\|\nabla w\|^p_{L_p}
\leq C(g,\lambda_0,p)\|u-f\|_{L_{p/p-1}}^{p/p-1} +C(\lambda,p)\|w\|_{L_{p/p-1}}^{p/p-1}+(w,|\nabla f|).
\end{eqnarray}

Provided \eqref{E:cd1} or \eqref{E:cd2}, 
estimate \eqref{E:estu2}, and embeddings \eqref{E:wp} or \eqref{E:w2}, resp., imply
\be\label{E:fin1} \frac {d \|w\|^2}{dt} +\lambda_0 \|\nabla w\|^p_{L_p}\leq C(1+\|w\|^{p/p-1}+\|w\|_{W_p^1}).\ee
By \eqref{E:ineqg},\eqref{E:inff} and usual arguments, \eqref{E:fin1} yields 
\be\label{E:fin2} \|w\|_{L_\infty(0,T; L_2)} +\| w\|_{L_p(0,T; W^{1,p}_{0})}\leq C.\ee

In case \eqref{E:cd3}, we multiply \eqref{E:el} by $\frac{w u}{g(w)}$, and integrate over $\Omega$, arriving at 
\begin{multline}\label{E:ele}
\left(\frac{\nabla u}{|\nabla u|},u\nabla w\right)-\left(\frac{\nabla u}{|\nabla u|},u\frac {g'(w)w}{g(w)}\nabla w\right)+\left(w,|\nabla u|\right)\\+ \left(\frac{\mu w}{g(w)},|u-f|\right)+\left(\frac {u-f}{|u-f|},\frac{\mu wf}{g(w)}\right)=0.
\end{multline}
Then, since the penultimate term is non-negative, 
\begin{multline}\label{E:ele1}
\left(w,|\nabla u|\right)
\leq (1+C_g)\|\nabla w\|\,\|u\| +\mu_2c_g(|f|w,1+w)\\
\leq C(1+\|\nabla w\|+\|w\|^2).
\end{multline}
Now, \eqref{E:next0} and \eqref{E:ele1} yield \be\label{E:fin3} \frac {d \|w\|^2}{dt} +\lambda_0 \|\nabla w\|^p_{L_p}\leq C(1+\|w\|^2),\ee which implies \eqref{E:fin2}.
In all the three cases, \eqref{E:tvi}, \eqref{E:el6} and \eqref{E:fin2} imply
\be\label{E:l1est} \|u\|_{L_\infty(0,T; L_1)}+\|u\|_{L_p(0,T; BV)} \leq C.\ee

The operator 
\[\Lambda:  BV\to \mathcal{M},\]
\[\Lambda(v)=(1-\lambda)|\nabla v|,\]
\[\langle\Lambda(v),\phi\rangle_{\mathcal M\times C_0(\Omega)}
=\langle|\nabla v|, (1-\lambda)\phi\rangle_{\mathcal M\times C_0(\Omega)},\]
is bounded, so \be\label{E:l2est} \|(1-\lambda)|\nabla u|\|_{L_p(0,T; \mathcal{M})} \leq C.\ee
The weighted $p$-Laplacian operator $$ \Delta_{p,\lambda}: W^{1,p}_0\to W^{-1}_{p/p-1}$$ is also bounded. 
Hence, \eqref{E:constraintp},\eqref{E:wp-},\eqref{E:fin2} and \eqref{E:l2est} yield an estimate for the time derivative of $w$:
\be\label{E:tder} \|w'\|_{L_p(0,T; W^{-1}_{q})}\leq C,\ q<2, \frac q {q-1}\geq p.\ee

\subsection{Solvability}\label{ssec:sss}

\begin{definition} 
Assume \eqref{E:inff} and \eqref{E:cd1}. A pair of functions $(u,w)$ from the class \[ u\in L_\infty(0,T; L_1)\cap L_p(0,T; BV),\] \[ w\in L_\infty(0,T; L_2)\cap L_p(0,T; W^{1,p}_{0})\cap W_p^1(0,T; W^{-1}_{p/p-1}),\] is called a weak solution to problem~\eqref{E:wTVx}--\eqref{E:constraintp2} if  
\begin{eqnarray}\label{E:wTVw}
TV_{g(w(t))}(u(t)) + \|\mu(u(t)-f)\|_{L_1}
\leq TV_{g(w(t))}(\mathbf{u}(t)) + \|\mu(\mathbf{u}(t)-f)\|_{L_1}, 
\end{eqnarray} 
for any $\mathbf{u}\in S(0,T; BV)$ and a.a. $t\in (0,T)$, 
\be\label{E:constraintpn}
w' -\Delta_{p,\lambda} +(1-\lambda)w=
(1-\lambda)|\nabla
u|\ee in the space $W^{-1}_{p/p-1}$ for a.a. $t\in (0,T)$, and \be \label{E:constraintp2n}
w(0)=F
\ee in $W^{-1}_{p/p-1}$.
\end{definition}

\begin{remark} 
This definition is correct since all members of~\eqref{E:constraintpn} belong to $W^{-1}_{p/p-1}$ for a.a. $t\in (0,T)$ (cf. the end of Subsection~\ref{ssec:ape}), and $w\in  W_p^1(0,T; W^{-1}_{p/p-1})\subset C([0,T]; W^{-1}_{p/p-1})$.\end{remark}

\begin{theorem} \label{T:t1}  Assume \eqref{E:inff} and \eqref{E:cd1}. Then there exists a weak solution to \eqref{E:wTVx}--\eqref{E:constraintp2}.
\end{theorem}
\begin{proof}
We can prove the existence of weak solutions via approximation of~\eqref{E:wTVx}--\eqref{E:constraintp2} by a more regular problem, and consequent passage to the limit (cf. \cite{PVorotnikov12,ZV08}). 
Let $(u_m,w_m)$ be a sequence of ``approximate'' solutions (possibly with ``approximate'' data $f_m$ and $F_m$). It essentially suffices to show that \eqref{E:wTVw}--\eqref{E:constraintp2n} is the limiting case of  \eqref{E:wTVx}--\eqref{E:constraintp2}, i.e., that it is possible to pass to the limit in all the members. 

Due to estimates \eqref{E:fin2}, \eqref{E:l1est}, \eqref{E:tder}, without loss of generality we may suppose that 
\be\label{E:cnv1} u_m\to u\ \mathrm{weakly-*\ in}\ L_\infty(0,T; \mathcal{M}),\ee 
\be\label{E:cnv2} u_m\to u\ \mathrm{weakly-*\ in}\ L_p(0,T; BV),\ee
\be\label{E:cnv3} w_m\to w\ \mathrm{weakly-*\ in}\ L_\infty(0,T; L_2), \ee
\be\label{E:cnv4} w_m\to w\ \mathrm{weakly\ in}\ L_p(0,T; W^{1,p}_{0}), \ee
\be\label{E:cnv5} w'_m\to w'\ \mathrm{weakly\ in}\ L_p(0,T; W^{-1}_{p/p-1}).\ee
Note that \be u\in L_\infty(0,T; L_1)\subset L_\infty(0,T; \mathcal{M})\cap L_p(0,T; BV).\ee
By \eqref{E:wp}, \eqref{E:wp-} and the Aubin-Lions-Simon theorem \cite{ZV08},
\be\label{E:cnv7} w_m\to w\ \mathrm{strongly\ in}\ L_p(0,T; C(\overline\Omega)), \ee 
\be\label{E:cnv6} w_m\to w\ \mathrm{strongly\ in}\ C([0,T]; W^{-1}_{p/p-1}), \ee so
\be\label{E:cnv8} w_m(0)\to w(0)\ \mathrm{in}\ W^{-1}_{p/p-1}, \ee and we can pass to the limit in \eqref{E:constraintp2n}.

Using the representation \be\|v\|_{L_1}=\sup\limits_{\varphi\in L_\infty,\,\|\varphi\|_{L_\infty}\leq 1}(\varphi,v),\ee and lower semicontinuity of suprema, we can check that 
\begin{align} 
\|\mu(u(t)-f)\|_{L_1}&\leq \lim_{m\to +\infty} \inf \|\mu(u_m(t)-f)\|_{L_1}\notag\\
&=\lim_{m\to +\infty} \inf \|\mu(u_m(t)-f_m)\|_{L_1}\label{E:cnv10} 
\end{align}
for a.a. $t\in (0,T)$.
By Lemma \ref{L:twin}, \be
\label{E:cnv9} TV_{g(w(t))}(u(t))\leq \lim_{m\to +\infty} \inf TV_{g(w(t))}(u_m(t)).\ee
But 
\begin {multline} \label{E:cnva} 
|TV_{g(w_m(t))}(u_m(t))- TV_{g(w(t))}(u_m(t))|\\
\leq \|g(w_m(t))-g(w(t))\|_{L_\infty}TV(u_m(t))
\leq C(g)\|w_m(t)-w(t)\|_{L_\infty}TV(u_m(t)),
\end{multline} 
so 
\begin{eqnarray}\label{E:cnva1} 
\|TV_{g(w_m)}(u_m)- TV_{g(w)}(u_m)\|_{L_{p/2}(0,T)}
\leq C\|w_m-w\|_{L_p(0,T;L_\infty)}\|u_m\|_{L_p(0,T; BV)}\to 0.
\end{eqnarray} 
Therefore, without loss of generality, \be \label{E:cnv11} TV_{g(w_m(t))}(u_m(t))- TV_{g(w(t))}(u_m(t))\to 0 \ee for a.a. $t\in (0,T)$. Due to  \eqref{E:cnv10}, \eqref{E:cnv9}, \eqref{E:cnv11}, we can pass to the limit in \eqref{E:wTVw}.

On the other hand, \eqref{E:wTVw} with $u=u_m$, $f=f_m$, $w=w_m$, $\mathbf{u}=u$ gives
\begin{eqnarray}\label{E:wTVwin}
TV_{g(w_m(t))}(u_m(t)) + \|\mu(u_m(t)-f_m)\|_{L_1}
\leq TV_{g(w_m(t))}(u(t)) + \|\mu(u(t)-f_m)\|_{L_1}.
\end{eqnarray}
Similarly to \eqref{E:cnva}--\eqref{E:cnv11}, we can check that
\be \label{E:cnva2} TV_{g(w_m(t))}(u(t))- TV_{g(w(t))}(u(t))\to 0 .\ee From \eqref{E:cnv10}, \eqref{E:cnv9},  \eqref{E:cnv11}--\eqref{E:cnva2} we conclude that 
\be \label{E:cnv12} TV_{g(w(t))}(u_m(t))\to TV_{g(w(t))}(u(t)) \ee for a.a. $t\in (0,T)$.

Fix any non-negative function $\phi\in C_0(\Omega)$. Let \be\kappa(t)=\left\|\frac\phi{g(w(t))}\right\|_{L\infty}\ee and \be\varphi(t)=\kappa(t)g(w(t))-\phi.\ee For a.a. $t\in (0,T)$, $\varphi(t)$ is a non-negative continuous function on $\overline\Omega$. By Lemma \ref{L:twin} and \eqref{E:cnv12}, we infer that 
\begin{multline}\label{tvin1} 
TV_{\phi}(u(t))=\kappa(t) TV_{g(w(t))}(u(t)) - TV_{\varphi(t)}(u(t)) \geq \lim_{m\to+\infty}\sup (\kappa(t) TV_{g(w(t))}(u_m(t)) - TV_{\varphi(t)}(u_m(t)))\\
=\lim_{m\to+\infty}\sup TV_{\phi}(u_m(t)).
\end{multline} 
But, due to \eqref{E:twinlem},
\be\label{tvin2} TV_{\phi}(u(t))\leq \lim_{m\to+\infty}\inf TV_{\phi}(u_m(t)).\ee
Thus, \be\label{tvin3} TV_{\phi}(u(t))= \lim_{m\to+\infty} TV_{\phi}(u_m(t)),\ee
for every non-negative $\phi\in C_0(\Omega)$, which yields \be \label{E:cnv13} |\nabla u_m(t)|\to |\nabla u(t)| \ee weakly-* in $\mathcal{M}$ for a.a. $t\in (0,T)$. Then \eqref{E:wp-} implies
\be \label{E:cnv14} (1-\lambda)|\nabla u_m(t)|\to (1-\lambda)|\nabla u(t)| \ee strongly in $W^{-1}_{p/p-1}$ for a.a. $t\in (0,T)$. Due to \eqref{E:cnv2} and \eqref{E:wp-}, 
\be \label{E:cnv16} \|(1-\lambda)|\nabla u_m|\|_{L_p(0,T; W^{-1}_{p/p-1})}\leq C. \ee
By \eqref{E:cnv14},  \eqref{E:cnv16} and \cite[Proposition 2.8, Remark 2.10]{mt5}, 
\begin{eqnarray}\label{E:cnv15}
(1-\lambda) |\nabla u_m|\to (1-\lambda)|\nabla u|  \mathrm{strongly\ in}\ L_q(0,T; W^{-1}_{p/p-1}),\ \forall q<p.
\end{eqnarray}

Rewrite \eqref{E:constraintpn} as \be\label{E:constraintpm}
w' +Aw=K(u,w),\ee
where 
\begin{eqnarray} 
A(w)=-\divo(\lambda|\nabla w|^{p-2}\nabla w) +(1-\lambda)w,\ K(u,w)
=-\nabla w\cdot \nabla\lambda+(1-\lambda)|\nabla u|.
\end{eqnarray}
It is easy to see that the operator $A:W^{1,p}_0\to W^{-1}_{p/p-1}$ is monotone, coercive and hemi-continuous (cf. \cite{oldlion}).  By~\eqref{E:cnv7} and~\eqref{E:cnv15},
\[K(u_m,w_m)\to K(u,w)\ \mathrm{strongly\ in}\ L_{p/p-1}(0,T; W^{-1}_{p/p-1}).\] 
Hence, we can successfully pass to the limit in \eqref{E:constraintpm} via Minty-Browder monotonicity technique (cf. \cite{oldlion}).
\end{proof}

\begin{definition} 
Assume \eqref{E:inff} and \eqref{E:cd2} or \eqref{E:cd3}. A pair of functions $(u,w)$ from the class \be u\in L_\infty(0,T; L_1)\cap L_2(0,T; BV),\ee \be w\in L_\infty(0,T; L_2)\cap L_2(0,T; W^{1,2}_{0})\cap W_2^1(0,T; W^{-1}_{q}),\ \forall q<2,\ee is called a pseudosolution to problem \eqref{E:wTVx}--\eqref{E:constraintp2} if there is a sequence $(u_m,w_m,p_m)$ such that each pair $(u_m,w_m)$ is a weak solution to \eqref{E:wTVx}--\eqref{E:constraintp2} with $p=p_m$, 
\[u_m\to u\ \mathrm{weakly-*\ in}\ L_\infty(0,T; \mathcal{M}),\]
\[u_m\to u\ \mathrm{weakly-*\ in}\ L_2(0,T; BV),\]
\[w_m\to w\ \mathrm{weakly-*\ in}\ L_\infty(0,T; L_2),\]
\[w_m\to w\ \mathrm{weakly\ in}\ L_2(0,T; W^{1,2}_{0}),\]
\[w_m\to w\ \mathrm{strongly\ in}\ L_2(0,T; L_q),\ \forall q<+\infty,\]
\[w_m\to w\ \mathrm{strongly\ in}\ C([0,T]; W^{-1}_{q}),\ \forall q<2,\] 
\[w'_m\to w'\ \mathrm{weakly\ in}\ L_2(0,T; W^{-1}_{q}),\ \forall q<2,\]
\[p_m\to 2.\]
\end{definition}

\begin{theorem}\label{T:t2} 
Assume \eqref{E:inff} and \eqref{E:cd2} or \eqref{E:cd3}. Then there exists a pseudosolution to \eqref{E:wTVx}--\eqref{E:constraintp2}.
\end{theorem}
\begin{proof}
The proof is based on estimates \eqref{E:fin2}, \eqref{E:l1est}, \eqref{E:tder} and the proof of Theorem \ref{T:t1}.
\end{proof}

\section{Experimental Results}\label{sec:exper}

\subsection{Implementation details}\label{ssec:impl}

The proposed scheme is implemented using the dual minimization~\cite{Ch04} for the weighted TV (Eqn.~\eqref{E:wTV}) and explicit Euler finite difference scheme for the non-homogenous linear diffusion (Eqn.~\eqref{E:constraint}). The edge indicator function $g(w) = 1/(1+w^2)$ is used for all the results reported here. We obtained similar results for other $g$ functions. The adaptive $\mu_1$ based results are reported here unless otherwise stated explicitly and $\mu_2$ provided similar results whereas $\mu_3$ provided slightly blurred cartoon components, see Section~\ref{ssec:adap} for details. The parameters $\delta x=1$, $\delta t=1/8$ and $\theta=10^{-2}$ are fixed, and the best results according to the $max$($\abs{u^{n+1}-u^n}$,$\abs{v^{n+1}-v^n}$$\leq\epsilon$) are shown. By constant $\lambda$ and $\mu$ in the results we mean they are taken as constant value $0.5$ and $1$ respectively. The implementation of the proposed scheme is done for the constant choice ($\mu\gets$constant, $\lambda\gets$constant), the adaptive choice ($\mu\gets$adaptive, $\mu\gets$constant) and the multiscale case ($\mu\gets$multiscale, $\lambda\gets$constant). 

The algorithm is visualized in MATLAB 7.8(R2009a) on a 64-bit Windows 7 laptop with 3GB RAM, 2.20GHz CPU.  It takes on average $<10$ \textit{sec} for $50$ iterations for $3$ channels image of size $481\times 321$. Implementation is done over the following databases: Brodatz texture collection\footnote{\href{http://multibandtexture.recherche.usherbrooke.ca}{http://multibandtexture.recherche.usherbrooke.ca}}, Mosaic art images\footnote{\href{http://www.cse.cuhk.edu.hk/leojia/projects/texturesep/}{http://www.cse.cuhk.edu.hk/leojia/projects/texturesep/}}, Kodak Lossless True Color Image Suite\footnote{\href{http://r0k.us/graphics/kodak}{http://r0k.us/graphics/kodak}}, Color Test Images Database\footnote{\href{http://www.hlevkin.com/TestImages/classic.htm}{http://www.hlevkin.com/TestImages/classic.htm}}, USC-SIPI Image Database\footnote{\href{http://sipi.usc.edu}{http://sipi.usc.edu}} and Simulated Brain Database\footnote{\href{http://brainweb.bic.mni.mcgill.ca/brainweb/}{http://brainweb.bic.mni.mcgill.ca/brainweb/}}, Berkeley segmentation dataset of $500$\footnote{\href{http://www.eecs.berkeley.edu/Research/Projects/CS/vision/bsds/}{http://www.eecs.berkeley.edu/Research/Projects/CS/vision/bsds/}}.

\subsection{Image decomposition results}\label{ssec:imag}

\subsubsection{Gray-scale images}\label{sssec:gray}
\begin{figure*}
\centering
\[\begin{array}{cccc}
\begin{minipage}[l]{2.8cm}
\includegraphics[width=2.8cm,height=2.8cm]{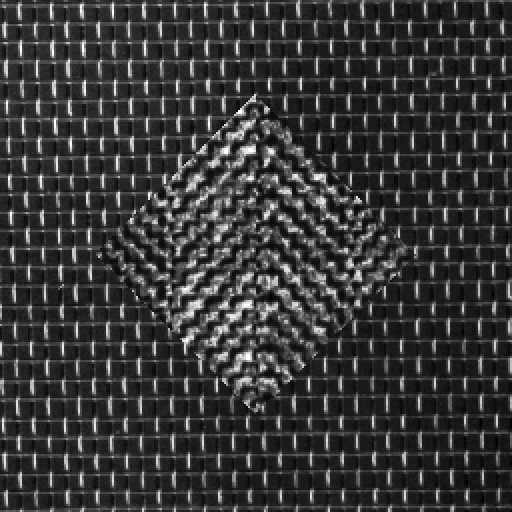}
\end{minipage}
&
\begin{minipage}[c]{2.8cm}
\includegraphics[width=2.8cm,height=2.8cm]{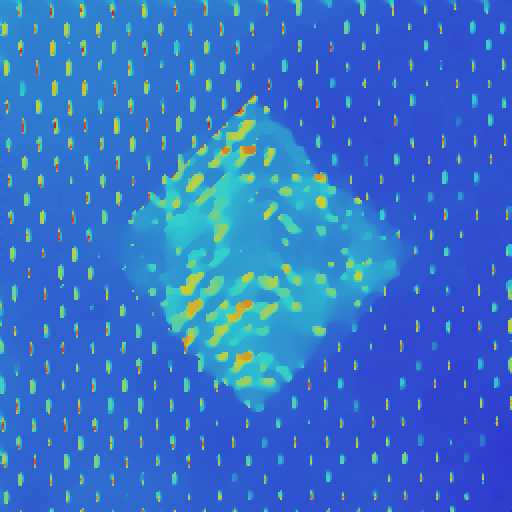}
\includegraphics[width=2.8cm,height=2.8cm]{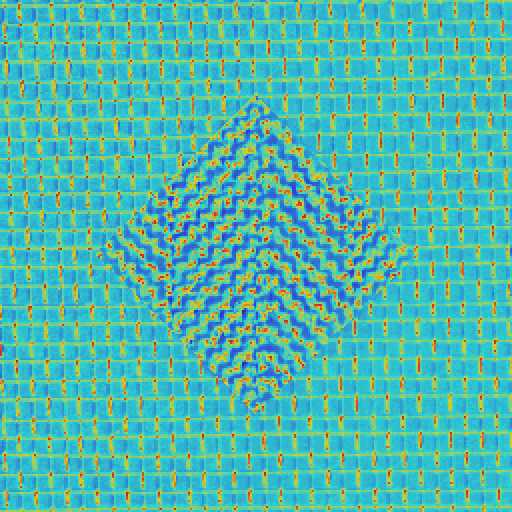}
\end{minipage}
&
\begin{minipage}[c]{2.8cm}
\includegraphics[width=2.8cm,height=2.8cm]{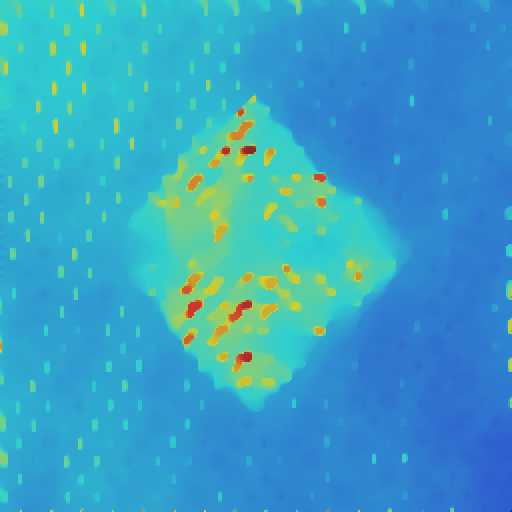}
\includegraphics[width=2.8cm,height=2.8cm]{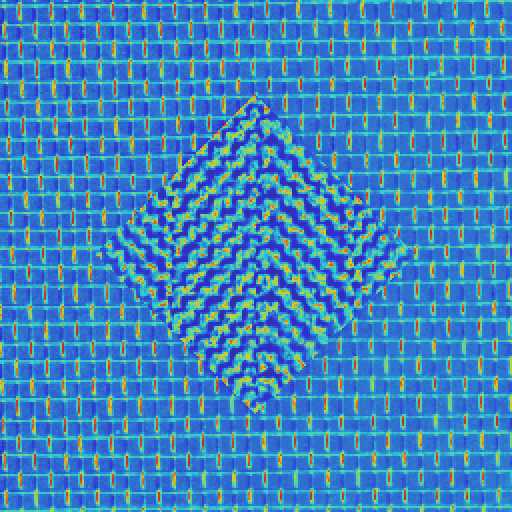}
\end{minipage}
&\quad
\begin{minipage}[r]{6.5cm}
\includegraphics[width=6.5cm,height=5cm]{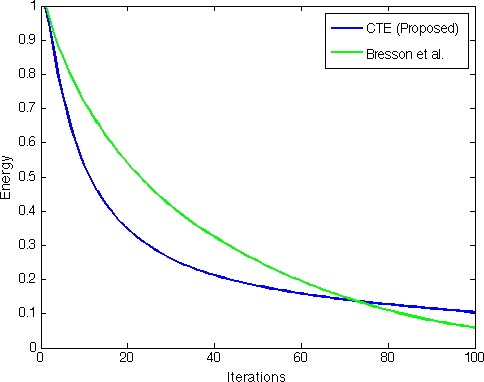}
\end{minipage}\\
\mbox{\scriptsize{(a) Original }} &
\mbox{\scriptsize{(b) u \& v (Our) }} &
\mbox{\scriptsize{(c) u \& v (\cite{BE07})}} &
\mbox{\scriptsize{(d) Energy Vs Iterations}} 
\end{array}\]
	\caption{(Color online) Comparison of our constant $\mu$, $\lambda$ proposed scheme (second column) with Bresson et al~\cite{BE07} (third column), shows that our scheme preserves large-scale textures and shape boundaries (stopping parameter $\epsilon=10^{-4}$).  Energy value comparison between our scheme CTE and Bresson et al.~\cite{BE07} shows similar convergence property. Best viewed electronically, zoomed in.}\label{fig:Bresson}
\end{figure*}
We first show decomposition results of~\cite{BE07} with our model in Figure~\ref{fig:Bresson} for a synthetic image which consists of two different texture regions. Comparing the cartoon - texture decomposition of our scheme (Figure~\ref{fig:Bresson}(b)) with the results of  Bresson et al (Figure~\ref{fig:Bresson}(c)), we see that they behave different visually. For example, the shape of the diamond at the center is preserved well in our scheme whereas the~\cite{BE07} scheme blurs it in the final result. Figure~\ref{fig:Bresson}(c) shows the energy value against number of iterations for the same synthetic image, which indicates that our adaptive CTE scheme decreases the energy values comparable to~\cite{BE07} model.  More grayscale image decomposition results are given in Figure~\ref{fig:teaser}. We see that the cartoon component obtained are piecewise constant segments indicating the potential for image segmentation~\cite{Pr09}. The texture and edges component are complementary and it is clear that edges are based on the cartoon subregions, see for example, Figure~\ref{fig:teaser}(d) top row.  

\subsubsection{Color images}\label{sssec:color}

\begin{figure*}[t]
\centering
\[\begin{array}{ccc}
\begin{minipage}[c]{4cm}
\includegraphics[width=4cm,height=4.5cm]{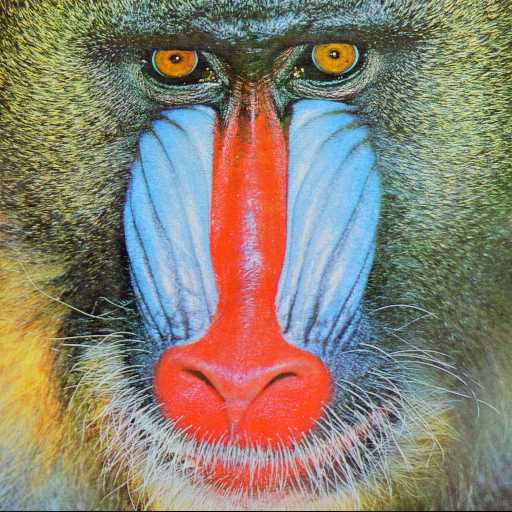}
\end{minipage}
&
\begin{minipage}[c]{4cm}
\includegraphics[width=4cm,height=4.5cm]{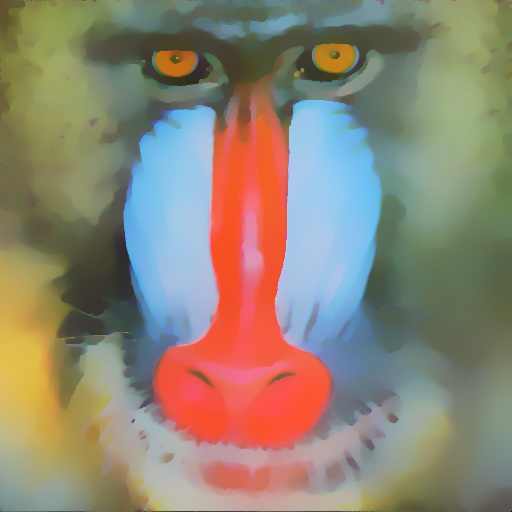}
\end{minipage}
\begin{minipage}[r]{2.3cm}
\includegraphics[width=2.25cm,height=2.23cm]{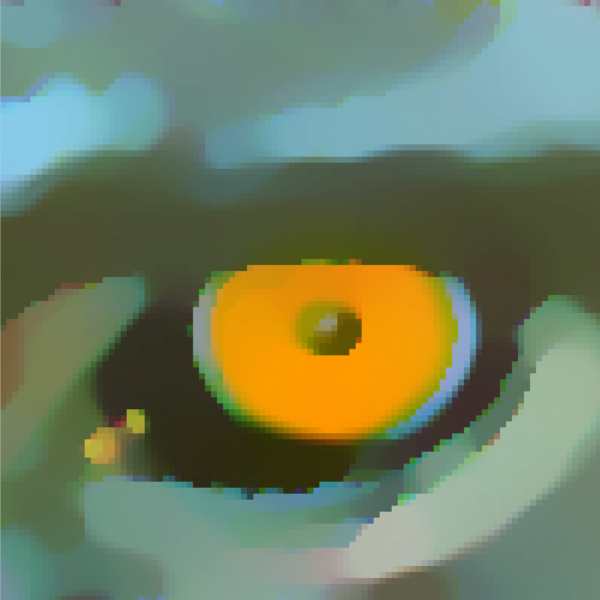}\\[.2ex]
\includegraphics[width=2.25cm,height=2.23cm]{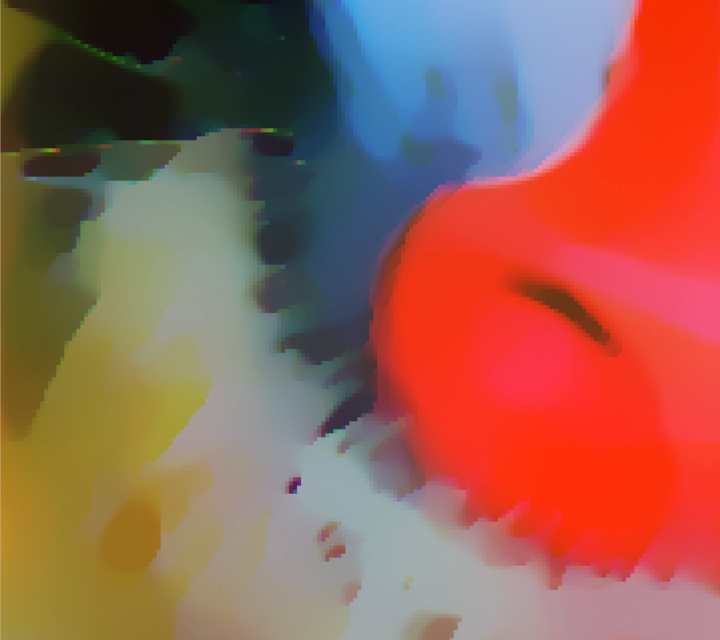}
\end{minipage}
&
\begin{minipage}[c]{4cm}
\includegraphics[width=4cm,height=4.5cm]{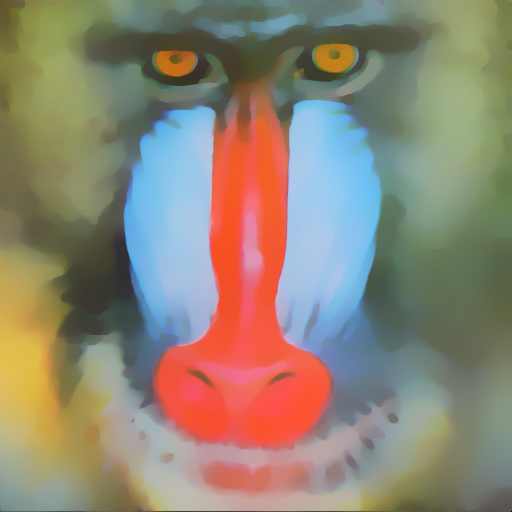}
\end{minipage}
\begin{minipage}[r]{2.3cm}
\includegraphics[width=2.25cm,height=2.23cm]{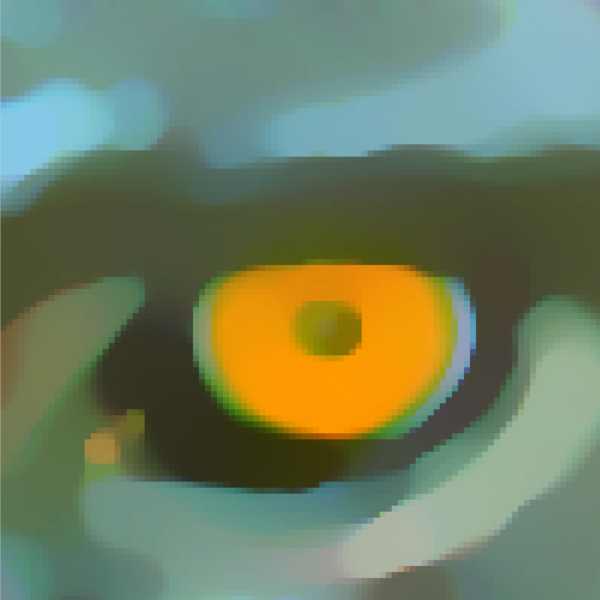}\\[.2ex]
\includegraphics[width=2.25cm,height=2.23cm]{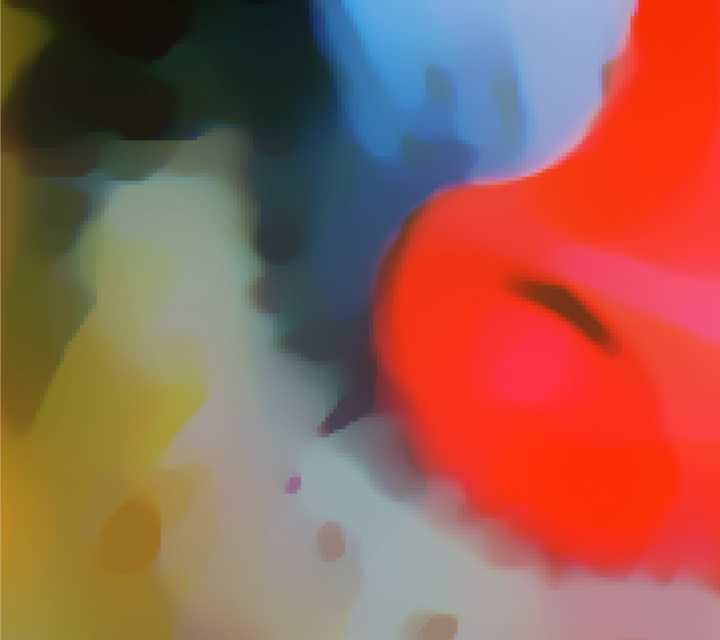}
\end{minipage}\\
\mbox{\scriptsize{(a) Original Image }} &
\mbox{\scriptsize{(b) Our constant $\mu$, $\lambda$ based scheme}} &
\mbox{\scriptsize{(c) Bresson et al~\cite{BE07}}} 
\end{array}\]

\[\begin{array}{ccc}
\begin{minipage}[c]{4cm}
\includegraphics[width=4cm,height=4.5cm]{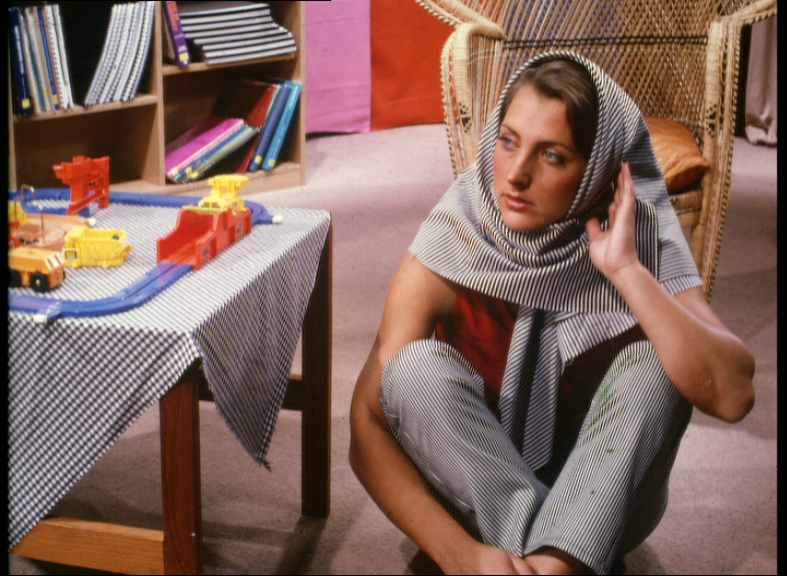}
\end{minipage}
&
\begin{minipage}[c]{4.cm}
\includegraphics[width=4cm,height=4.5cm]{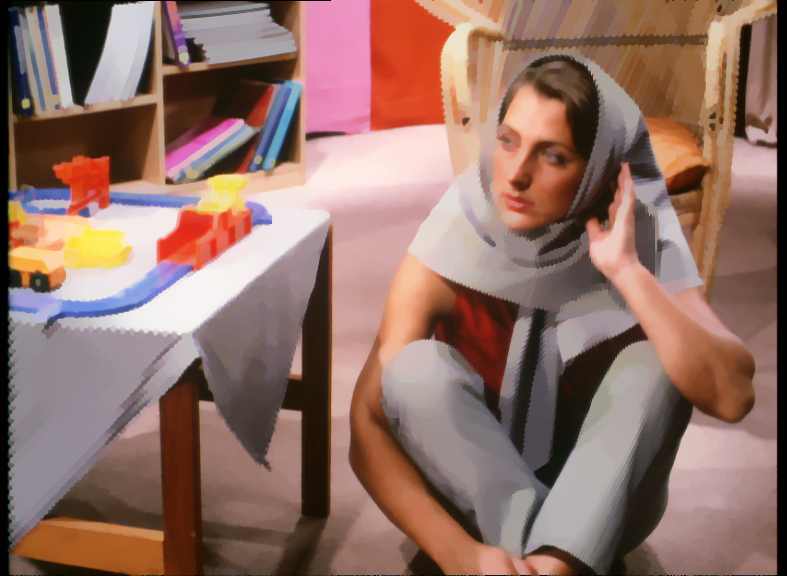}
\end{minipage}
\begin{minipage}[r]{2.3cm}
\includegraphics[width=2.25cm,height=2.23cm]{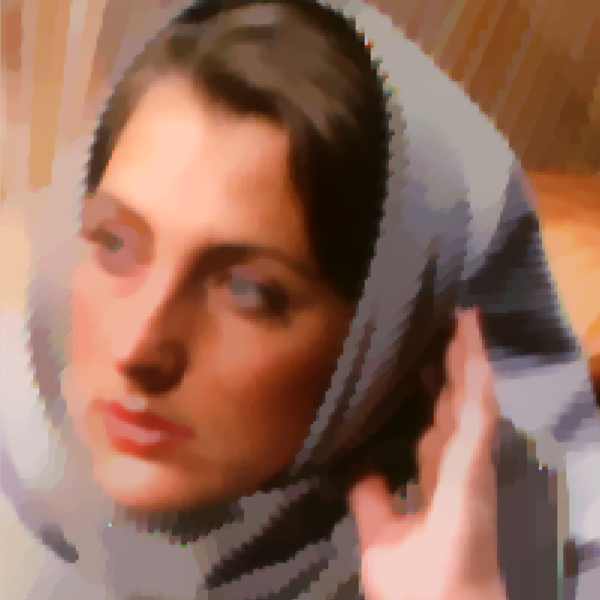}\\[.2ex]
\includegraphics[width=2.25cm,height=2.23cm]{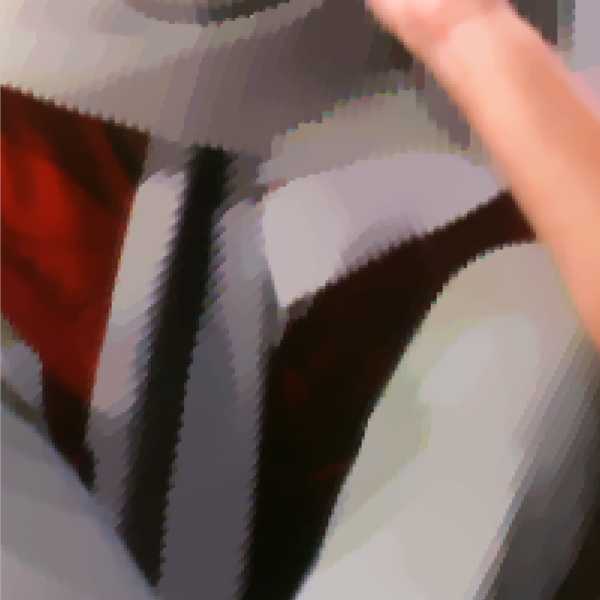}
\end{minipage}
&
\begin{minipage}[c]{4.cm}
\includegraphics[width=4cm,height=4.5cm]{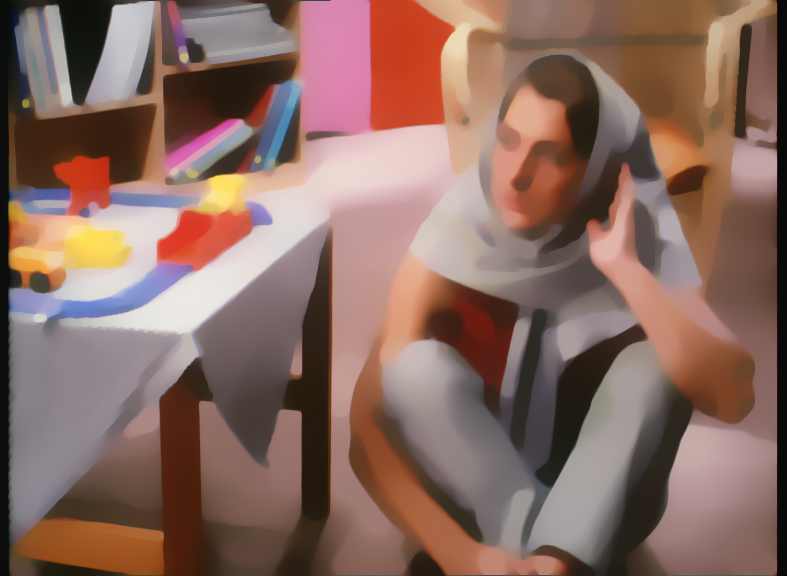}
\end{minipage}
\begin{minipage}[r]{2.3cm}
\includegraphics[width=2.25cm,height=2.23cm]{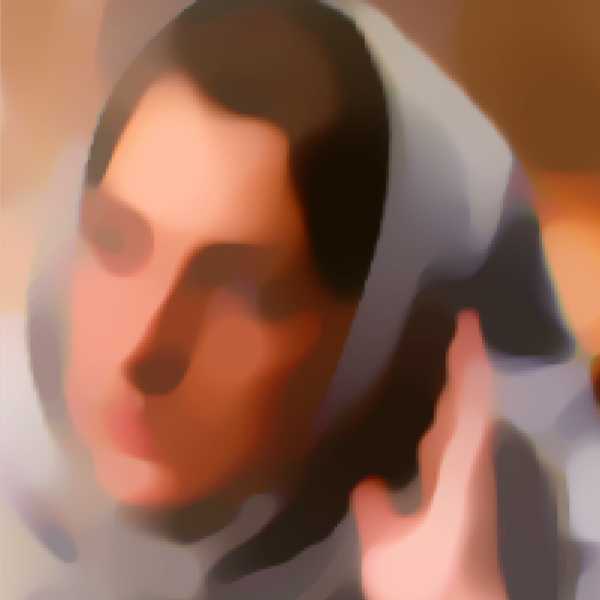}\\[.2ex]
\includegraphics[width=2.25cm,height=2.23cm]{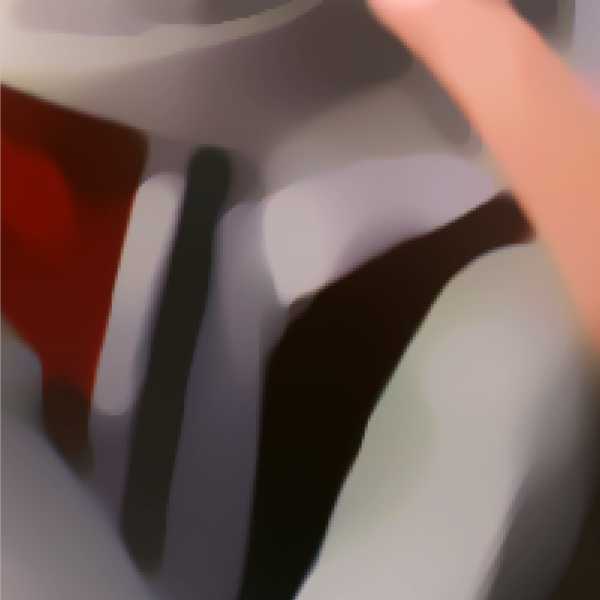}
\end{minipage}\\
\mbox{\scriptsize{(a) Original Image}} &
\mbox{\scriptsize{(b) Our adaptive $\mu_1$, constant $\lambda$ based scheme}} &
\mbox{\scriptsize{(c) Bresson et al~\cite{BE07}}} 
\end{array}\]
\caption{(Color online) Our diffusion constrained total variation scheme (stopping parameter $\epsilon=10^{-4}$) provides better edge preserving cartoon component $u$ when compared to the traditional TV regularization model~\cite{BE07}. Even with constant $\mu$, $\lambda$ the proposed scheme provides better results (see top row (b)). The crop regions highlight that the proposed scheme provides better preservation of large scale textures compared to~\cite{BE07} model. Best viewed electronically, zoomed in.}\label{fig:CTE_Vs_Bresson}
\end{figure*}
\begin{figure*}
\centering
	\includegraphics[width=3.cm, height=3cm]{CTE_images_barbara}
	\includegraphics[width=3.cm, height=3cm]{CTE_images_mandrill}
	\includegraphics[width=3.cm, height=3cm]{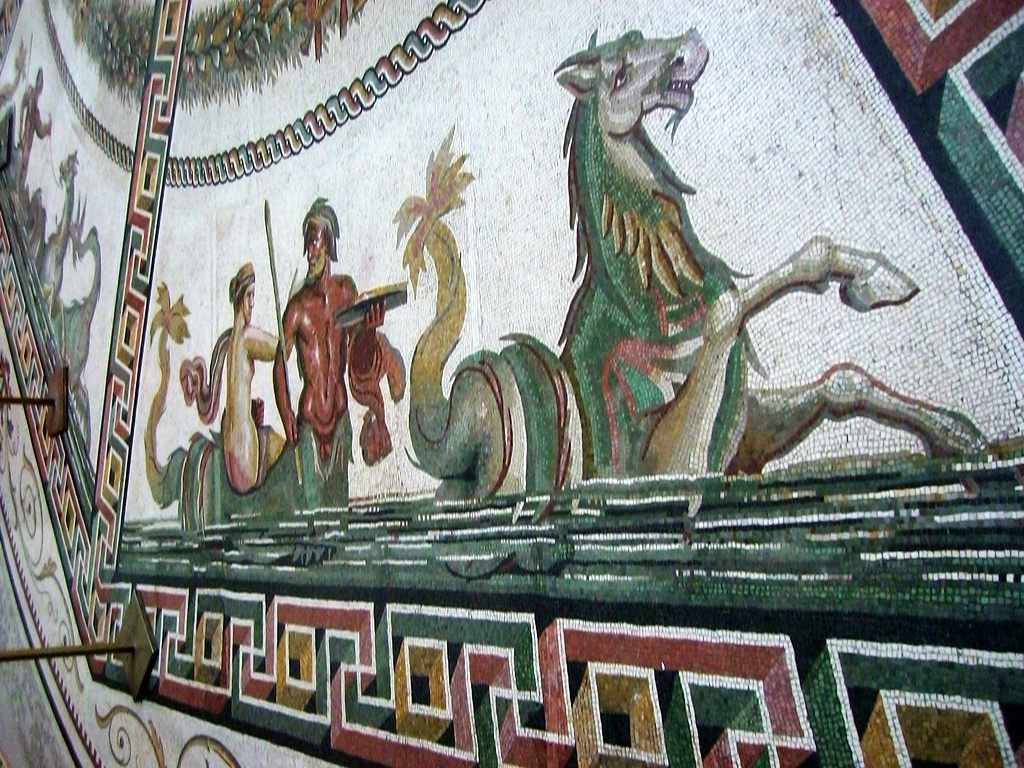}
	\includegraphics[width=3.cm, height=3cm]{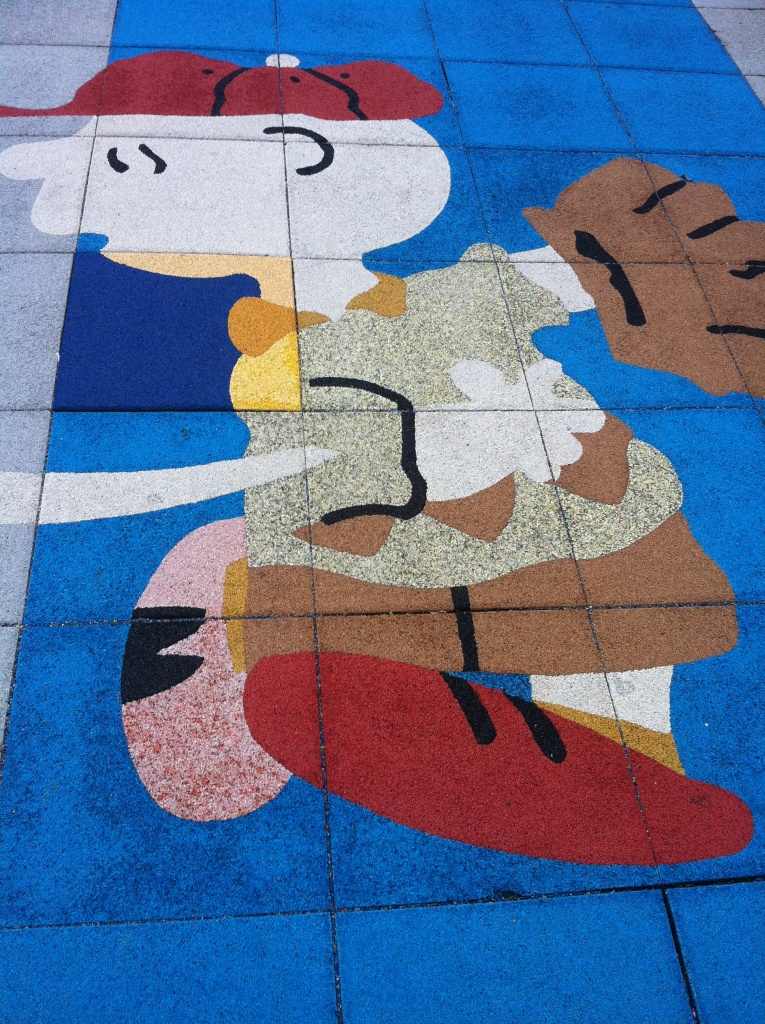}
	\includegraphics[width=3.cm, height=3cm]{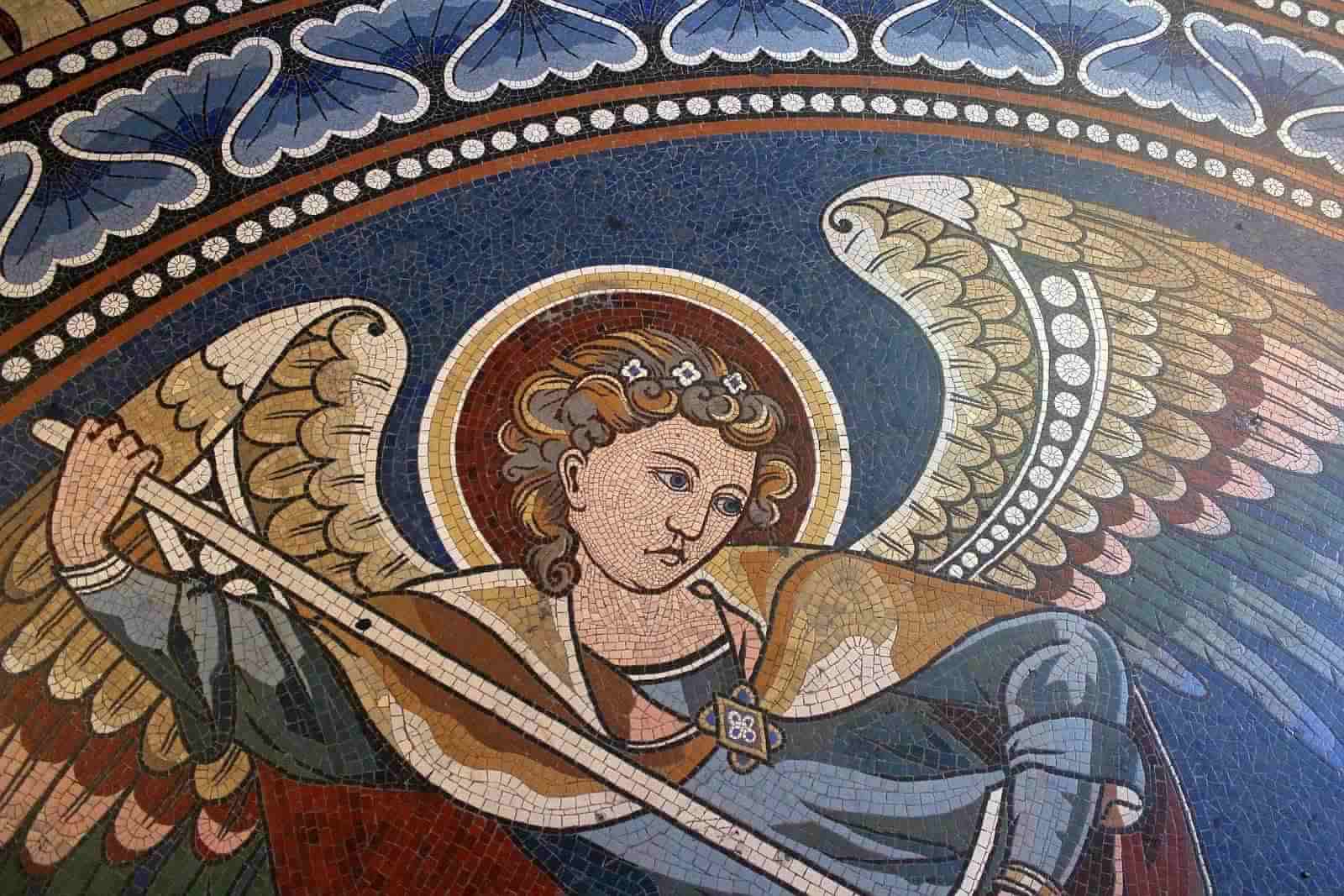}\\
	\includegraphics[width=3cm, height=3cm]{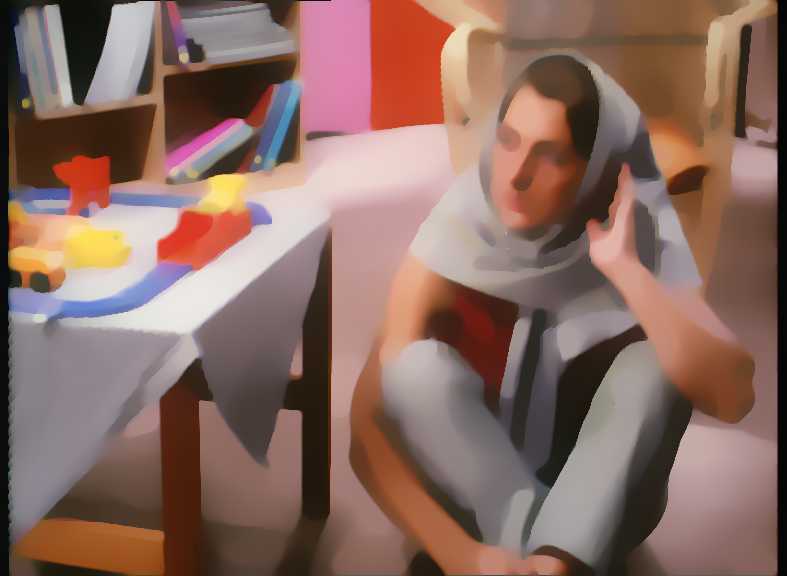}
	\includegraphics[width=3cm, height=3cm]{CTE_Results_Lambda_Constant100_mandril_u}
	\includegraphics[width=3cm, height=3cm]{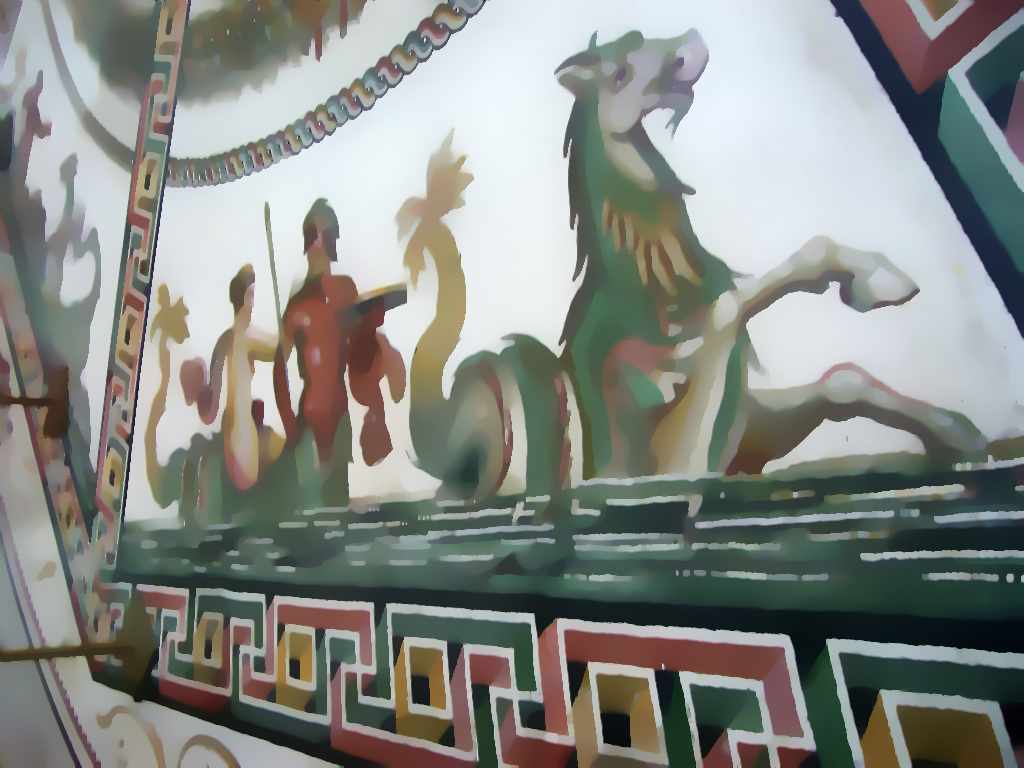}
	\includegraphics[width=3cm, height=3cm]{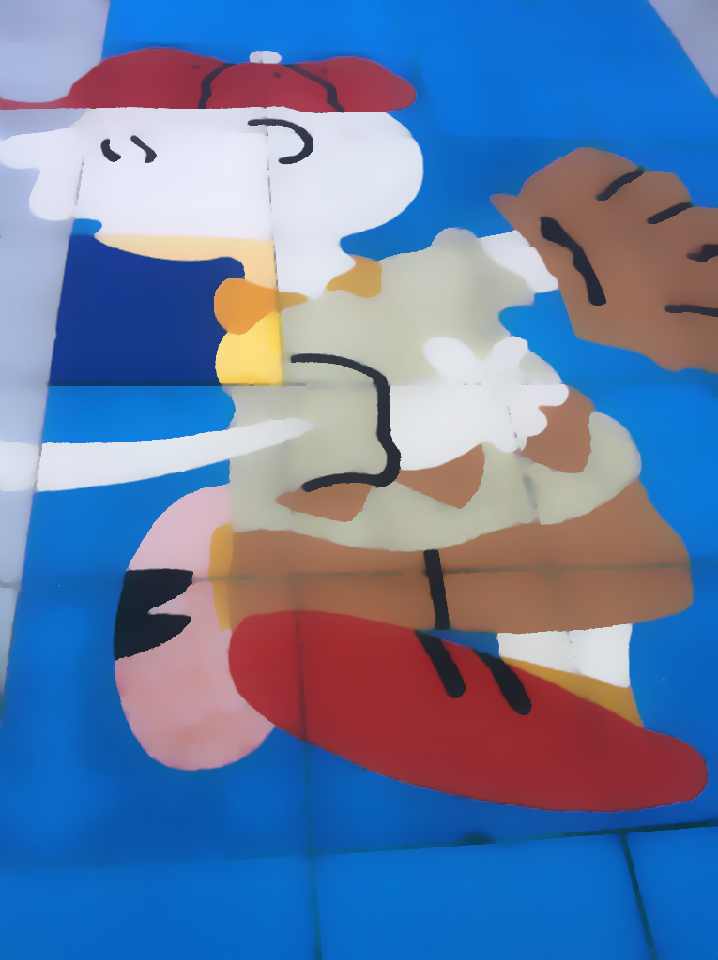}	
	\includegraphics[width=3cm, height=3cm]{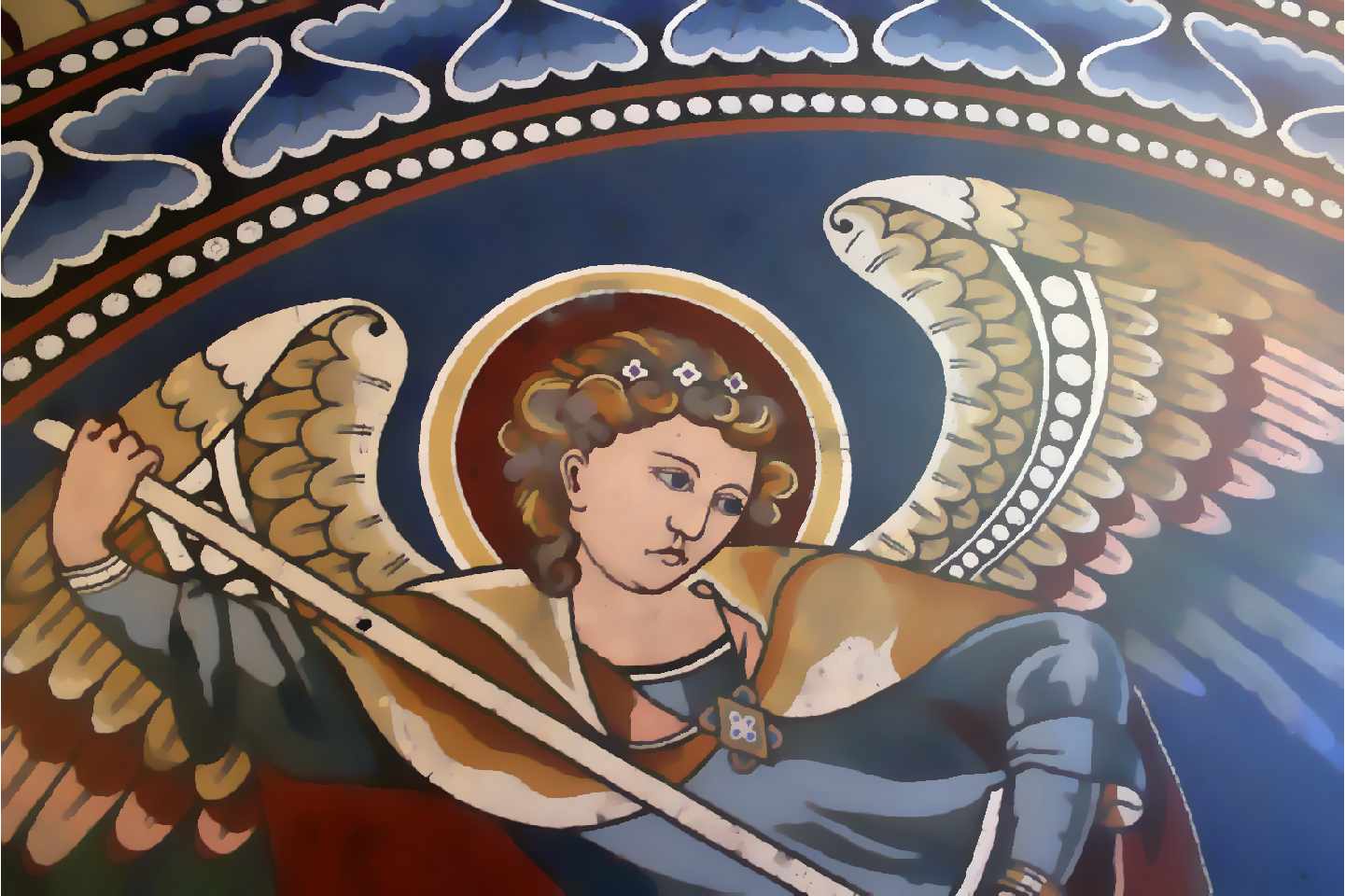}\\	
	\includegraphics[width=3cm, height=3cm]{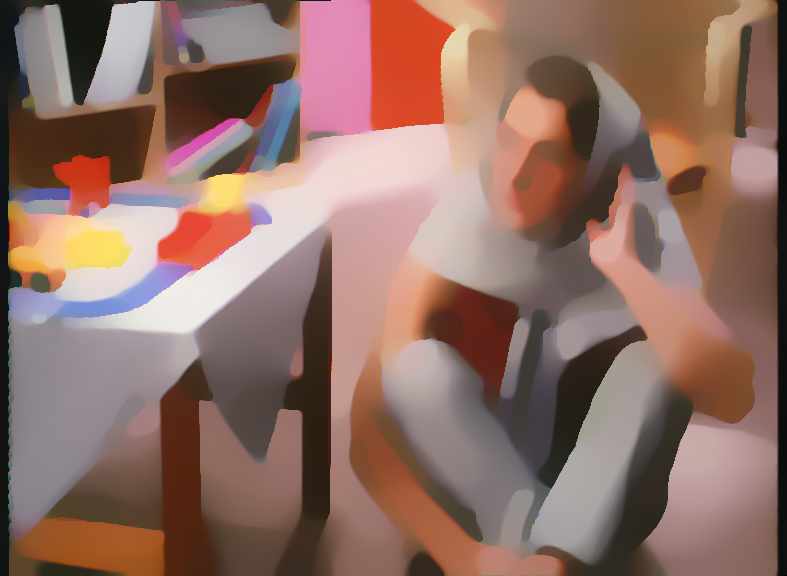}
	\includegraphics[width=3cm, height=3cm]{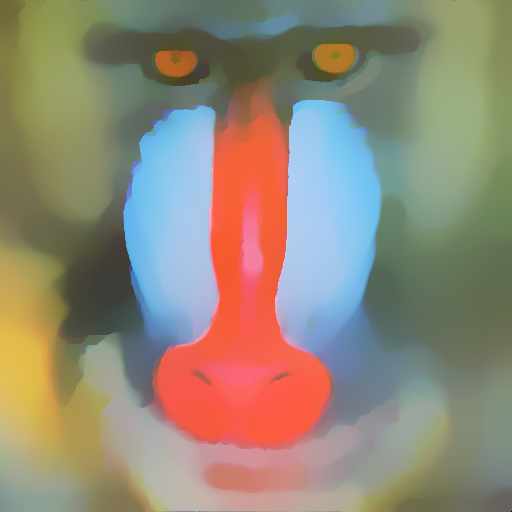}
	\includegraphics[width=3cm, height=3cm]{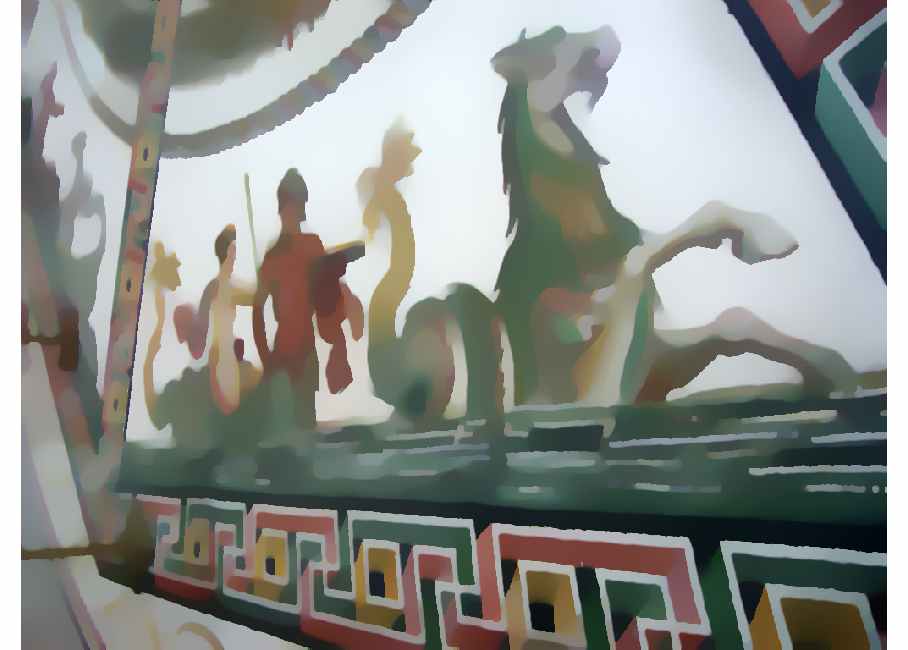}
	\includegraphics[width=3cm, height=3cm]{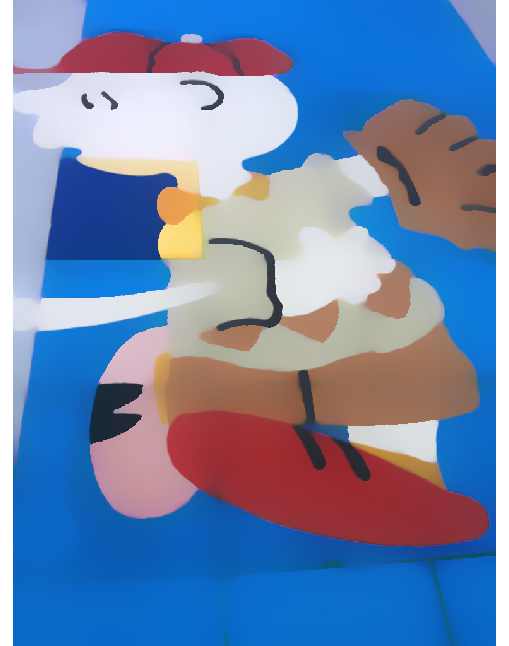}	
	\includegraphics[width=3cm, height=3cm]{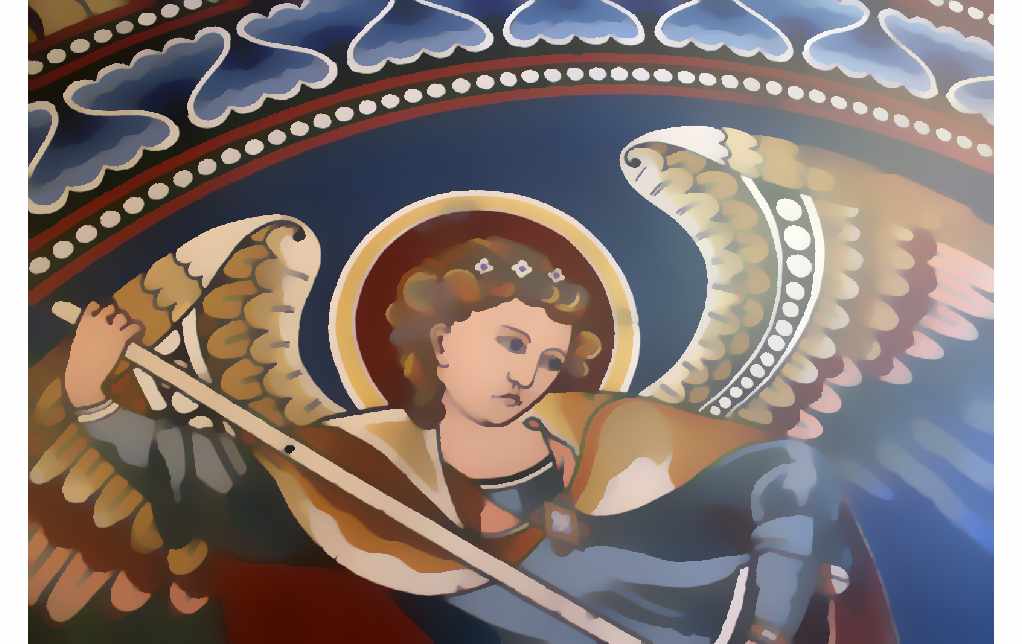}\\
	\includegraphics[width=3cm, height=3cm]{CTE_Results_Lambda_Variable400_barbara_u}
	\includegraphics[width=3cm, height=3cm]{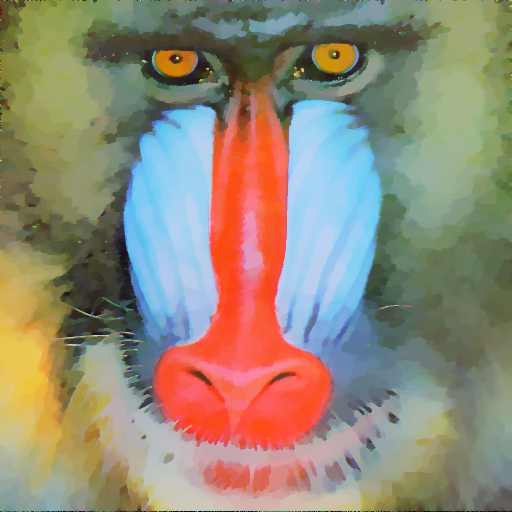}
	\includegraphics[width=3cm, height=3cm]{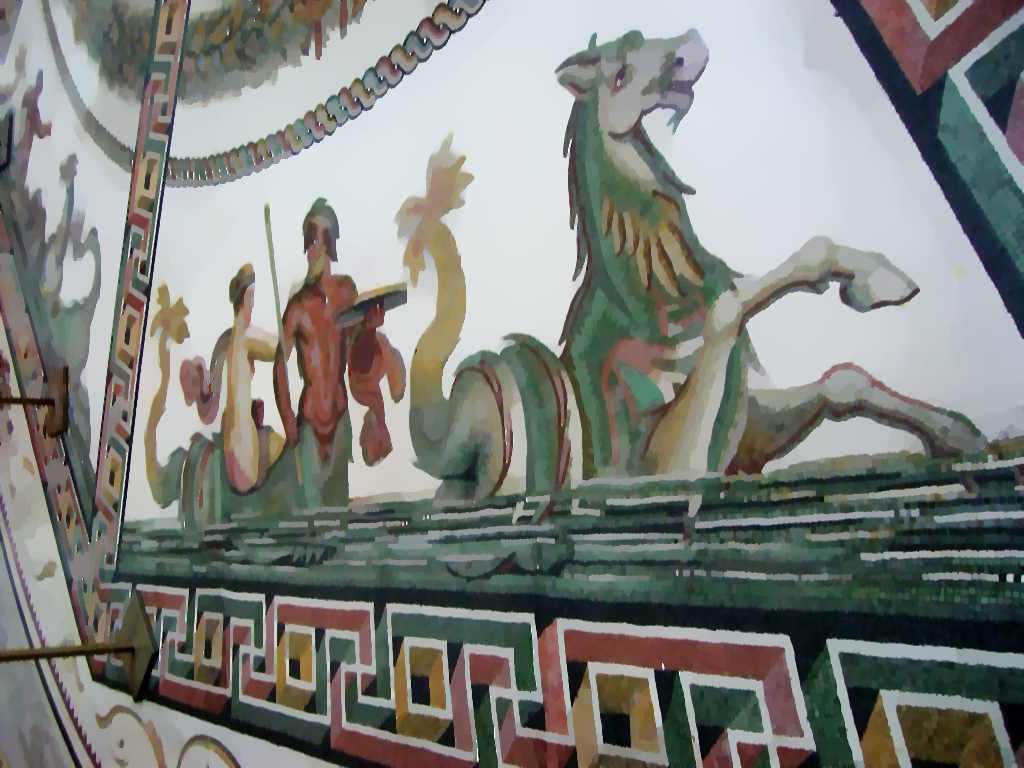}
	\includegraphics[width=3cm, height=3cm]{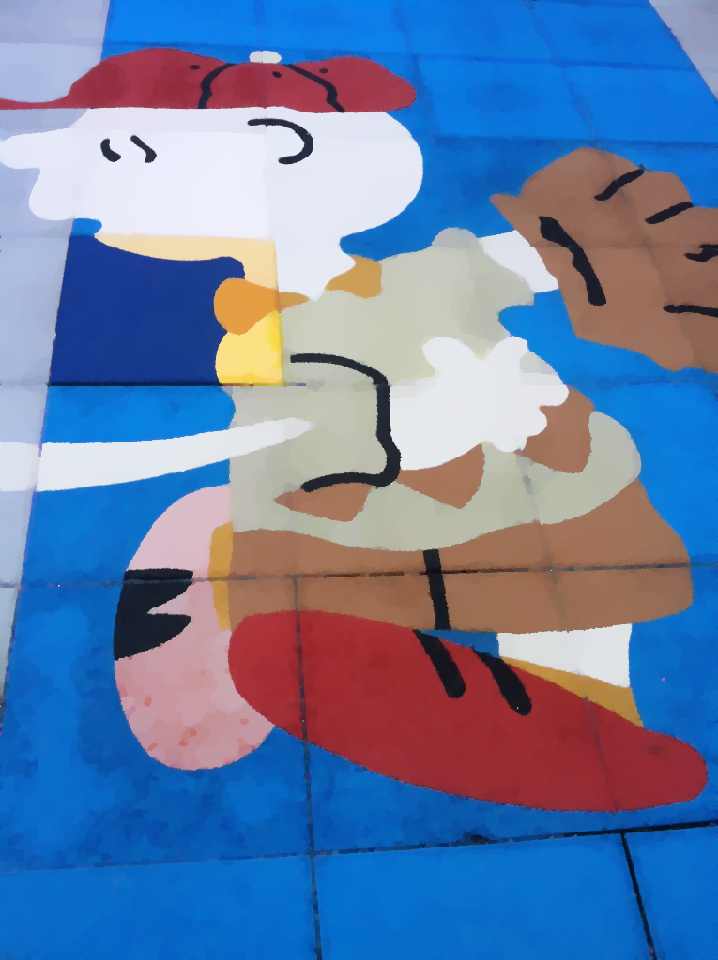}	
	\includegraphics[width=3cm, height=3cm]{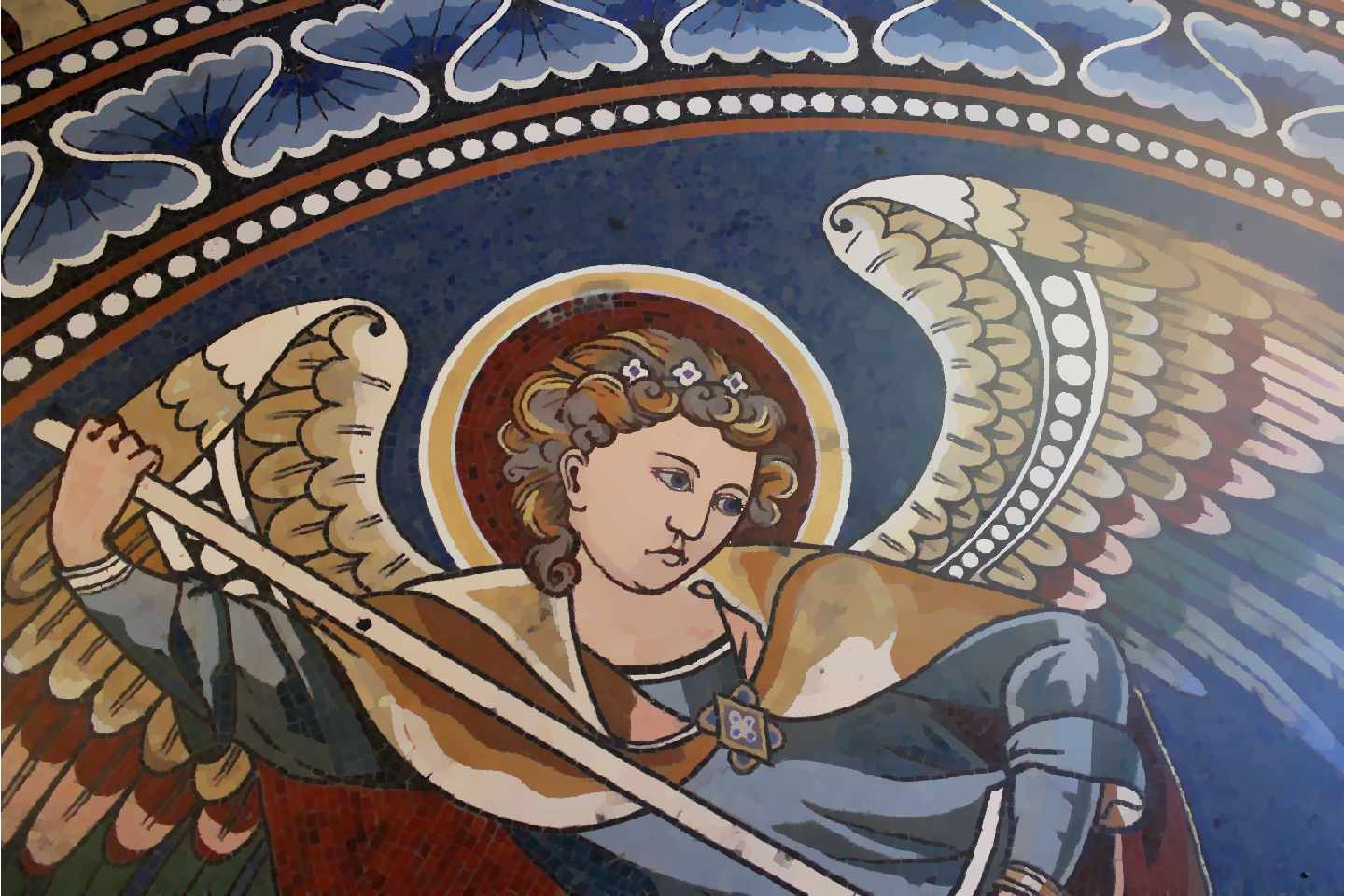}\\		

	\caption{(Color online) Effect of constant $\mu$ in color image decomposition using our coupled scheme on the cartoon ($u$) component.
	First row: Original color RGB images. Cartoon component result for constant $\mu$, $\lambda$ with stopping parameter,
	Second row:  $\epsilon=10^{-4}$, and Third row: $\epsilon=10^{-6}$. As can be seen decreasing the stopping parameter removes more texture details and provides piecewise constant cartoon image. Last row: Shows the proposed scheme results with adaptive $\mu_1$, see Eqn.~\eqref{E:CTElambda} and constant $\lambda$ (stopping parameter $\epsilon=10^{-4}$) . Best viewed electronically, zoomed in.}\label{fig:cartoons}
\end{figure*}
\begin{figure*}
\centering
	\includegraphics[width=3cm, height=3cm]{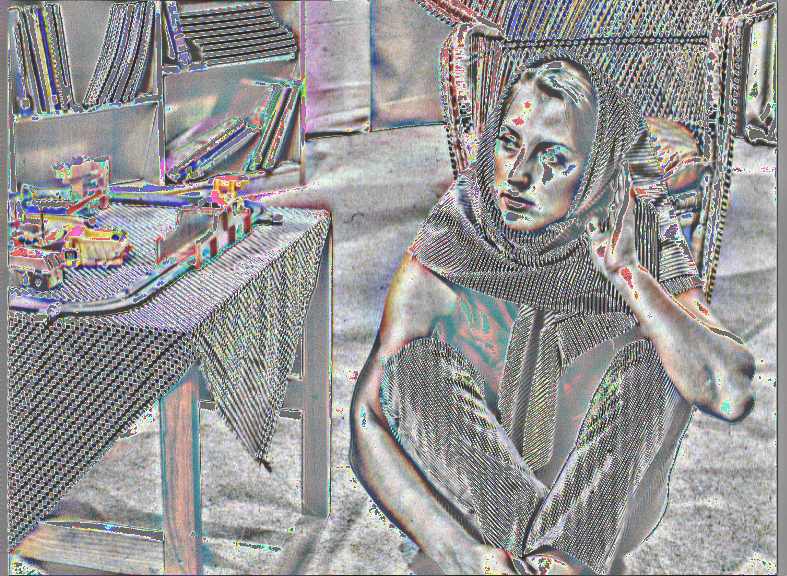}
	\includegraphics[width=3cm, height=3cm]{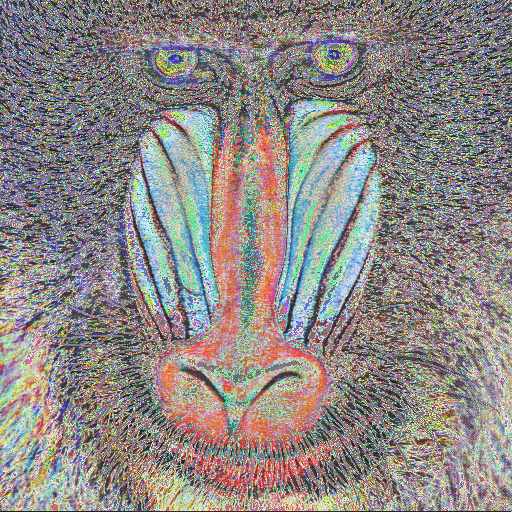}
	\includegraphics[width=3cm, height=3cm]{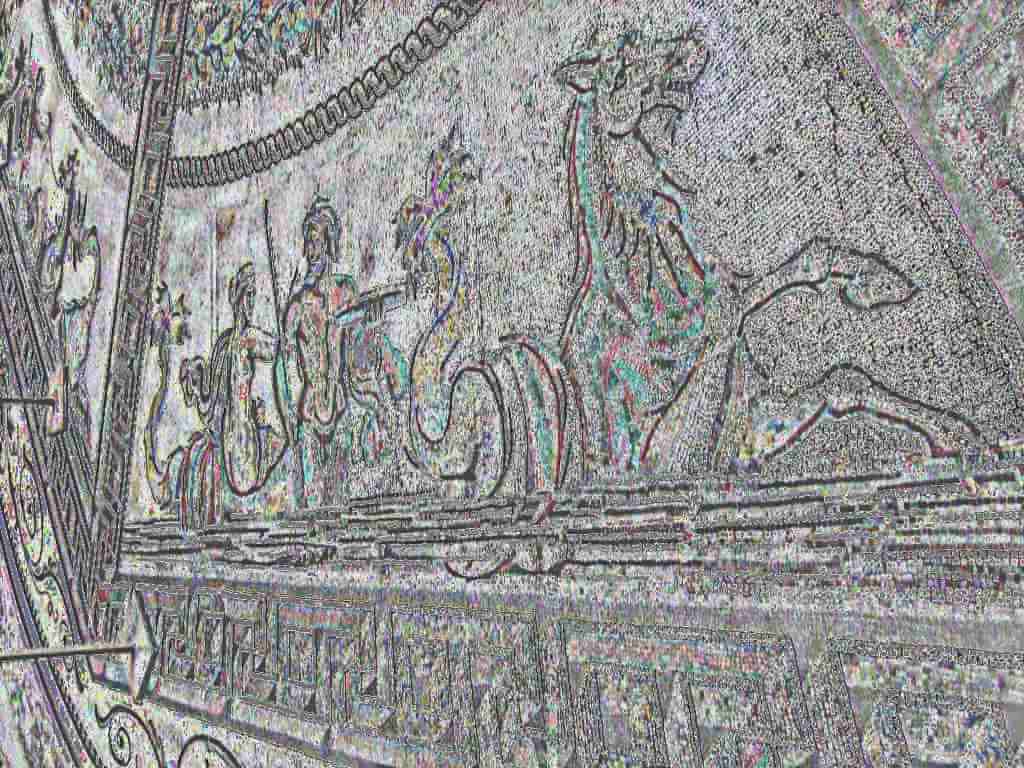}
	\includegraphics[width=3cm, height=3cm]{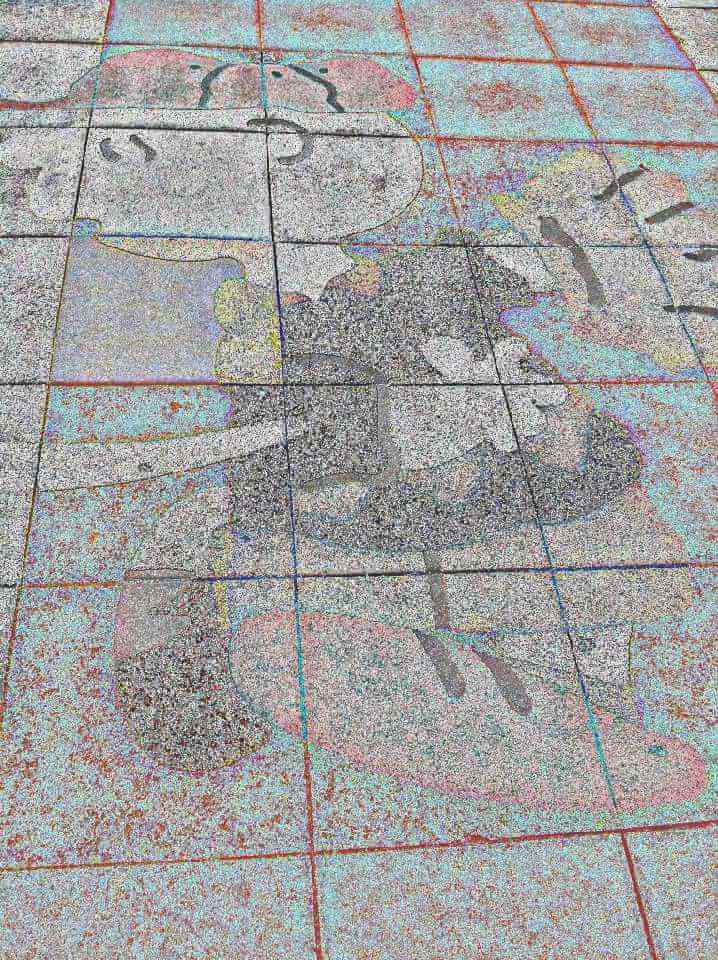}	
	\includegraphics[width=3cm, height=3cm]{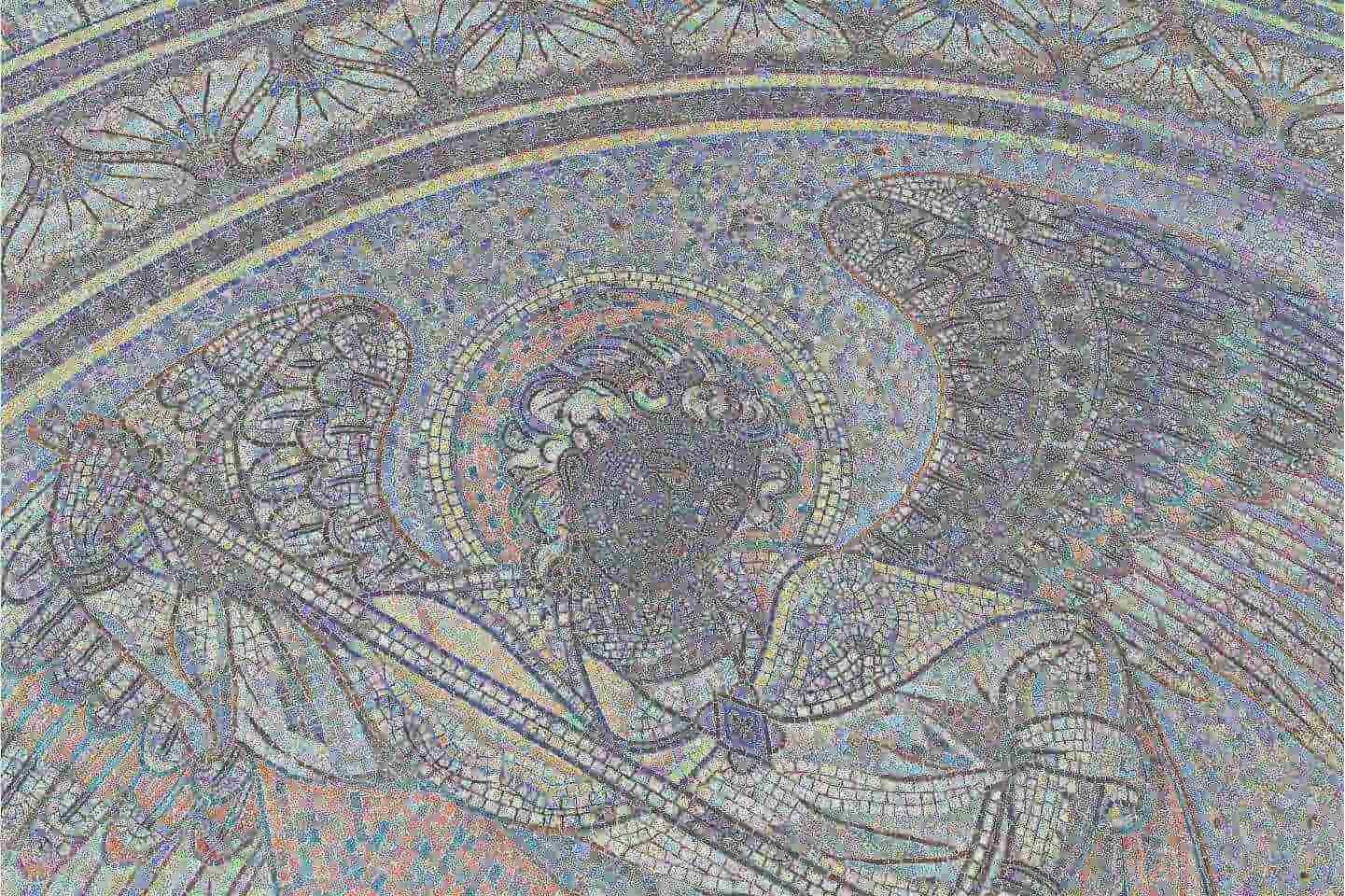}\\	
	\includegraphics[width=3cm, height=3cm]{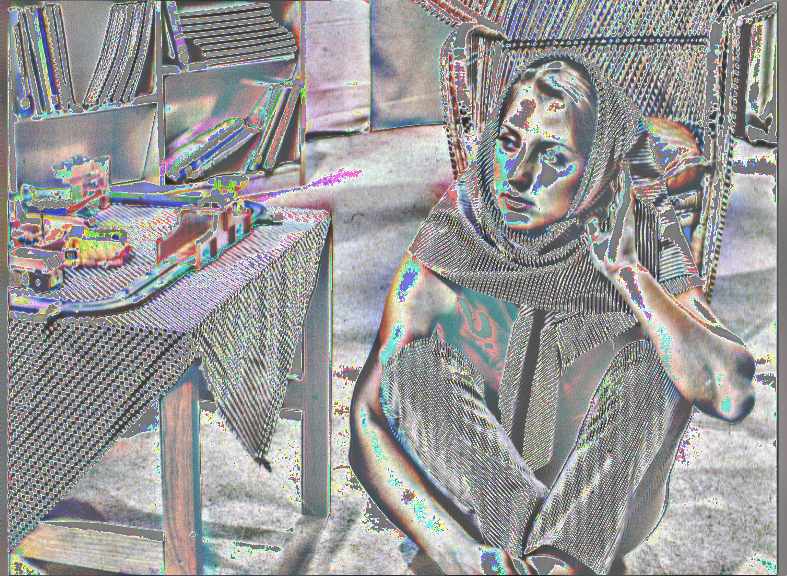}
	\includegraphics[width=3cm, height=3cm]{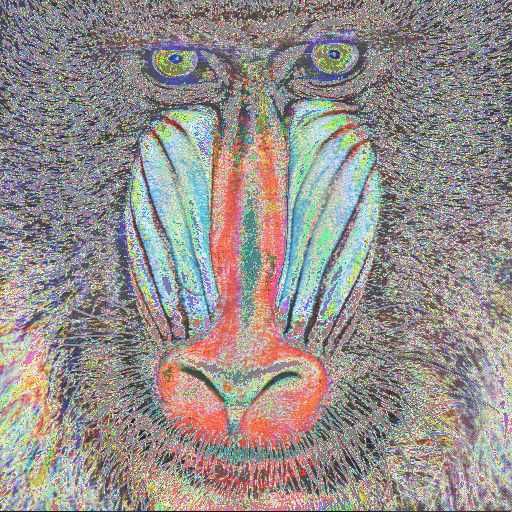}
	\includegraphics[width=3cm, height=3cm]{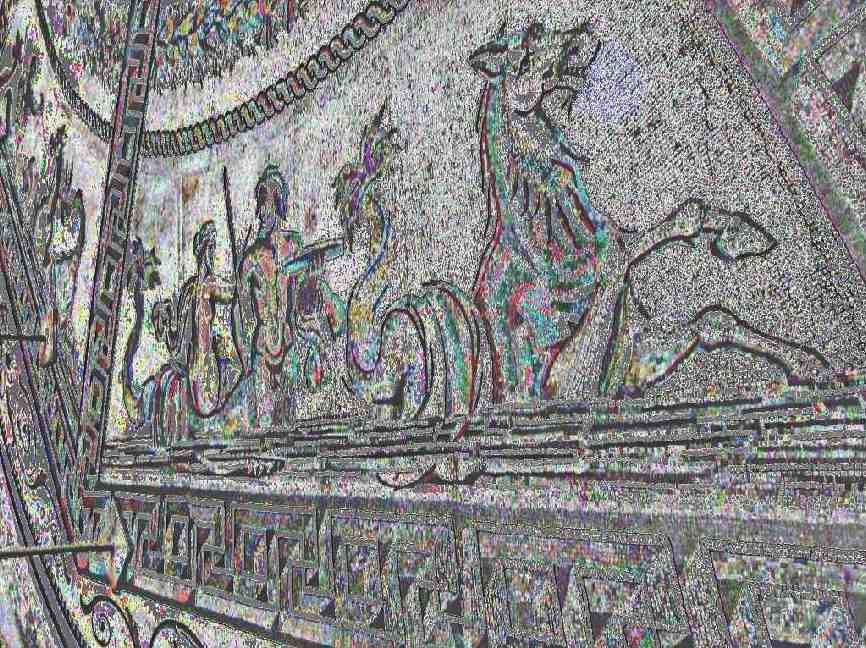}
	\includegraphics[width=3cm, height=3cm]{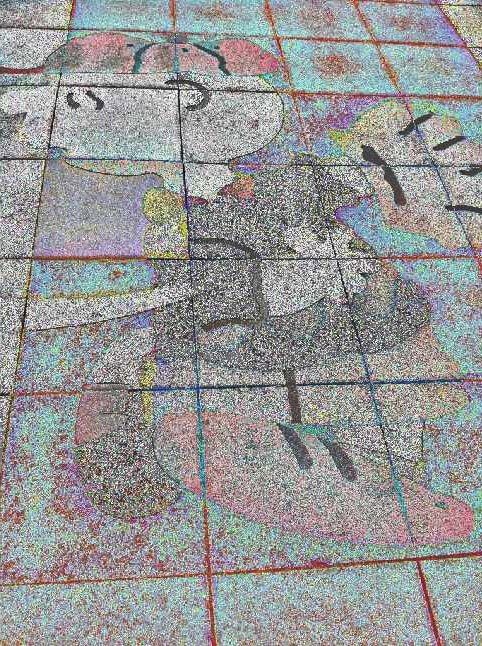}	
	\includegraphics[width=3cm, height=3cm]{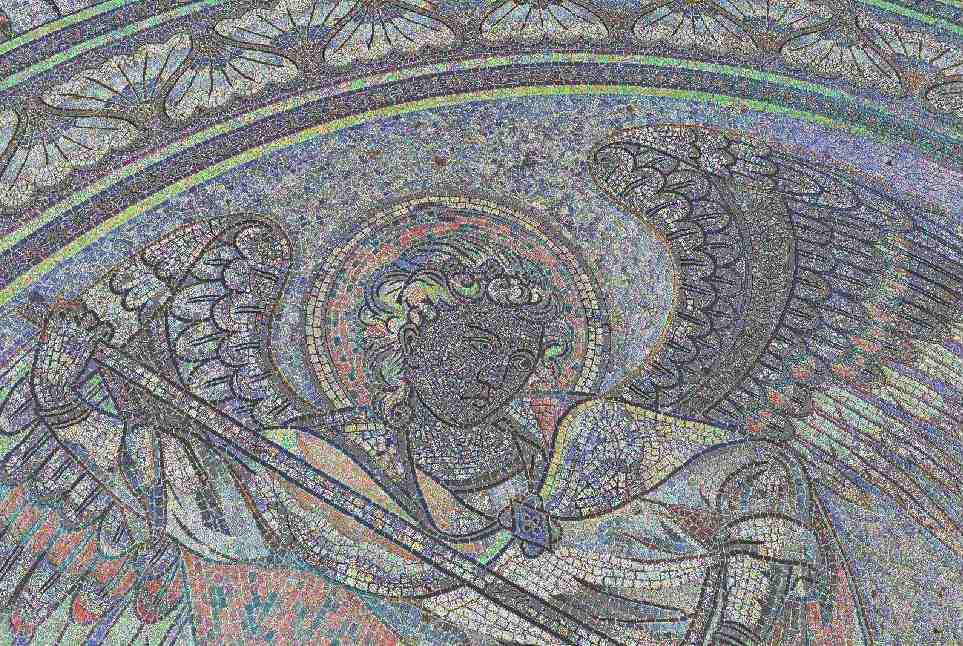}\\
	\includegraphics[width=3cm, height=3cm]{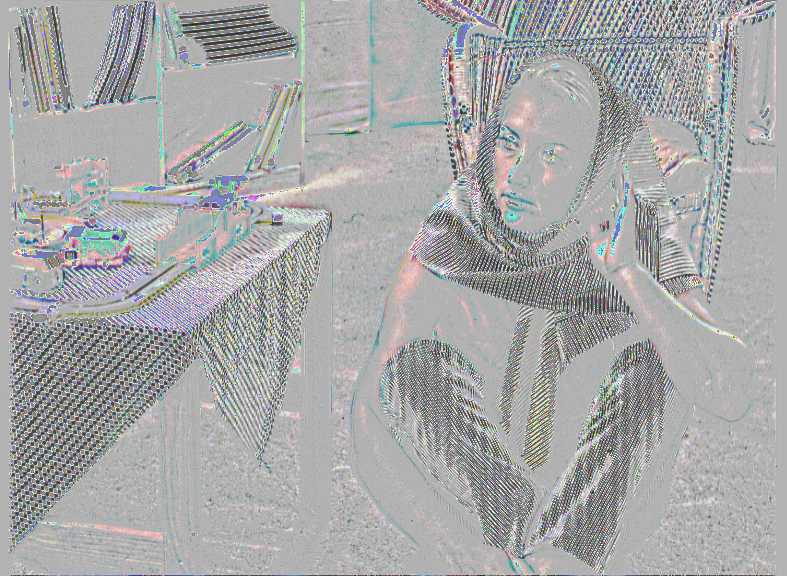}
	\includegraphics[width=3cm, height=3cm]{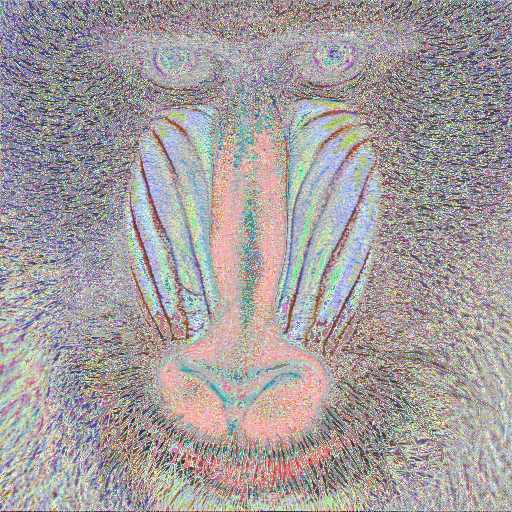}
	\includegraphics[width=3cm, height=3cm]{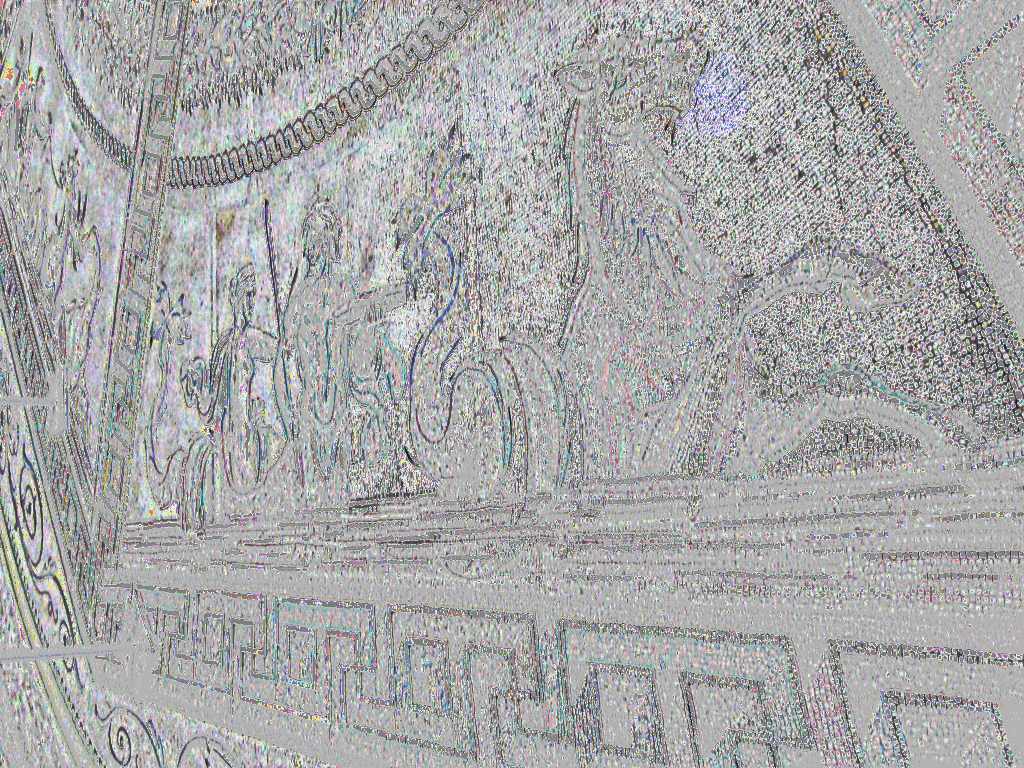}
	\includegraphics[width=3cm, height=3cm]{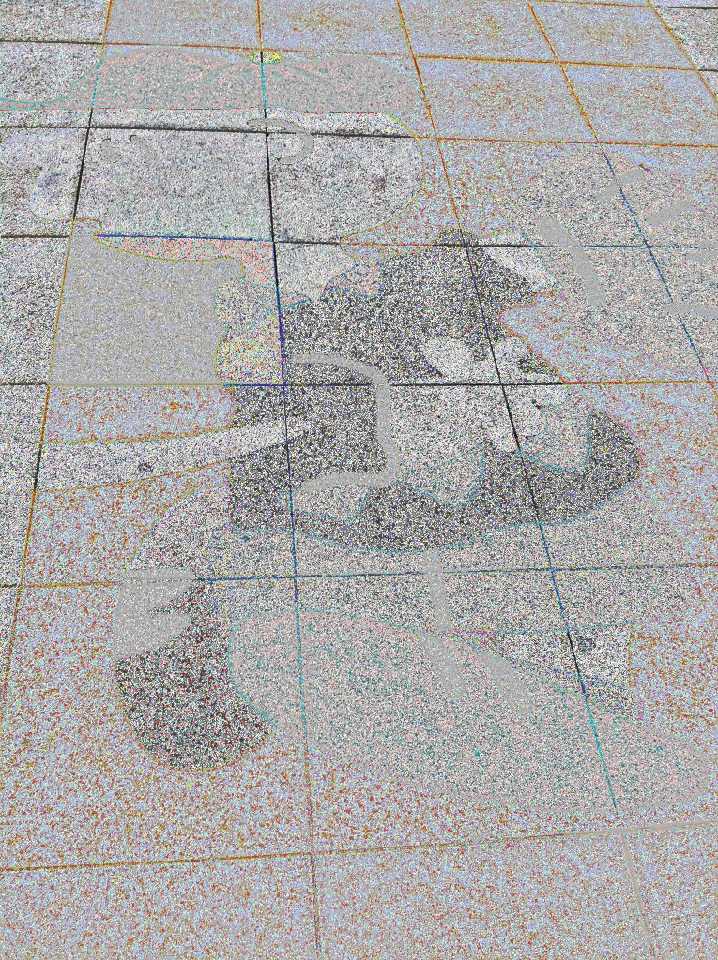}	
	\includegraphics[width=3cm, height=3cm]{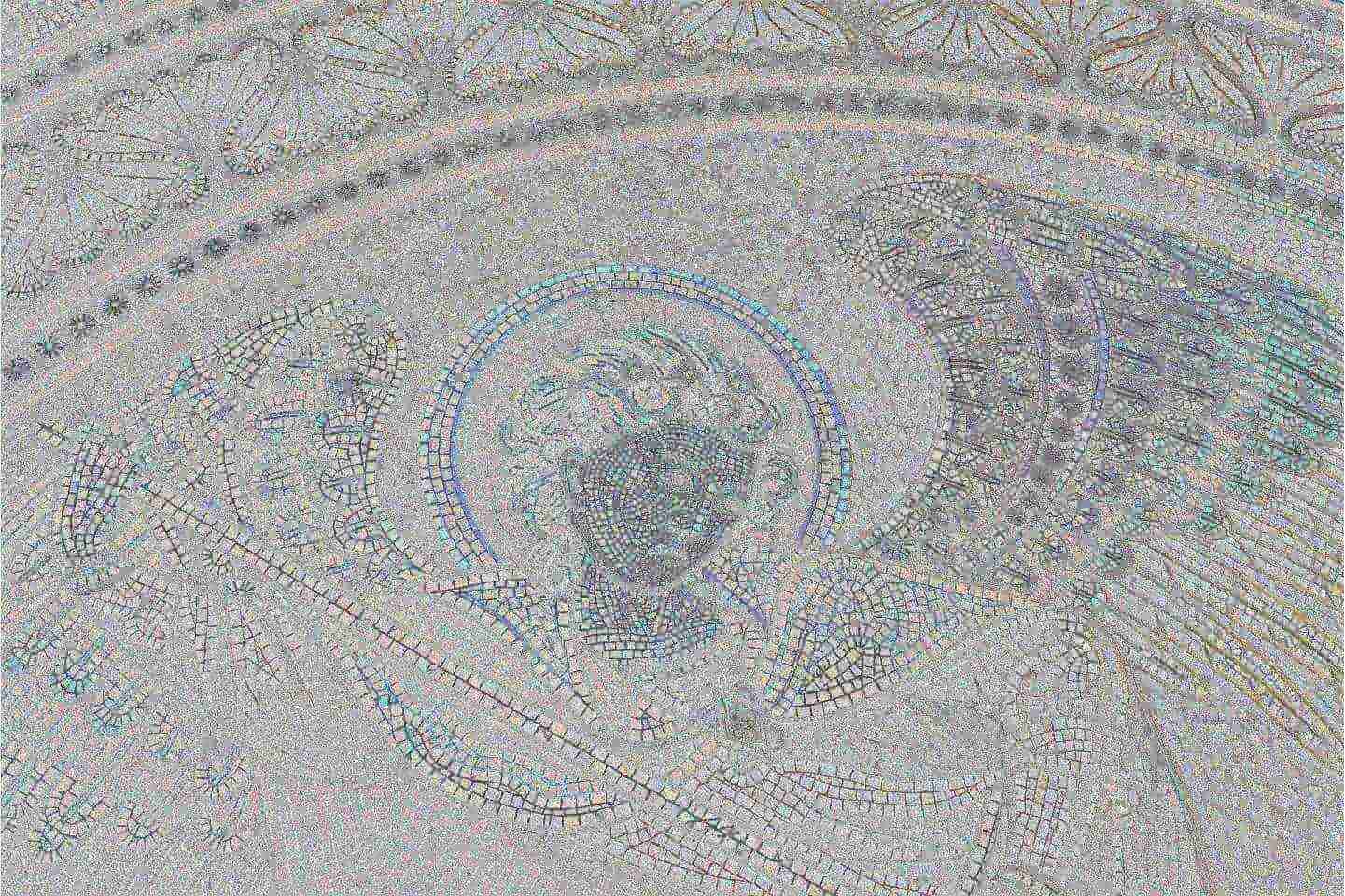}\\		

	\caption{(Color online) Corresponding texture components  $v$. Arrangement is as in Figure~\ref{fig:cartoons}. Best viewed electronically, zoomed in.}\label{fig:textures}
\end{figure*}
\begin{figure*}
\centering
	\includegraphics[width=3cm, height=3cm]{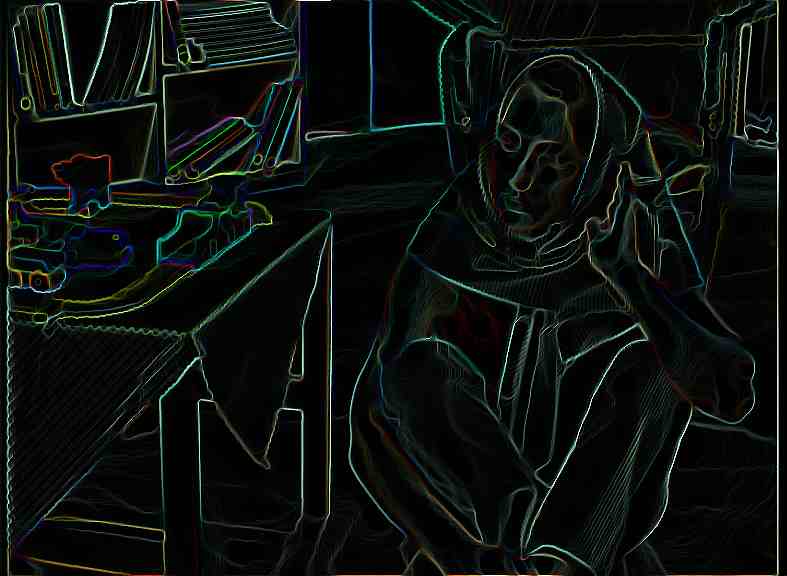}
	\includegraphics[width=3cm, height=3cm]{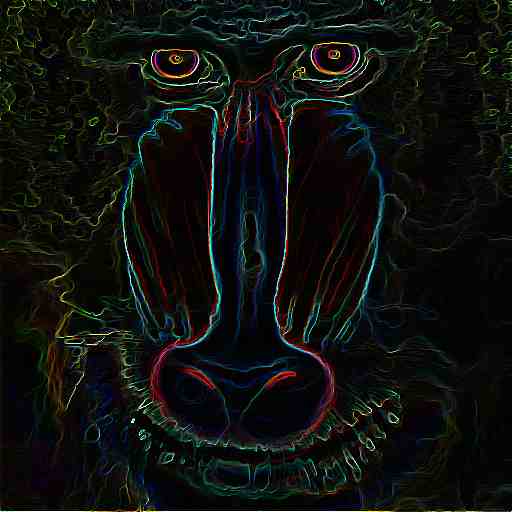}
	\includegraphics[width=3cm, height=3cm]{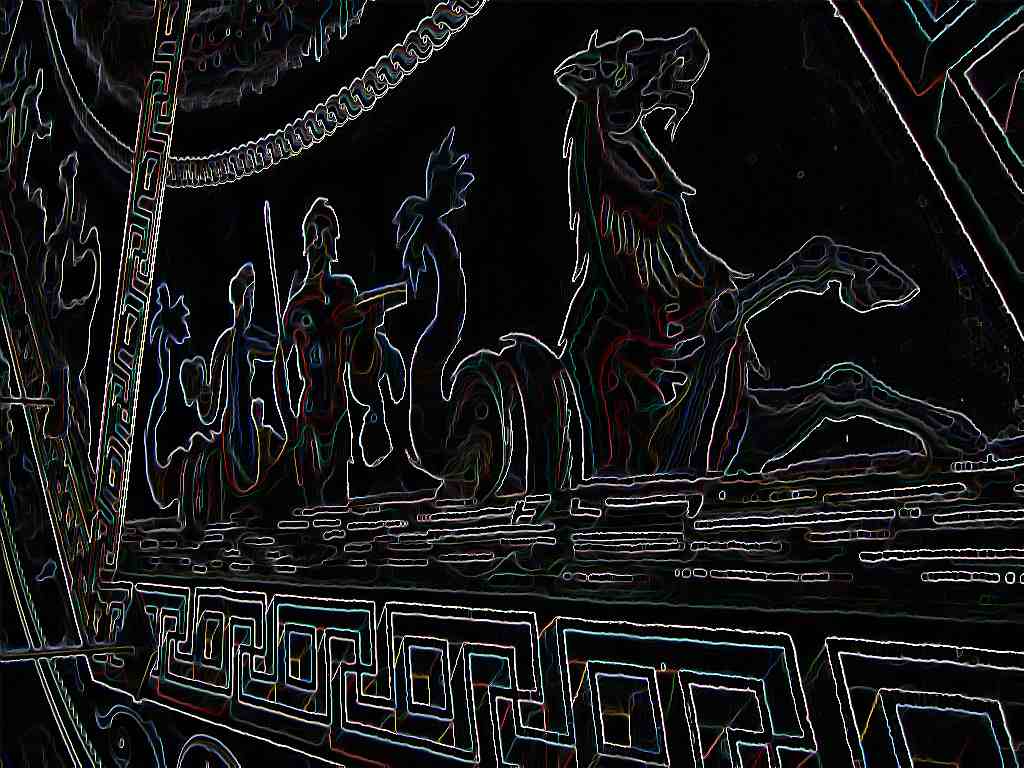}
	\includegraphics[width=3cm, height=3cm]{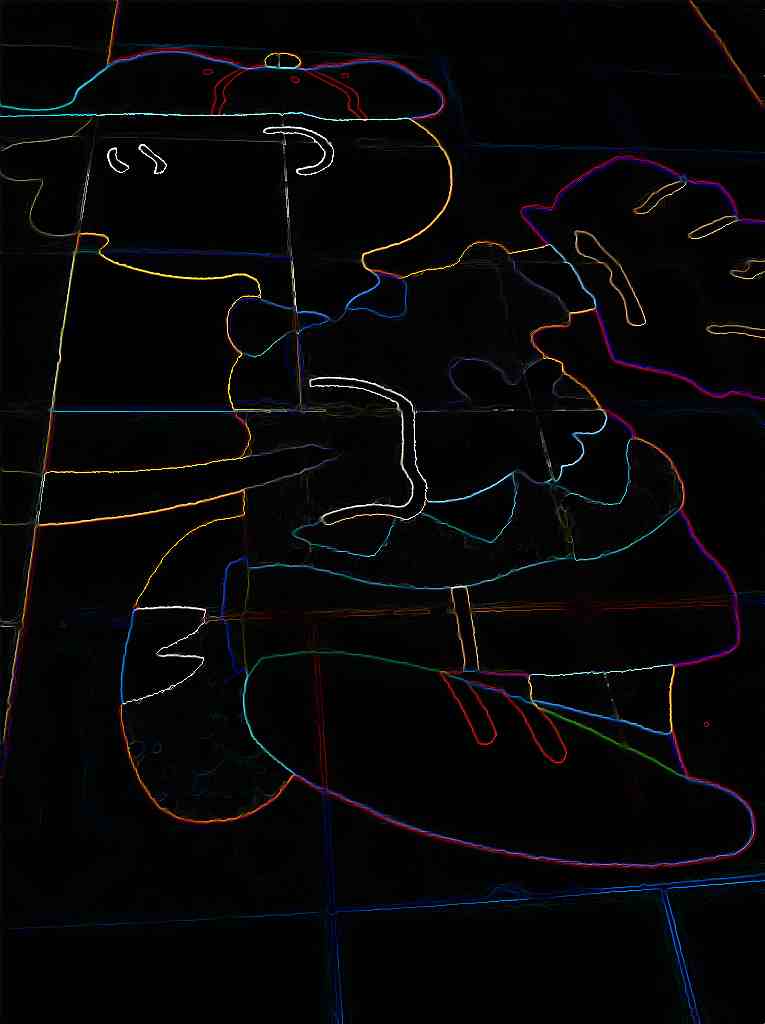}	
	\includegraphics[width=3cm, height=3cm]{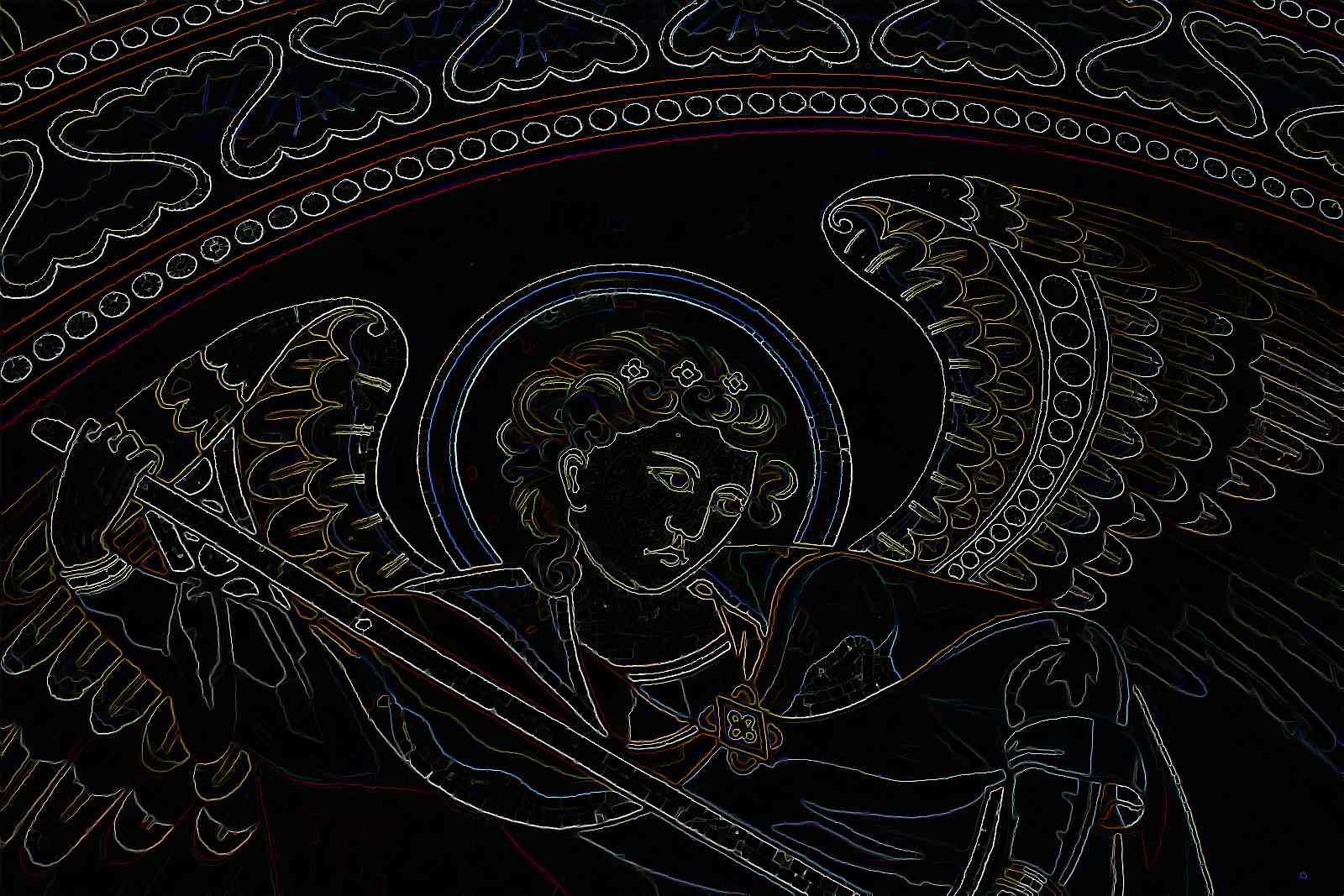}\\	
	\includegraphics[width=3cm, height=3cm]{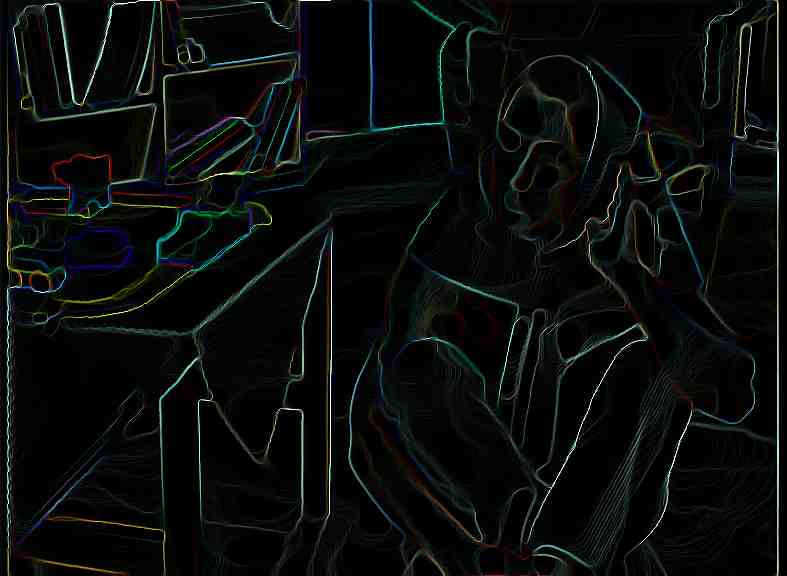}
	\includegraphics[width=3cm, height=3cm]{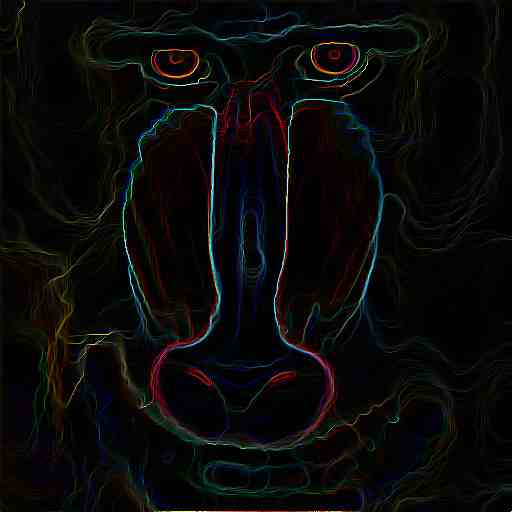}
	\includegraphics[width=3cm, height=3cm]{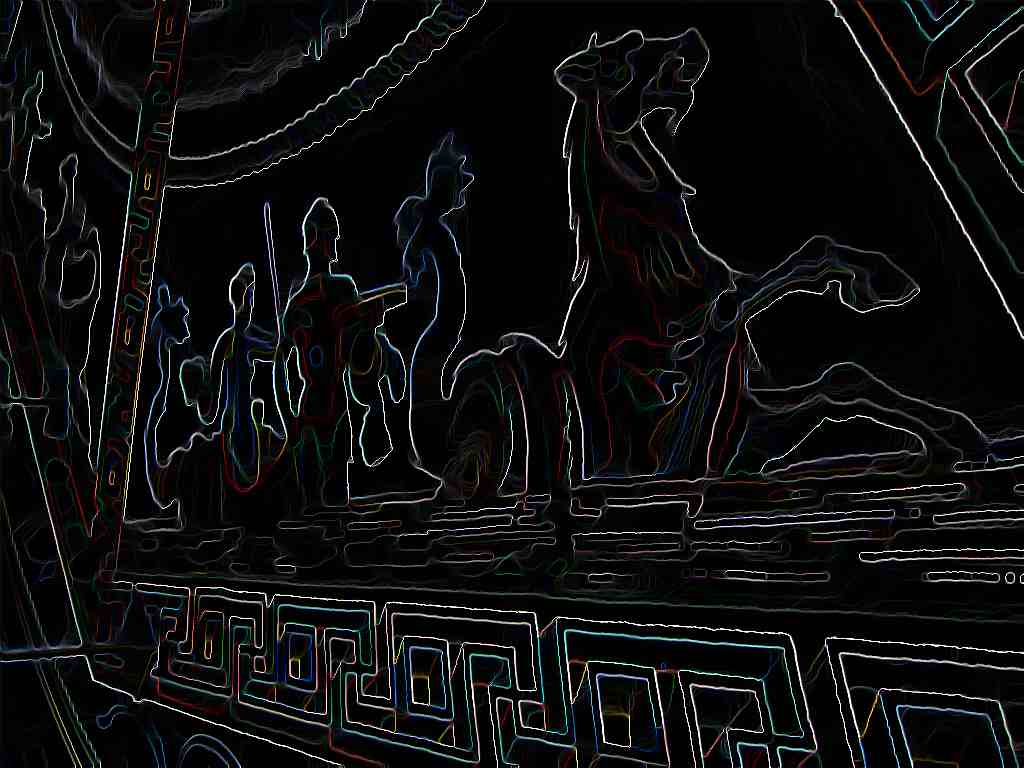}
	\includegraphics[width=3cm, height=3cm]{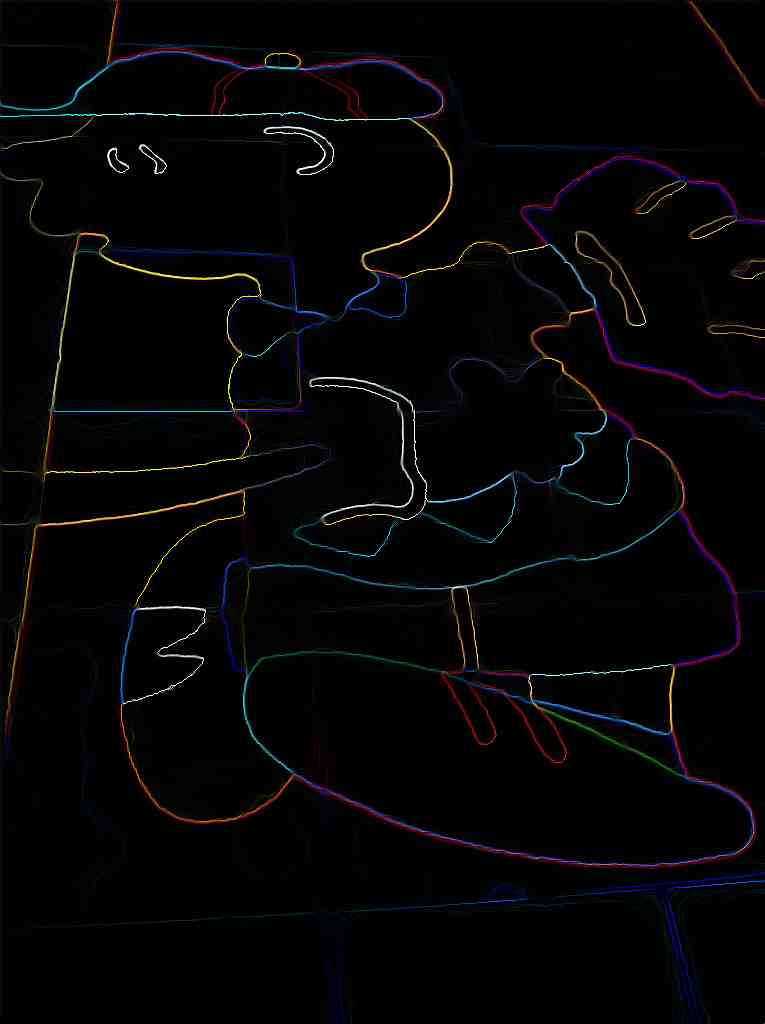}	
	\includegraphics[width=3cm, height=3cm]{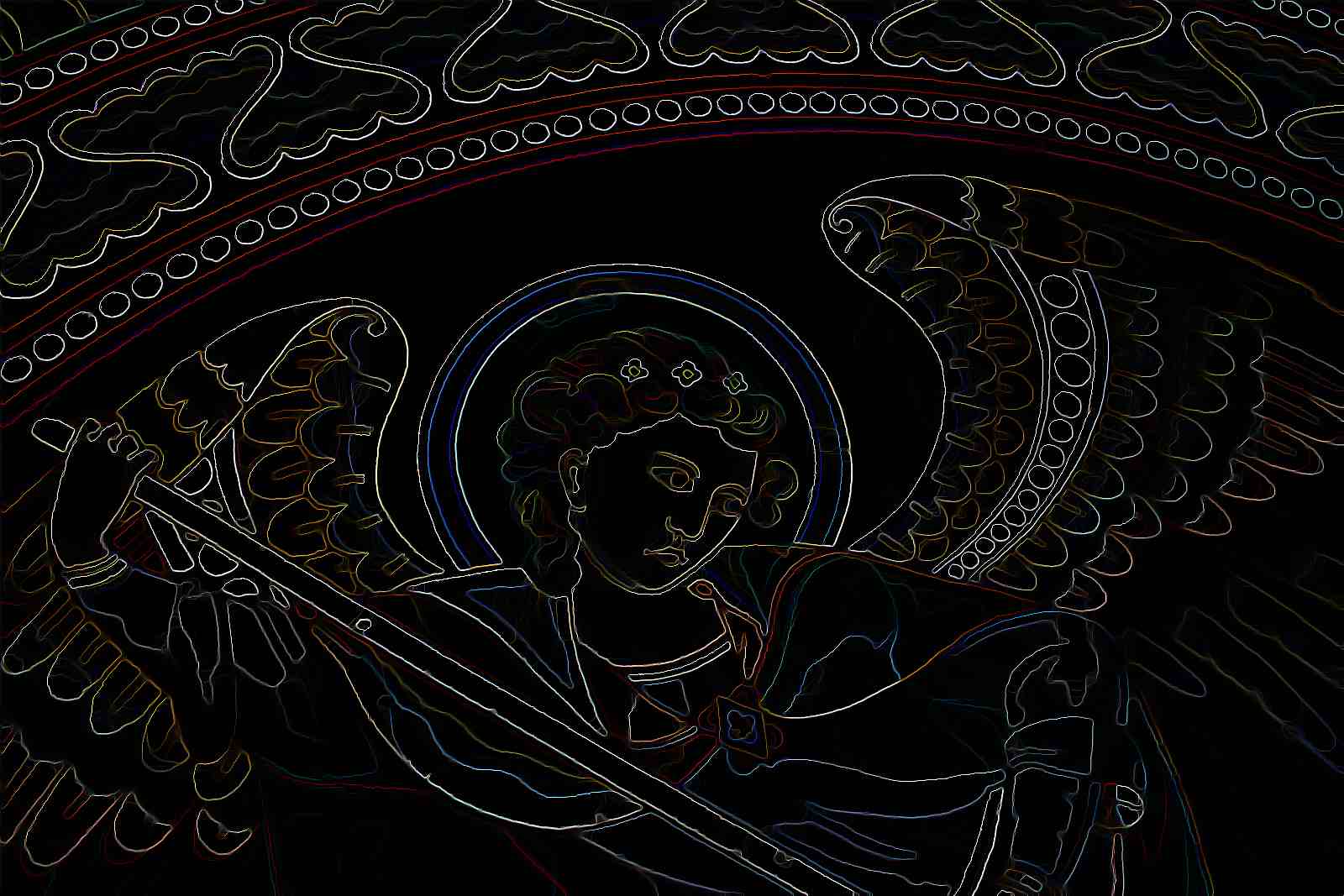}\\
	\includegraphics[width=3cm, height=3cm]{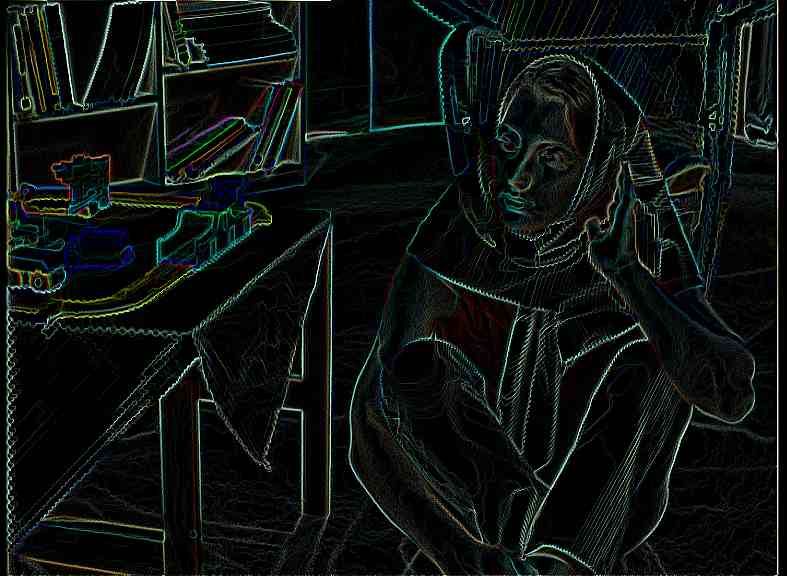}
	\includegraphics[width=3cm, height=3cm]{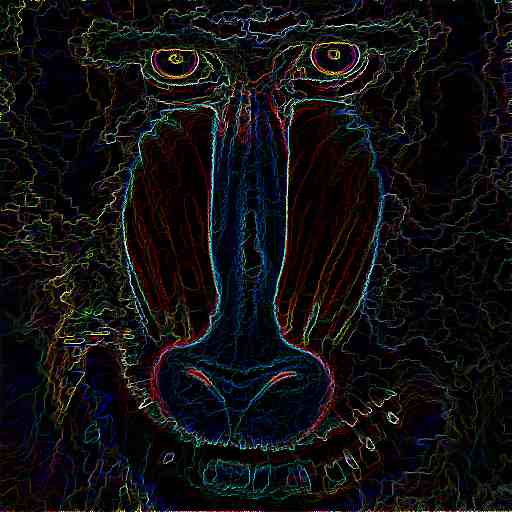}
	\includegraphics[width=3cm, height=3cm]{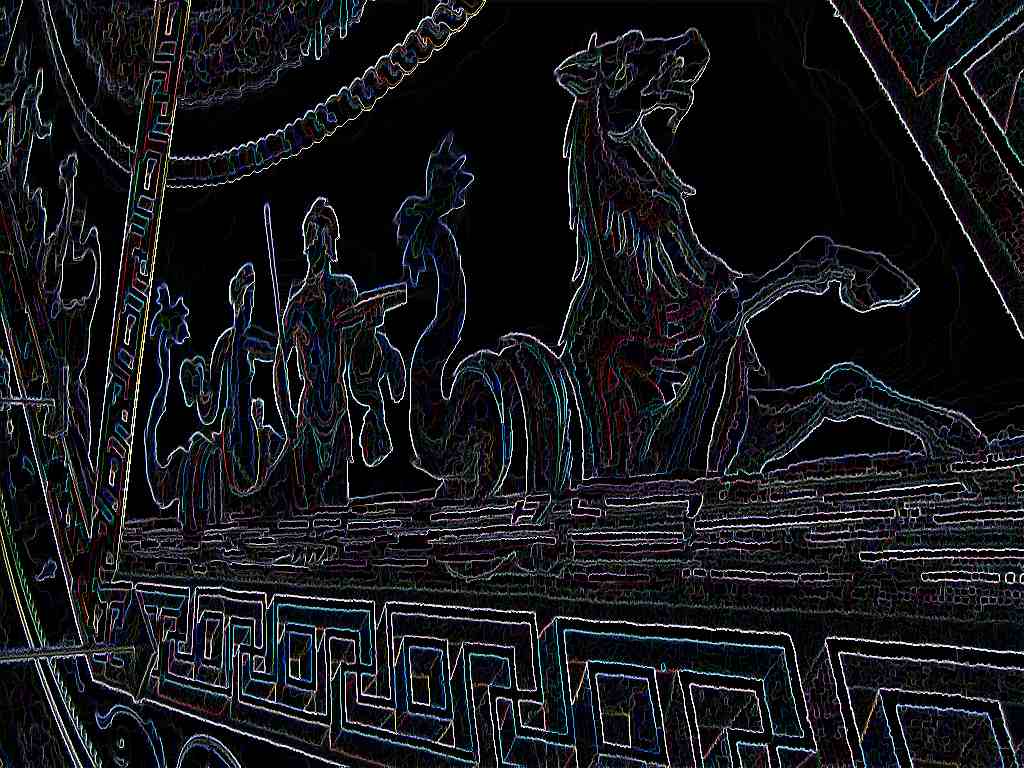}
	\includegraphics[width=3cm, height=3cm]{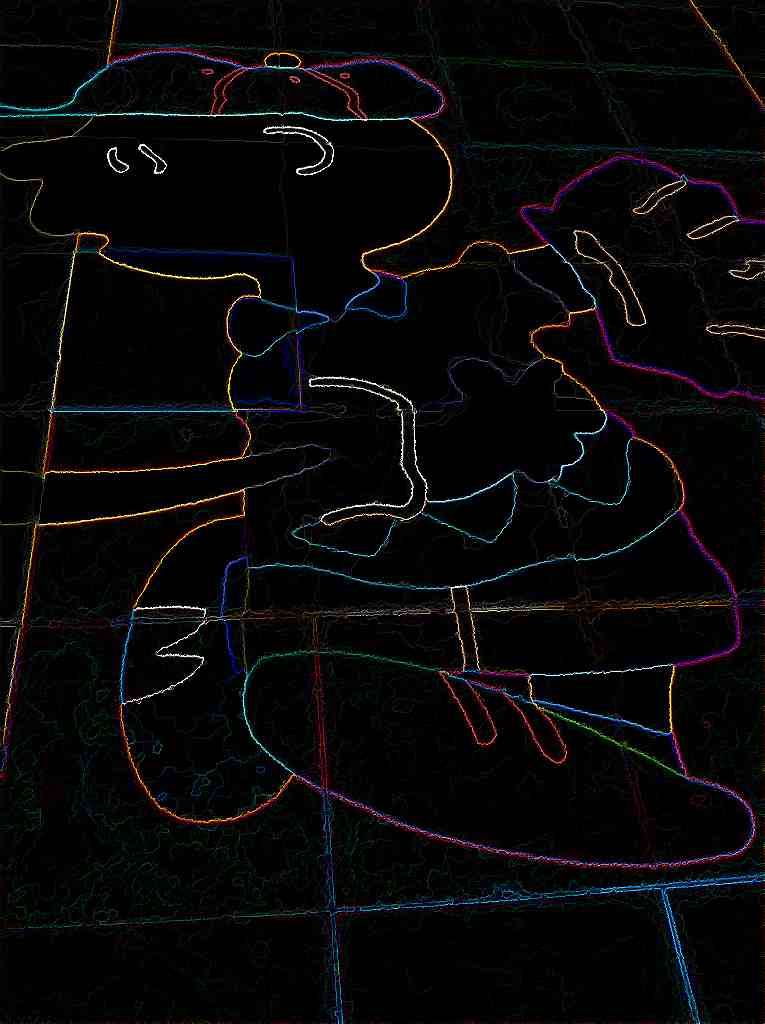}	
	\includegraphics[width=3cm, height=3cm]{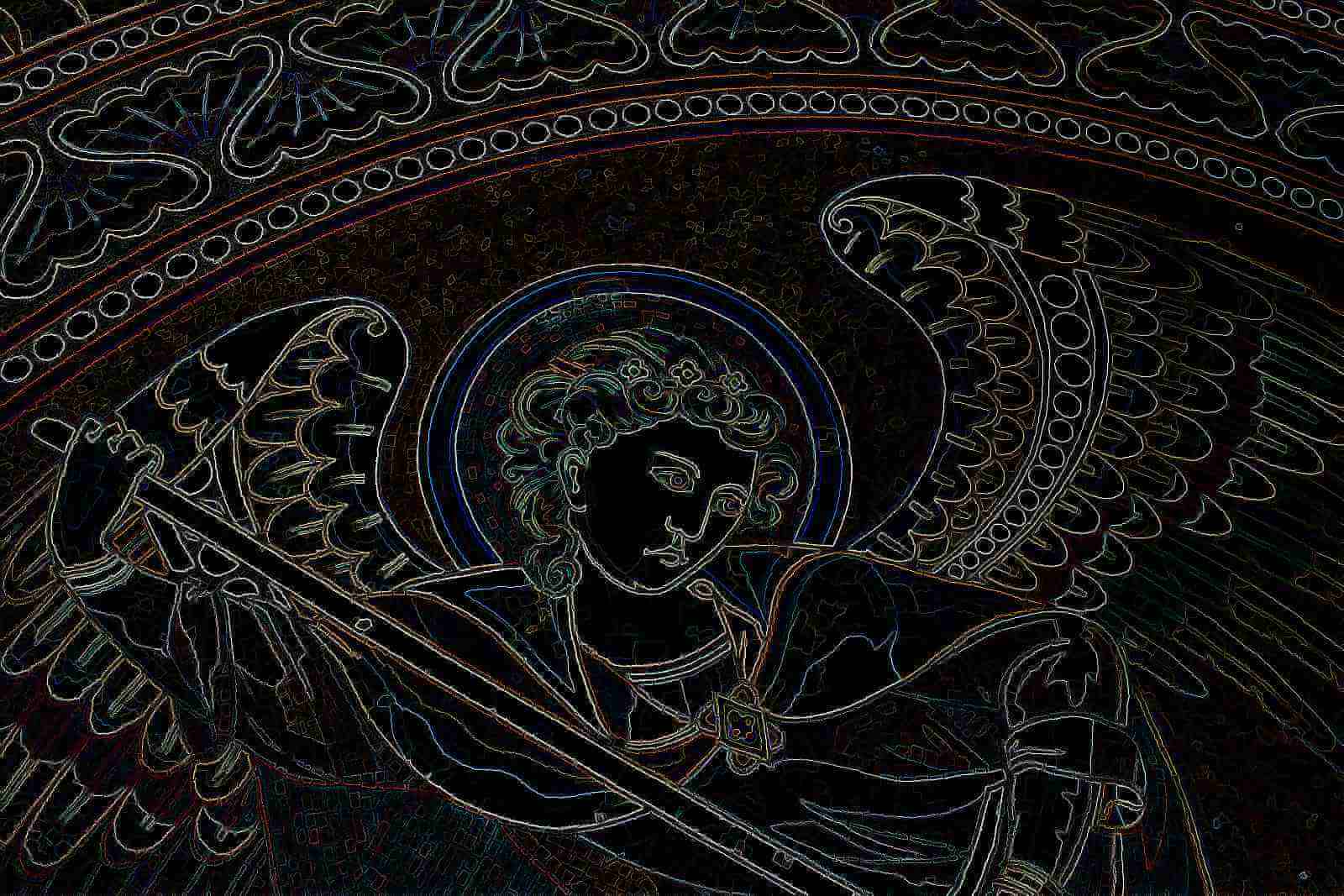}\\		

	\caption{(Color online) Corresponding edge functions $w$. Arrangement is as in Figure~\ref{fig:cartoons}. Best viewed electronically, zoomed in.}\label{fig:gfunctions}
\end{figure*}

We further provide image decomposition for color images~\cite{BC08,DAVese10} by using vectorial TV version in Eqn.~\eqref{E:wTV} following~\cite{BC08}. We consider the following vectorial TV with non-homogenous diffusion equation: 
\begin{eqnarray*}\label{E:wTV_VTV}
\min\limits_{\mathbf{u}=(u_{1},u_{2},u_{3})}\Bigg\{\int_{\Omega}g(w_{i})\sqrt{\sum_{i=1}^{3}|\nabla u_{i}|}\,dx
+ \sum_{i=1}^{3}\int_{\Omega}\mu(x)\,|u_{i}-f_{i}|\,dx\Bigg\},
\end{eqnarray*}
\begin{eqnarray*}\label{E:constraint_VTV}
\frac{\partial w_{i}}{\partial t} = \lambda_{i}(x) div(\nabla w_{i}) + (1-\lambda_{i}(x))(|\nabla u_{i}|-w_{i}),
\end{eqnarray*}
were each scalar function $u_{i}:\Omega\rightarrow \mathbb{R}$, $1\leq i\leq 3$ represent one component of the RGB color system. Note that following our alternating iterative scheme given by equations~\eqref{E:wieqn},~\eqref{E:uisoln} and~\eqref{E:visoln} the new solutions $(u_{i},v_{i},w_{i})$ are given by:
\begin{equation}\label{E:uisoln1}
u_{i}=f_{i}-v_{i} - \theta div\, \mathbf{p}_{i},
\end{equation} 
where $\mathbf{p}_{i}=(p_{i_{1}},p_{i_{2}})$ satisfies $g(w_{i}) \nabla(\theta\, \mathrm{div}\, \mathbf{p}_{i} -(f_{i}- v_{i}))-
|\nabla(\theta \mathrm{div}\, \mathbf{p}_{i} - (f_{i}- v_{i})) |\mathbf{p}_{i}=0$, which is solved using a fixed point method: $\mathbf{p}^{0}_{i} = 0$ and
\begin{equation*}\label{E:psoln1}
\mathbf{p}^{n+1}_{i} = \frac{\mathbf{p}^{n}_{i}+\delta t \nabla(\mathrm{div}(\mathbf{p}^{n}_{i}) -
(f_{i}-v_{i})/\theta) }{1+\displaystyle\frac{\delta t}{g(w_{i})} \sqrt{\sum_{i=1}^{3}|\nabla(\mathrm{div}(\mathbf{p}^{n}_{i}) - (f_{i}-v_{i})/\theta)|^{2}}}.
\end{equation*}

\begin{equation}\label{E:visoln1}
v_{i}=  
\begin{cases} 
f_{i}-u_{i}-\theta\mu(x) & \text{if $f_{i}-u_{i}\geq \theta\mu(x)$},\\
f_{i}-u_{i}+\theta\mu(x) & \text{if $f_{i}-u_{i}\leq -\theta\mu(x)$},\\
0 & \text{if $|f_{i}-u_{i}|\leq \theta\mu(x)$}.
\end{cases}
\end{equation}

\begin{equation}\label{E:wieqn1}
w^{n+1}_{i} = w^{n}_{i} + \frac{\delta t}{(\delta x)^2}(\lambda(x)\tilde \Delta w^{n}_{i} + (1-\lambda(x))\left(\abs{\nabla u_{i}} - w^{n}_{i}\right)).
\end{equation}

\begin{figure*}
\centering
        \subfigure[Input ($f$) and $u$ components using the proposed model]{
        \includegraphics[width=3.75cm]{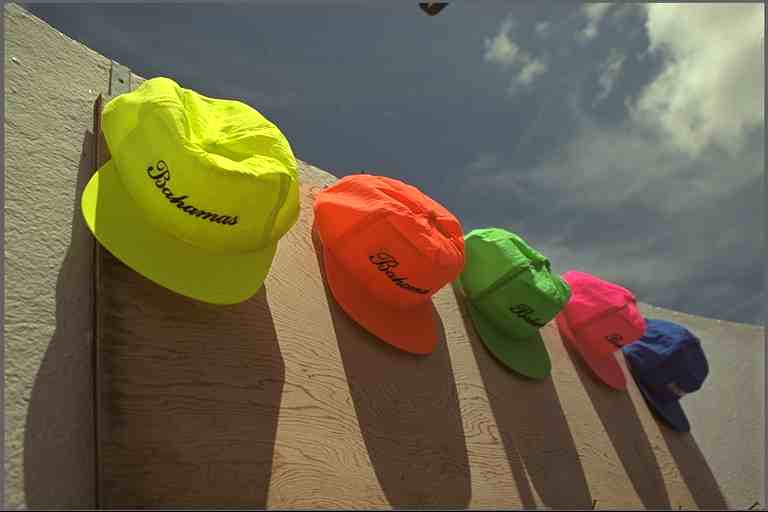}
	\includegraphics[width=3.75cm]{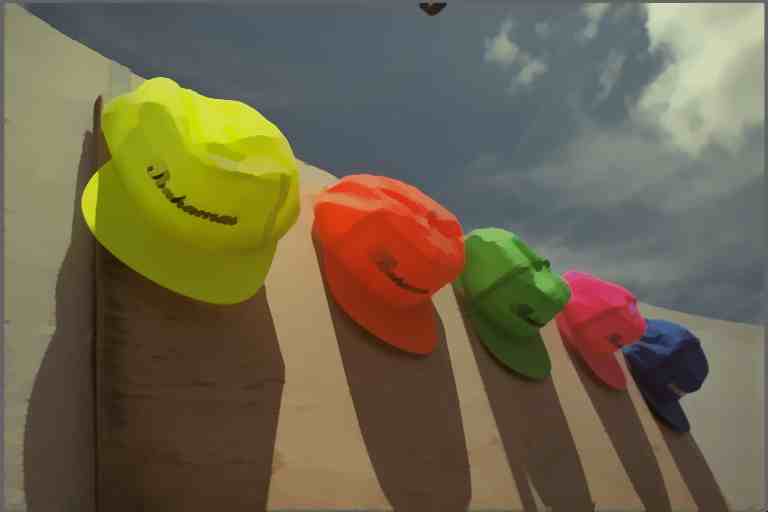}~~
	\includegraphics[width=3.75cm]{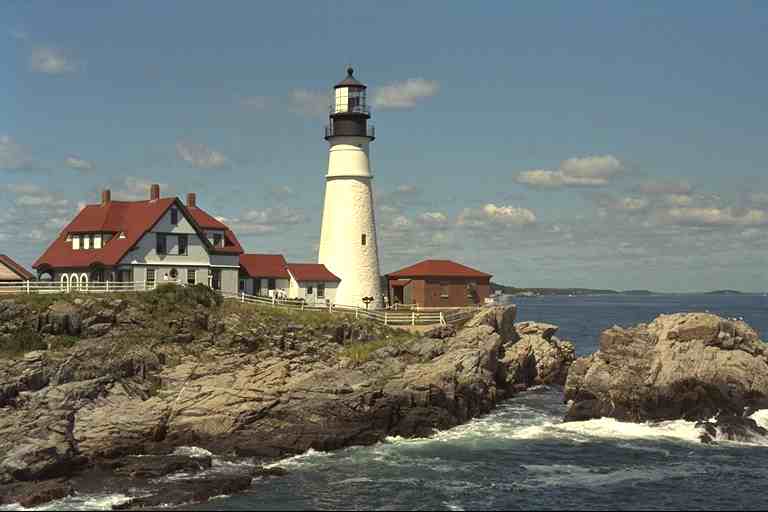}
	\includegraphics[width=3.75cm]{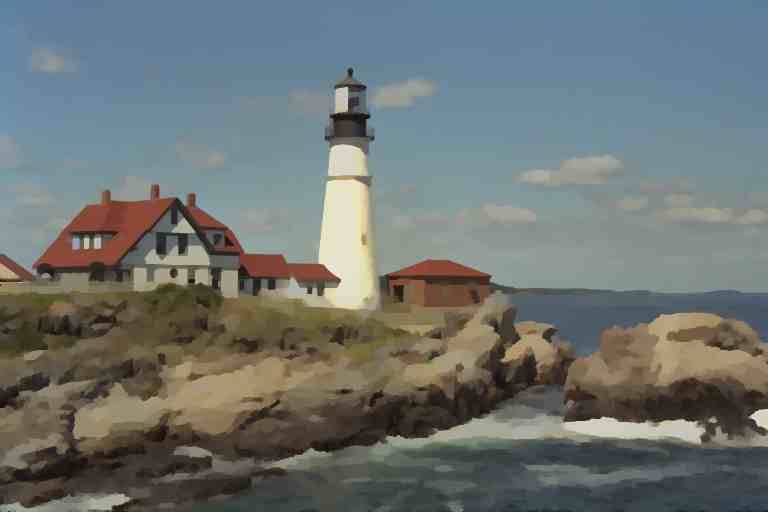}}

	\includegraphics[width=7.75cm]{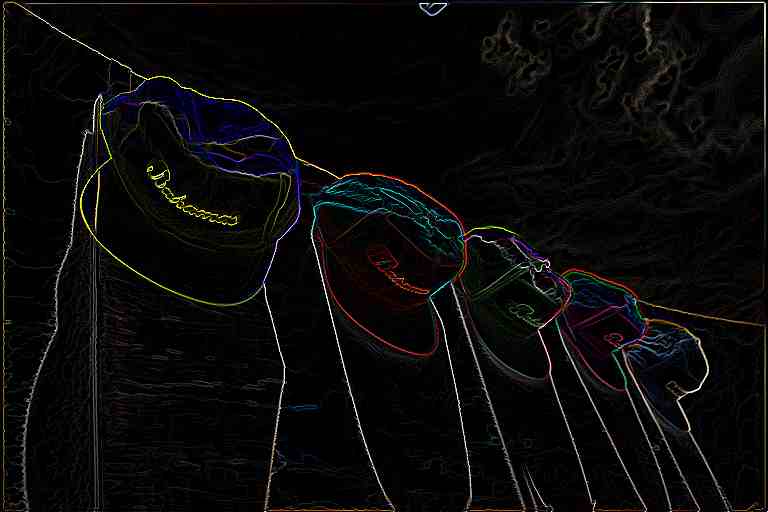}
	\includegraphics[width=7.75cm]{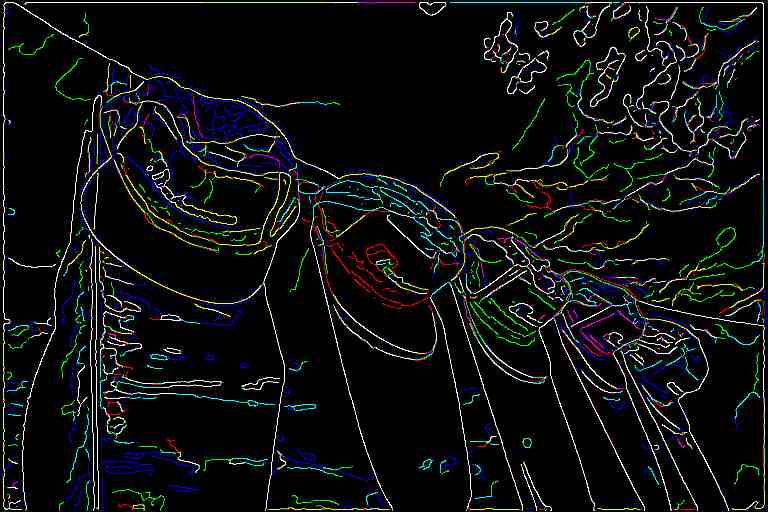}

	\subfigure[(Pseudo) Edges ($w$)]{\includegraphics[width=7.75cm]{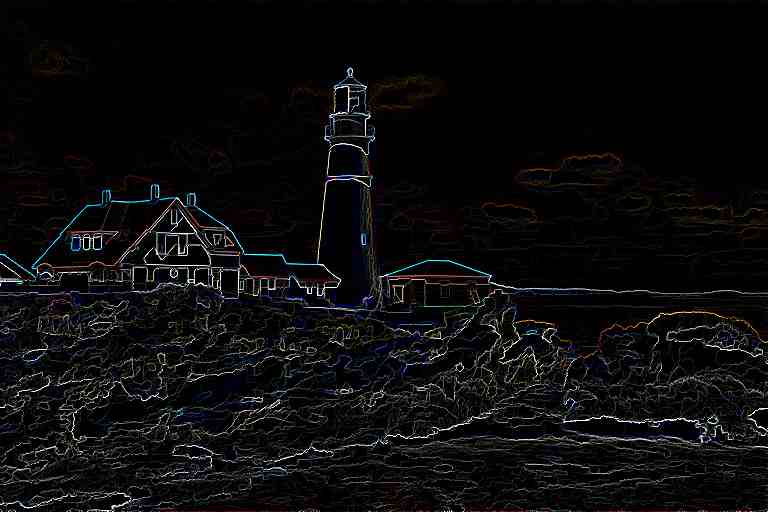}}
	\subfigure[Canny Edges]{
	\includegraphics[width=7.75cm]{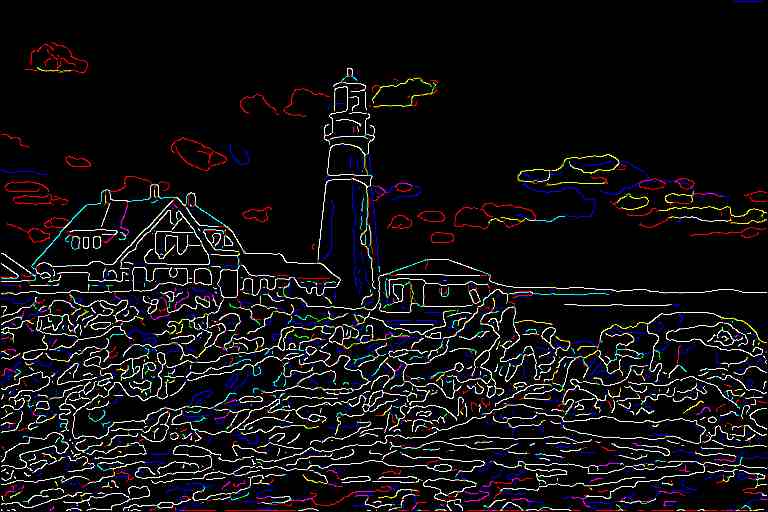}}
	\caption{(Color online) Pseudo edge maps \& Canny Edges. First Row: Input ($f$) and $u$ components using the proposed model with  ($\mu\gets$adaptive, $\lambda\gets$constant) and stopping parameter $\epsilon=10^{-4}$. Second and Third Row: (Pseudo) Edges ($w$) (First Column) and Canny Edges (Second Column).}\label{fig:Edges}
\end{figure*}

Figure~\ref{fig:CTE_Vs_Bresson} shows the cartoon components of our CTE scheme with constant and adaptive $\mu_1$ against the traditional TV based scheme of~\cite{BE07} for two standard RGB test images. As can be seen, our scheme obtains better edge preserving cartoon ($u$) components. The close-up shots indicate that our scheme also avoids blurring of edges, see for example $Barbara$ face. See also Figure~\ref{fig:CTE_Vs_Bresson_Synthetic} where the proposed scheme with adaptive $\mu_1$ provides better shape preservation as stopping parameters increase from $\epsilon=10^{-4}$ to $\epsilon=10^{-7}$. Next, Figures~\ref{fig:cartoons}-\ref{fig:gfunctions} shows decomposition for a variety of RGB images for two different iteration values in our proposed CTE scheme with constant $\mu$ against adaptive $\mu_1$ based results. As can be seen in Figure~\ref{fig:cartoons}, increasing the number of iterations removes more texture details and provides piecewise constant (smoother) cartoon images. Our adaptive $\mu_1$ based scheme results (last row) on the other hand keep most of the salient edges. This can be seen further in Figure~\ref{fig:textures} (last row) where the adaptive $\mu_1$ based scheme captures only small scale oscillations corresponding to texture components whereas the constant $\mu$ based results remove strong edges as well. Figure~\ref{fig:gfunctions} show the corresponding edge functions, and it can be seen that adaptive scheme has more information than in constant parameter case. Thus, we conclude that, using adaptive $\mu$ provides an image adaptive way of obtaining edge preserving cartoon components without sacrificing overall decomposition properties. 

Figure~\ref{fig:Edges} shows the cartoon component of two different color images as well as  the edge functions given by the proposed model, and Canny edge maps~\cite{CA86} computed in all the three channels\footnote{Computed using the MATLAB's default command {\tt edge(u,`canny')}.}. As observed the function $w$ captures edges based on the cartoon subregions {\it i.e.} large intensity gradients while the Canny edge maps captures large and small intensity gradients. 

\subsubsection{Comparison with other decomposition methods}\label{sssec:others}

\begin{sidewaystable}
\centering
\begin{tabular}{lllllllllll}
\hline
 Image/MSSIM & OSV & VC & AGO & BP & CP & DHC & BCM	& MCTE & ACTE & CCTE\\
\hline
\multirow{2}{*}{Girl1$\dagger$/0.2081}       & 27.5748                &  28.1667                  & 28.0247
                                 & 27.1771
                                 & 27.9446                & 28.3862                   &  28.4113                  
                                 & {\bf 28.8536}   & 28.3557                   & \underline{28.8266}\\
                                 
                                 &0.7664 	&0.7669		&0.7705   		&0.7531		&0.7604 	&0.7732	&0.7293	& {\bf 0.7971}	&0.7893		&\underline{0.7939}\\
                                 \hline
                                 
\multirow{2}{*}{Couple1$\dagger$/0.2029}  & 27.1529                & 28.3845                & 27.7122   
                                 & 27.4384 
                                 &  27.8213                  & 28.1447                  & 28.5528                 
   & \underline{28.5297}                   & 28.0973                 & {\bf 29.4769}\\
   
   &0.6918		& 0.7080 		& 0.7022		& 0.7063		& 0.7042		&0.7283		& 0.7045 		&\underline{0.7467}		&0.7439		& {\bf 0.7548}\\
			   \hline
   
\multirow{2}{*}{Girl2$\dagger$/0.1482}          & 29.6239                   & 30.6253                 & 29.6591
                                  & 30.6260 
                                  &  29.7908                  & 32.0758                 & 30.9228                 
                                  & {\bf 32.2619}    & 30.7951                 & \underline{32.1304}\\
                                  
                                  &0.8730		&0.8039	&0.8829  		&0.8779		&0.8751		&\underline{0.8950}		&0.7991		&{\bf 0.8964}		&0.8933		&0.8821\\
                                  \hline
                                  
\multirow{2}{*}{Girl3$\dagger$/0.2261}           & 29.4733                   & 29.3198                & 28.8185  
                                   & 28.2566
                                   & 28.7244                  & 30.6547                & 29.3895                
                 & {\bf 31.3260}          & 29.1061         & \underline{31.1257}\\
                 
                 &0.8470	&0.8029	&0.8335	&0.8092	&0.8145 	&0.8014	&0.7831	&\underline{0.8469} 		&{\bf 0.8595}	&	0.8379\\
                 \hline
                 
 \multirow{2}{*}{House1$\dagger$/0.2297}     & 27.1664                  & 26.5337                 & 27.7950 
                                   & 27.4868
                                   & 28.0119                   & 29.2988                 & \underline{29.4386}                  
                                   & 29.0679          & 28.4171                & {\bf 29.5162}\\
                                   
                                   &0.7858 	&0.7626	&0.7831 	&0.7950	&0.7795	&0.7993	&0.7603	&{\bf 0.8138} 	&\underline{0.8090} 	&0.8064\\
                                   \hline
                                                                      
\multirow{2}{*}{Tree$\dagger$/0.3959}            & 23.7801                  &  24.4972                & 24.0548  
                                   & 24.2811
                                   & 24.3028                  & 26.3365                  & 26.0382
                & \underline{26.9758}           & 24.3702                 & {\bf 27.0157}\\
                
                &0.7259	&0.7549	&0.7231	&0.7487	&0.7314	&0.7654	&\underline{0.7684}                  &{\bf 0.7702}	&0.7489 	&0.7666\\
                \hline
                
\multirow{2}{*}{Jelly1$\dagger$/0.1700}          & 29.7350                 & 29.8272                  & 29.4804  
                                   & 30.0963
                                   & 29.2544                  & 30.6860                  & 30.2663                  
                 & {\bf 31.4578}         & \underline{31.2529}            & 30.9478\\
                 
                 &0.9181	&0.8297	&0.9351	&0.9153	&0.9178	&0.9360	&0.8407 	&\underline{0.9369} 		&{\bf 0.9376}	&0.9145\\
                 \hline
                 
\multirow{2}{*}{Jelly2$\dagger$/0.2294}          & 28.8000                 & 28.9457                    & 27.5507    
                                    & 28.1674
                                    & 27.9968                 & 29.1260                   & 29.2134                    
                                    & {\bf 30.0485}  & 29.5462                    & \underline{29.9156}\\
                                    
                                    &0.9098	&0.8385	&0.9029	&0.8923	&0.9005	&0.8227 	&0.8443	&{\bf 0.9202}	&\underline{0.9182}	&0.9020\\
                                    \hline
                                    
\multirow{2}{*}{Splash$\ddagger$/0.4151}       & 30.6488               & 28.9810                     & 29.9935    
                                    & 31.4922 
                                    & 31.2080                & \underline{33.1049}                     & 31.0800                     
                 & {\bf 33.8236}        & 31.3856                     & 31.2644\\
                 
                 &0.8825 	&0.8312	&0.8891	&0.8818	&0.8740	&0.8669	&0.8732	&\underline{0.8975} 		&0.8935	&{\bf 0.9020}\\
                 \hline
                 
\multirow{2}{*}{Tiffany$\ddagger$/0.4482}         & 29.2555              & 28.0636                     & 28.6905 
                                     & 30.2667
                                     & 29.7630               & 31.2730                     & 29.6920                    
                              & {\bf 32.0779}       & 30.7028                        & \underline{31.3727}\\
                              
                              &0.8617	&0.8254	&0.8534	&0.8425	&0.8491	&0.8567	&0.8628 	&{\bf 0.8708}	&0.8612	&\underline{0.8685}\\
                              \hline
\multirow{2}{*}{Mandril$\ddagger$/0.7097}        & 20.4765           & 21.6728                    & 20.7883 
                                     & 21.9056
                                     & 20.9060              & 23.7888      & {\bf 23.9897}                     
           & \underline{23.7899}              & 20.8689                      & 22.3085\\
           
           &0.65255	&0.7651 	&0.6443	&0.6962 	&0.6530 	&0.8002	&{\bf 0.8247} 	&\underline{0.8117} 		&0.6520	&0.7713\\
           \hline

\hline
\end{tabular}

\caption{PSNR  ({d}B) and MSSIM comparison of various decomposition schemes with noise level $\sigma=30$ for standard test images from the USC-SIPI database  with size $\dagger = 256\times 256$ (Noisy $\text{PSNR} = 22.11$) and $\ddagger = 512\times 512$ (Noisy $\text{PSNR} = 22.09$). Each row indicates PSNR/MSSIM values for different test images. The diffusion with multiscale ($\mu\gets$multiscale, $\lambda\gets$constant), adaptive choice ($\mu\gets$adaptive, $\lambda\gets$constant) and constant CTE scheme ($\mu\gets$constant, $\lambda\gets$constant) are given as MCTE, ACTE and CCTE (last three columns) respectively (stopping parameter $\epsilon=10^{-4}$). Compared methods are OSV~\cite{OSVese03}, VC~\cite{VChan04}, AGO~\cite{AujolGilboa06}, BP~\cite{BP10}, CP~\cite{ChambollePock11}, DHC~\cite{DongHinterRincon11}, BCM~\cite{BuadesColl06}.
Best results are indicated in boldface and the second best is underlined.}\label{t:USC_SIPI}.
\end{sidewaystable}

\begin{figure*}
\centering
\subfigure[\scriptsize{$Barbara$ image}]{
\includegraphics[width=3.cm,height=3.cm]{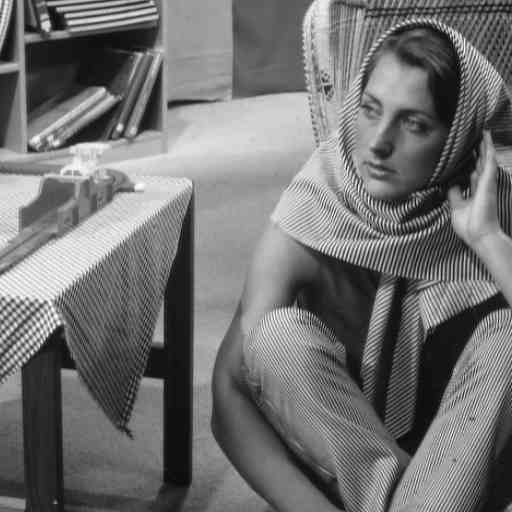}}
\subfigure[\scriptsize{Canny edges}]{
\includegraphics[width=3.cm,height=3.cm]{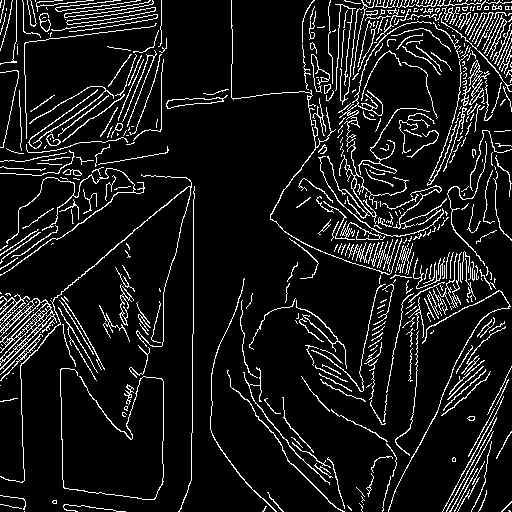}}
\subfigure[\scriptsize{Noisy $Barbara$}]{
\includegraphics[width=3.cm,height=3.cm]{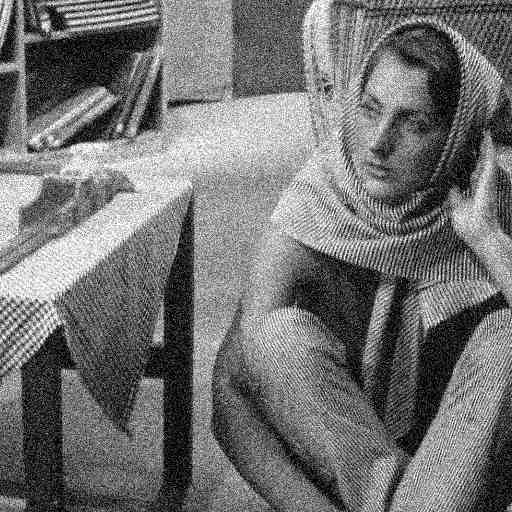}}
\subfigure[\scriptsize{Canny edges}]{
\includegraphics[width=3.cm,height=3.cm]{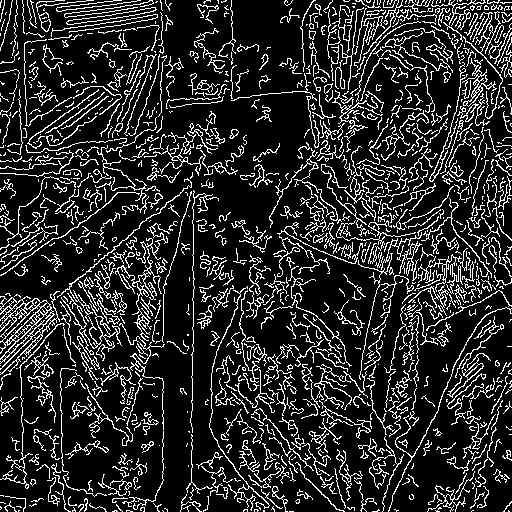}}
\subfigure[\scriptsize{Input noise}]{
\includegraphics[width=3.cm,height=3.cm]{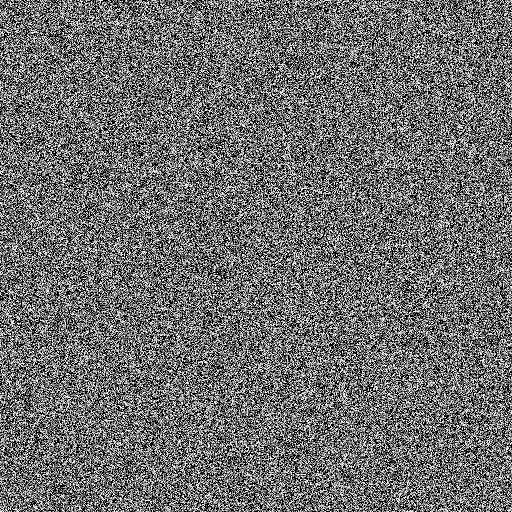}}

\[\begin{array}{ccc}
\includegraphics[width=2.5cm,height=2.75cm]{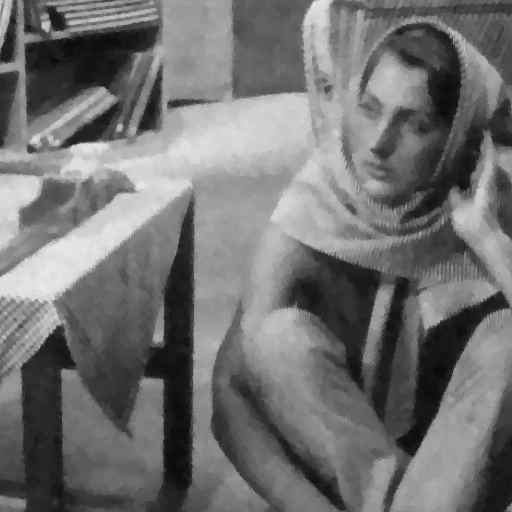}
\includegraphics[width=2.5cm,height=2.75cm]{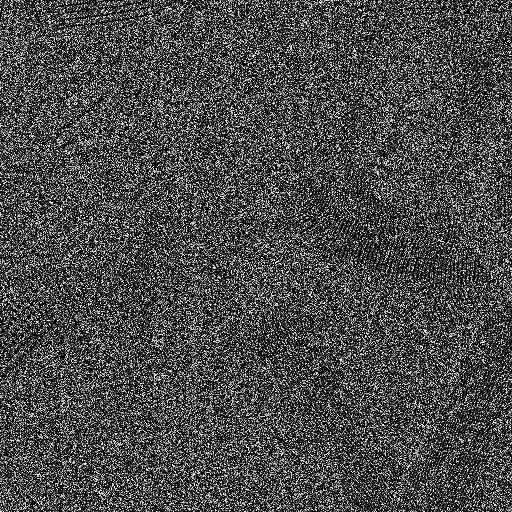}
&
\includegraphics[width=2.5cm,height=2.75cm]{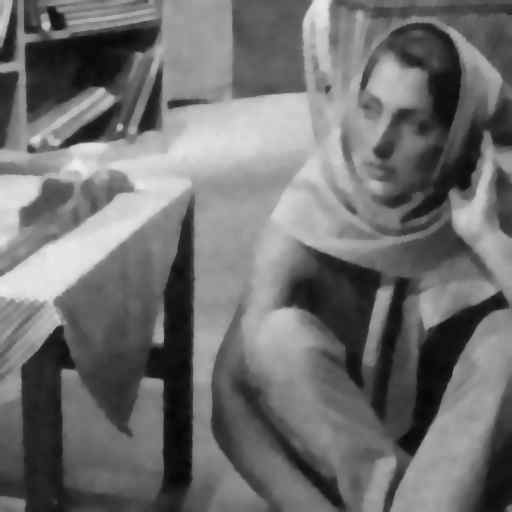}
\includegraphics[width=2.5cm,height=2.75cm]{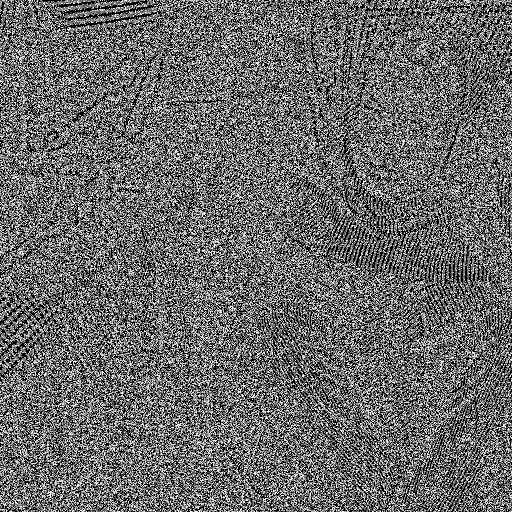}
&
\includegraphics[width=2.5cm,height=2.75cm]{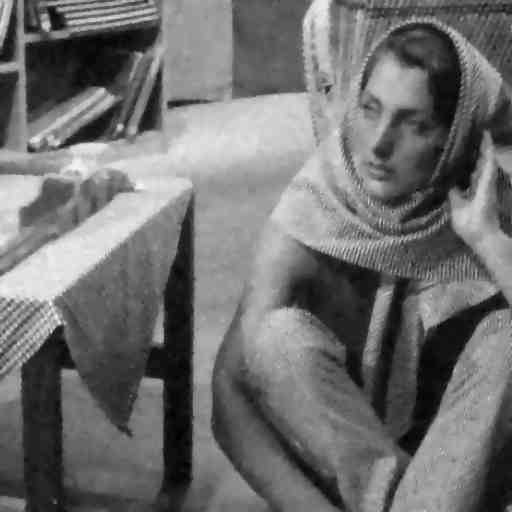}
\includegraphics[width=2.5cm,height=2.75cm]{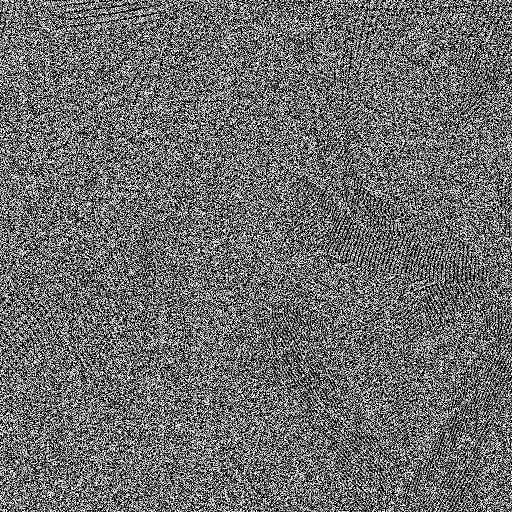}
\\
\mbox{\scriptsize{(b) $u$ \& $f-u$ for VC~\cite{VChan04}}}
&
\mbox{\scriptsize{(c) $u$ \& $f-u$ for OSV~\cite{OSVese03}}}
&
\mbox{\scriptsize{(d) $u$ \& $f-u$ for AGO~\cite{AujolGilboa06}}}
\end{array}\]

\includegraphics[width=3.cm,height=3.cm]{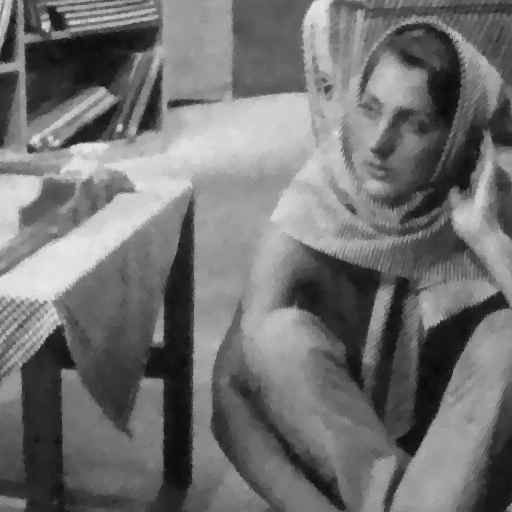}
\includegraphics[width=3.cm,height=3.cm]{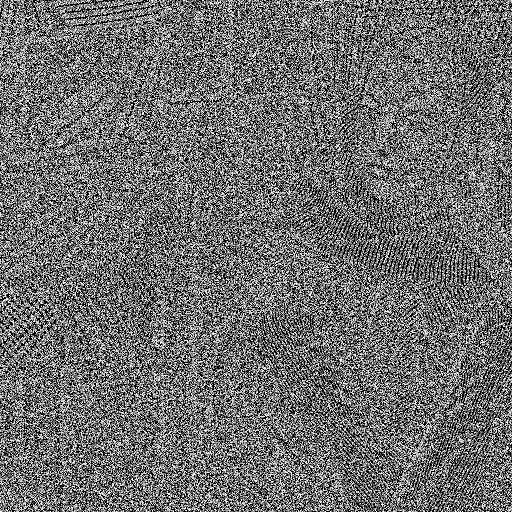}
\includegraphics[width=3.cm,height=3.cm]{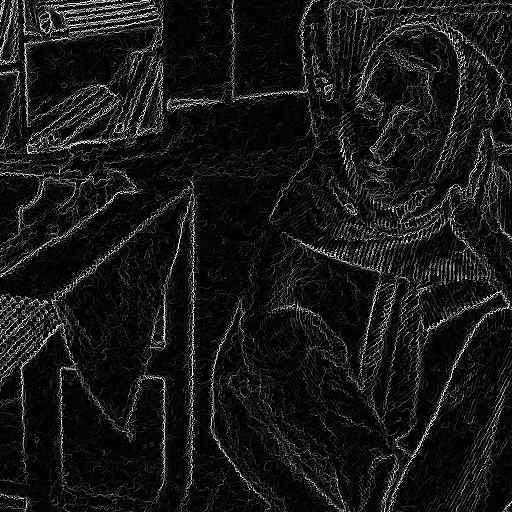}
\includegraphics[width=3.cm,height=3.cm]{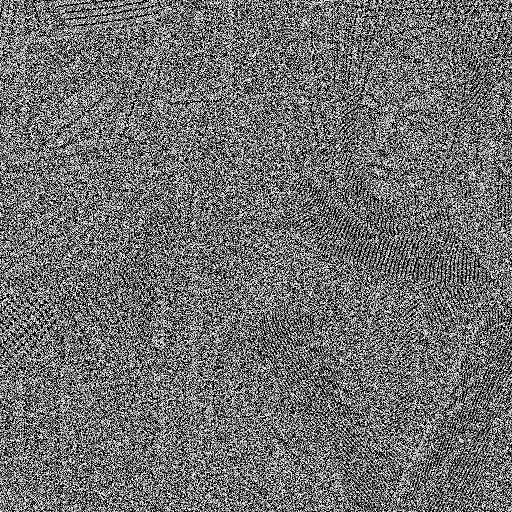}
\includegraphics[width=3.cm,height=3.cm]{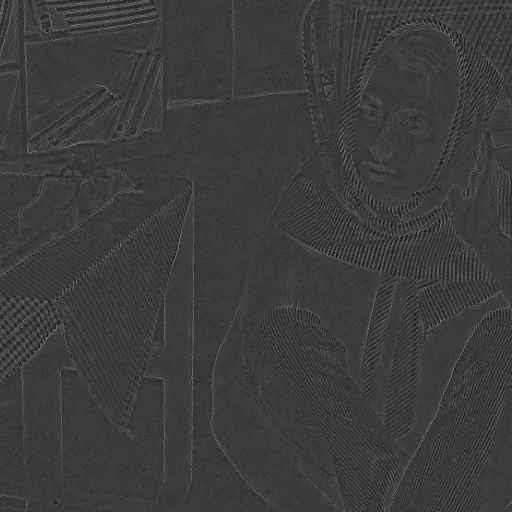}
\\
\includegraphics[width=3.cm,height=3.cm]{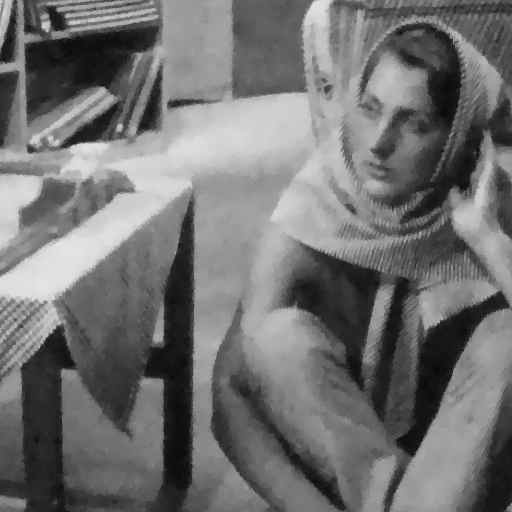}
\includegraphics[width=3.cm,height=3.cm]{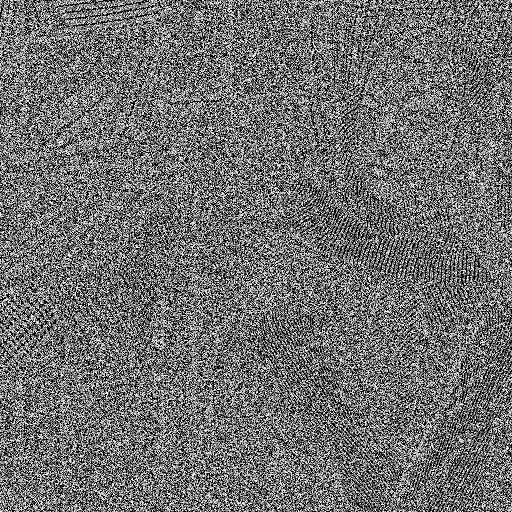}
\includegraphics[width=3.cm,height=3.cm]{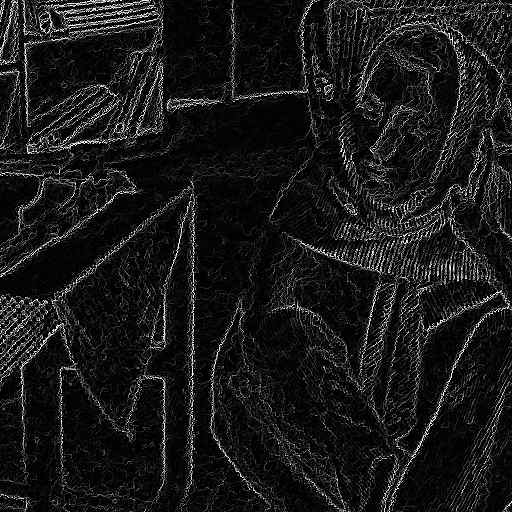}
\includegraphics[width=3.cm,height=3.cm]{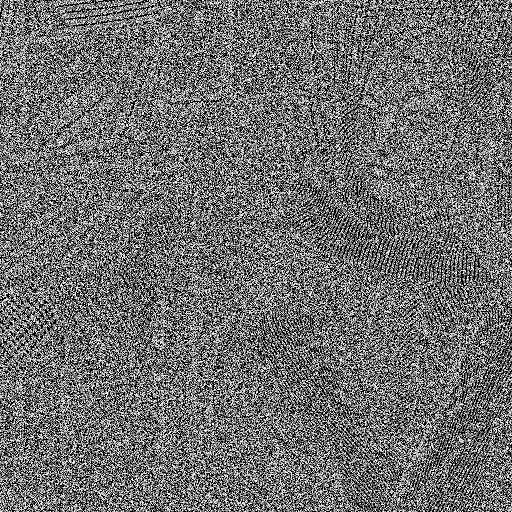}
\includegraphics[width=3.cm,height=3.cm]{barbara_CTEDeno30_lambdaconstant_diff_4}
\\
\setcounter{subfigure}{4}
\subfigure[\scriptsize{$u,v,w$, $f-u$, and $f-(u+v+w)$ components for the proposed scheme.}]{\includegraphics[width=3.cm,height=3.cm]{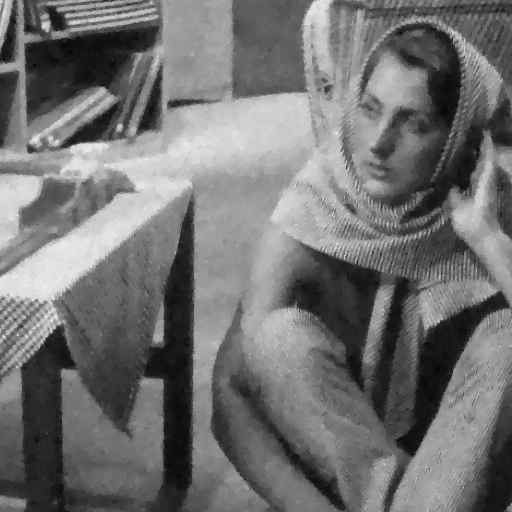}
\includegraphics[width=3.cm,height=3.cm]{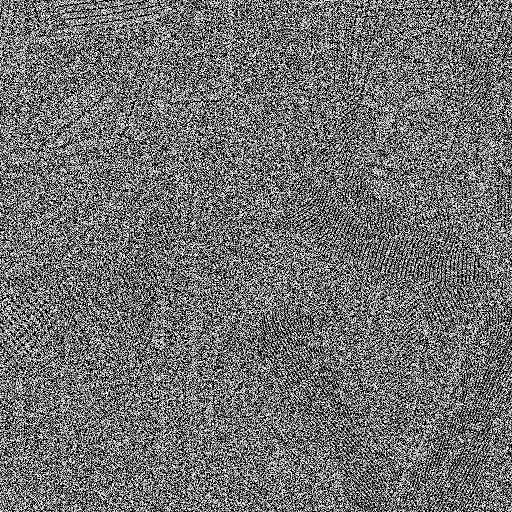}
\includegraphics[width=3.cm,height=3.cm]{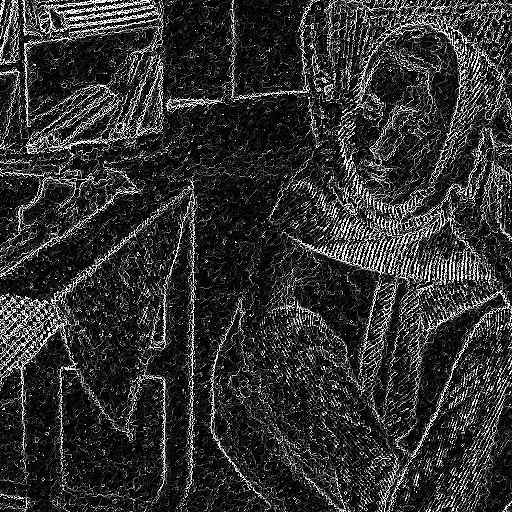}
\includegraphics[width=3.cm,height=3.cm]{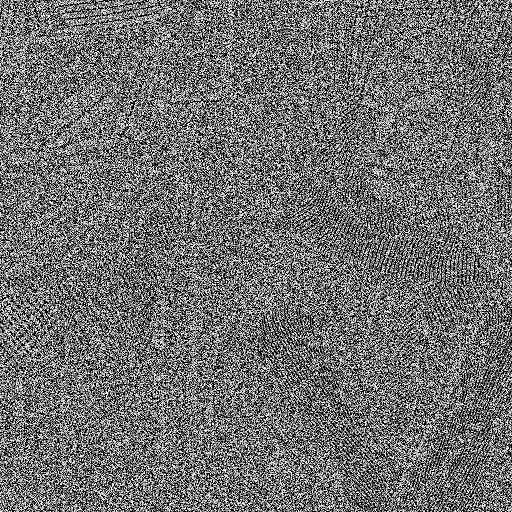}
\includegraphics[width=3.cm,height=3.cm]{barbara_CTEDeno30_lambdaconstant_diff_4}
}
\caption{Decomposition of noise inputs.
	(a) $House$, $F-16$ and $Barbara$ images corrupted by Gaussian noise images (standard deviation $\sigma=30$).
	(b)-(d) $u$ \& $f-u$ components provided by the VC~\cite{VChan04}, OSV~\cite{OSVese03}
and AGO~\cite{AujolGilboa06} schemes, respectively.
	(e) $u,v,w,f-u$ and $f-(u+v+w)$ components using the proposed CTE scheme with (CCTE - $\mu\gets$constant, $\lambda\gets$constant) for the Third row, with (ACTE - $\mu\gets$adaptive, $\lambda\gets$constant) for the Fourth row and (MCTE - $\mu\gets$multiscale, $\lambda\gets$) for the Fifth row. The corresponding PSNR  ({d}B) and MSSIM measures for each model are displayed in Table~\ref{t:USC_SIPI}.}\label{fig:USC_SIPI}. 
\end{figure*}
Next we use the grayscale digitize images from USC-SIPI Image Database to compare the performance of the CTE model together with the adaptive and multi-scale case with the following schemes: TV-$H^{-1}$ decomposition model (OSV, using finite differences)~\cite{OSVese03}, TV-$G_{p}$ decomposition model (VC, finite differences channel-wise)~\cite{VChan04}, TV-$Garbor$ decomposition model (AGO, Chambolle's projection algorithm )~\cite{AujolGilboa06}, $BV^{2}-L^{2}$ model (BP, which uses an algorithm close to the one set by dual minimization of~\cite{Ch04} and ~\cite{VChan04})~\cite{BP10}, $TV-L^{1}$ model (CP, a first-order primal-dual algorithm)~\cite{ChambollePock11}, $L^{\tau}$-$TV$ scheme for color images (DHC, using a unified Moreau-Yosida based primal-dual approach)~\cite{DongHinterRincon11}, and the NL- Means algorithm (BCM, a improvement of the classical and state-of-the-art methods in image denosing)~\cite{BuadesColl06} which computes the local average color of the most resembling pixels using non-local means. 

Numerical results are given in Table~\ref{t:USC_SIPI} for sixteen standard natural grayscale images of size $256\times 256$ and $512\times 512$. As we can see the proposed scheme based on multi-scale lambda ($\mu\gets$multiscale, $\lambda\gets$constant) and the adaptive choice ($\mu\gets$multiscale, $\lambda\gets$constant) performs well for a variety of images. Even with the adaptive choice (ACTE, $\mu\gets$adaptive, $\lambda\gets$constant) outperforms previous variational models bases $H^{-1}$, $G_{p}$ and $Garbor$ norms and in some cases the $L^{\tau}$-$TV$ and the NL-Means schemes. For textured images such as $Mandrill$ and $Barbara$ we perform all variational models except the NL-means algorithm. We remark that the proposed CTE model does not aim to give state-of-the-art results for image denoising, and instead concentrates on demonstrating how our decomposition model can be harnessed for noise removal and edge detention. Figure~\ref{fig:USC_SIPI} shows decomposition for $Barbara$ into different components for some of the schemes in Table~\ref{t:USC_SIPI}. We note that for other schemes we are able to discompose the image in $u$ (smooth) and $f-u$ (random noise). In our case, we are able to obtain edge variable $w$ component. We also notice that for our model the random noise component $f-u$ is directly given by the $v$ component.


\subsection{Image denoising results}\label{ssec:denois}

\subsubsection{Error metrics}\label{sssec:metrics}
Currently there are no quantitative ways to evaluate different decomposition algorithms. In particular which smooth, texture and edge separation model are better is an open research question in image quality assessment. The proposed decomposition provides piecewise cartoon component which is obtained using a weighted TV regularization in an edge preserving way, see Figure~\ref{fig:cartoons} (last row). Hence, as a byproduct we obtain image denoising, with $u$ the denoised image and $v+w$ the `noise' part.  To compare the schemes quantitatively for the purpose of denoising, we utilize two commonly used error metrics in the image denoising literature, one is the classical peak signal to noise ratio (PSNR)~\cite{AK06}, and the other one is the mean structural similarity measure (MSSIM)~\cite{WB04}.
\begin{enumerate}
\item PSNR is given in decibels ($dB$). A difference of $0.5\,dB$ can be identified visually. Higher PSNR value indicates optimum denoising capability. 
\begin{eqnarray}\label{E:psnr}
\text{PSNR}(u) := 20*\log10{\left(\frac{u_{max}}{\sqrt{MSE}}\right)}\,dB,
\end{eqnarray} 
where $\text{MSE} = (mn)^{-1} \sum\sum (u-u_0)$, $m\times n$ denotes the image size, $u_{max}$ denotes the maximum value, for example in $8$-bit images $u_{max} = 255$.

\item MSSIM index is in the range $[0,1]$. The MSSIM value near one implies the optimal denoising capability of the scheme~\cite{WB04} and is mean value of
the SSIM metric. The SSIM is calculated between two windows $\omega_1$ and $\omega_2$ of common size $N\times N$, 
\begin{eqnarray*}\label{E:ssim}
\text{SSIM}(\omega_1, \omega_2) = \frac{(2\mu_{\omega_1}\mu_{\omega_2}+c_1)(2\sigma_{\omega_1\omega_2}+c_2)} {(\mu_{\omega_1}^2+\mu_{\omega_2}^2+c_1)(\sigma_{\omega_1}^2+\sigma_{\omega_2}^2+c_2)},
\end{eqnarray*}
where $\mu_{\omega_i}$ the average of $\omega_i$, $\sigma^2_{\omega_i}$ the variance of $\omega_i$, $\sigma_{\omega_1\omega_2}$ the covariance, $c_1,c_2$ stabilization parameters, see~\cite{WB04} for more details.
\end{enumerate}
\begin{remark} 
Note that the SSIM is a better error metric than PSNR as it provides a quantitative way of measuring the structural similarity of denoised image against the original noise-free image. To adapt PSNR and SSIM error metrics  to color images one can convert the color image into gray\footnote{For example, using MATLAB's command \texttt{rgb2gray}.} and then compute PSNR and MSSIM error metrics for the converted gray-scale image. In this paper, in order to compare with the scheme on~\cite{PVorotnikov12},  we compute the PSNR and MSSIM on each channel and use the average as a final value.
\end{remark} 

\subsubsection{Comparison with the previous model}\label{sssec:previous}

Table~\ref{table:psnr} compares proposed scheme with that of~\cite{PVorotnikov12} using the both PSNR (dB) and MSSIM error metrics average for Berkeley segmentation dataset (BSDS) images. We implemented both the schemes on the full Berkeley segmentation dataset of  $500$ noisy images for two different noise levels and obtained similar improvements. Figures~\ref{fig:CTE_Vs_Surya_Berkeley_1}-\ref{fig:CTE_Vs_Surya_Berkeley_2} show some example images corresponding to Table~\ref{table:psnr}. As can be seen from the zoomed in versions, the proposed CTE scheme provides cleaner cartoon components (denoised images, see Figures~\ref{fig:CTE_Vs_Surya_Berkeley_1}(b-e)-\ref{fig:CTE_Vs_Surya_Berkeley_2}(b,c)) in contrast to original coupled PDE model~\cite{PVorotnikov12} which either excessively blurs out details (Figure~\ref{fig:CTE_Vs_Surya_Berkeley_1}(f,g)) or keeps noisy regions (Figure~\ref{fig:CTE_Vs_Surya_Berkeley_2}(d)) in final results.
\begin{figure*}
\centering
\subfigure[ \scriptsize{Input noise ($\sigma=20$), $(\mbox{PSNR},\mbox{MSSIM})=(24.1759,0.5130)$}]{\includegraphics[width=4.2cm,height=3.5cm]{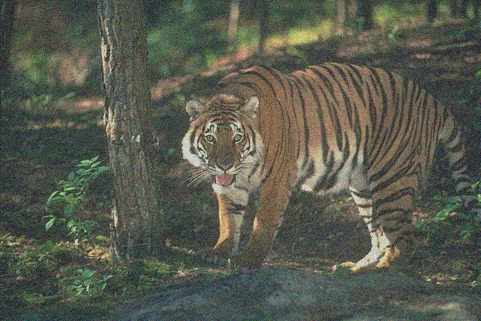}\includegraphics[width=3.5cm,height=3.5cm]{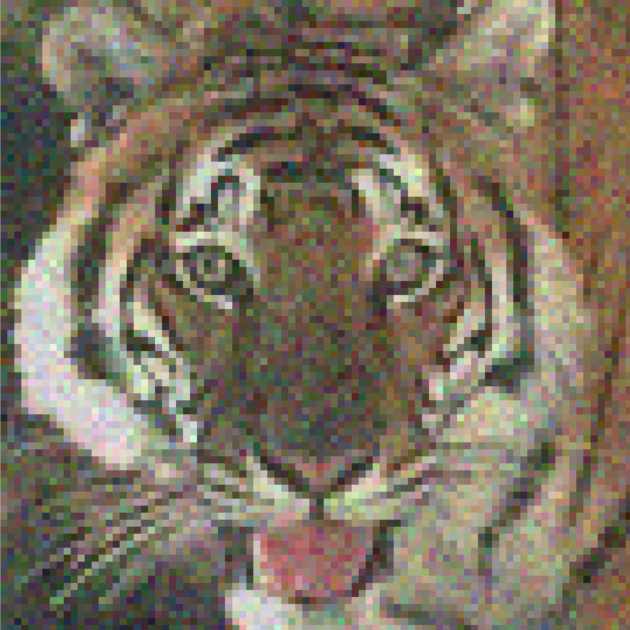}}\\
	\[\begin{array}{cc}
\includegraphics[width=4.2cm,height=3.5cm]{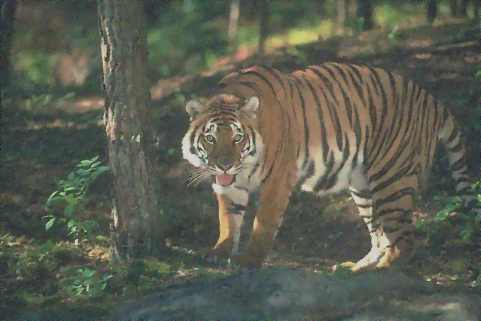}
\includegraphics[width=3.5cm,height=3.5cm]{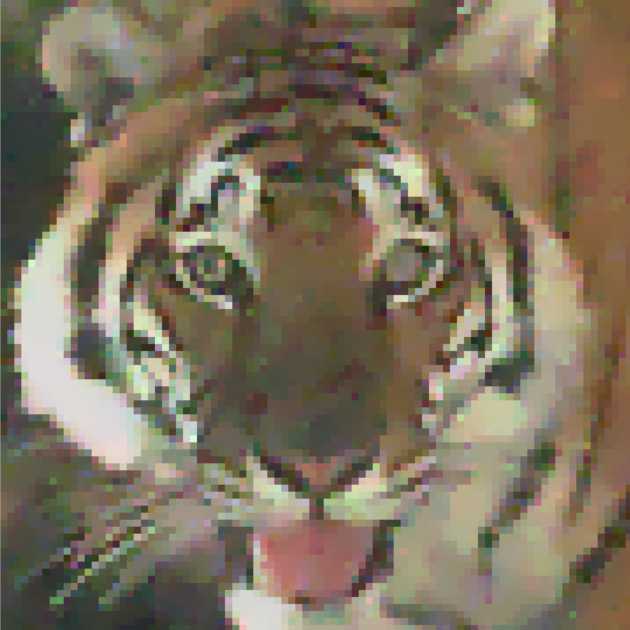}&
\includegraphics[width=4.2cm,height=3.5cm]{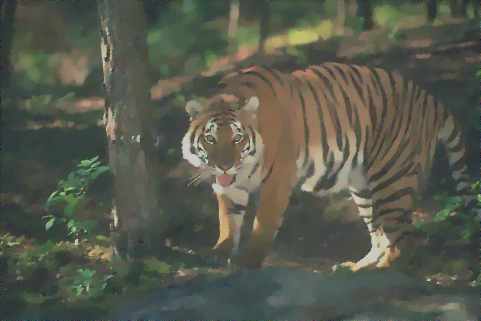}
\includegraphics[width=3.5cm,height=3.5cm]{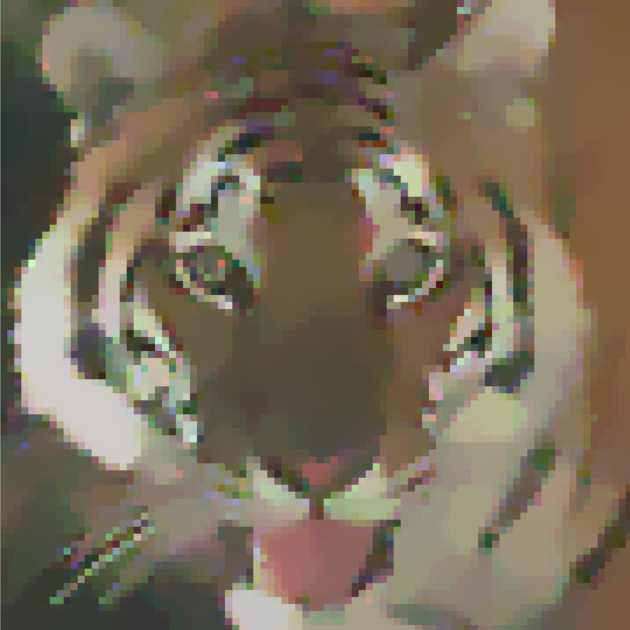}\\
\mbox{\scriptsize{(b) Proposed ($\mu\gets$adaptive, $\lambda\gets$constant ) }} &
\mbox{\scriptsize{(c) Proposed ($\mu\gets$adaptive, $\lambda\gets$constant) }} \\
\mbox{\scriptsize{$(\mbox{PSNR},\mbox{MSSIM})=(28.6484,0.7675)$}} &
\mbox{\scriptsize{$(\mbox{PSNR},\mbox{MSSIM})=(27.5146,0.6993)$}}\\
 
\includegraphics[width=4.2cm,height=3.5cm]{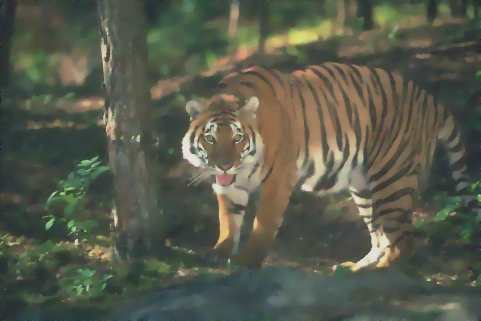}
\includegraphics[width=3.5cm,height=3.5cm]{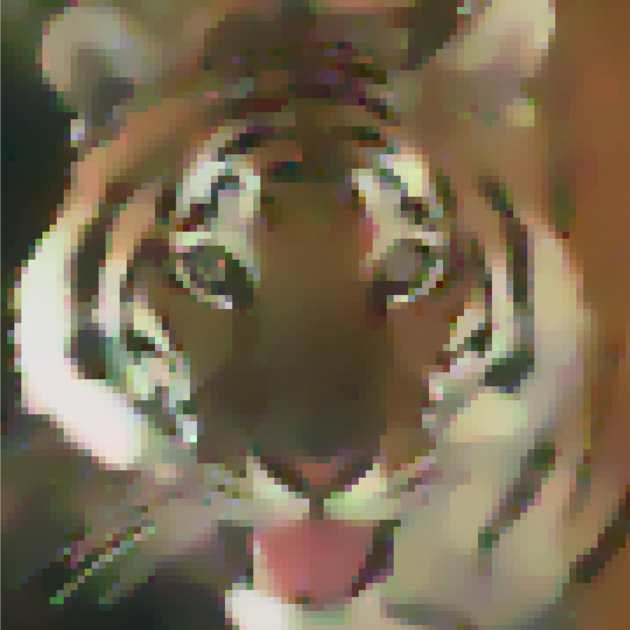}&
\includegraphics[width=4.2cm,height=3.5cm]{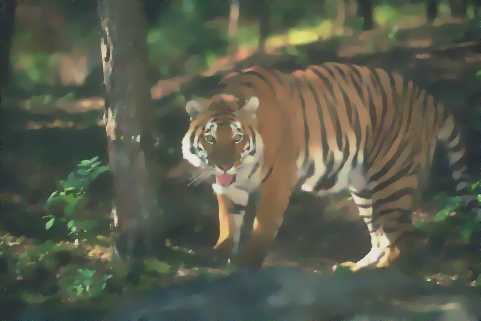}
\includegraphics[width=3.5cm,height=3.5cm]{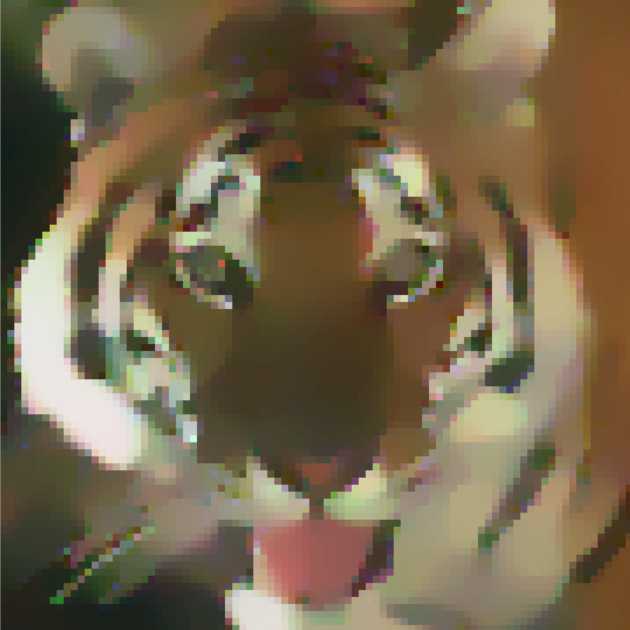}\\
\mbox{\scriptsize{(d) Proposed ($\mu\gets$constant, $\lambda\gets$adaptive ) }} &
\mbox{\scriptsize{(e) Proposed ($\mu\gets$constant, $\lambda\gets$adaptive) }} \\
\mbox{\scriptsize{$(\mbox{PSNR},\mbox{MSSIM})=(27.6574,0.7285)$}} &
\mbox{\scriptsize{$(\mbox{PSNR},\mbox{MSSIM})=(26.8005,0.6849)$}}\\

\includegraphics[width=4.2cm,height=3.5cm]{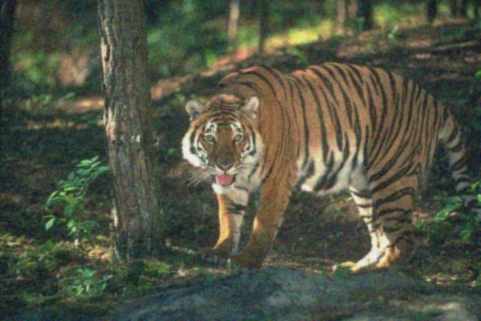}
\includegraphics[width=3.5cm,height=3.5cm]{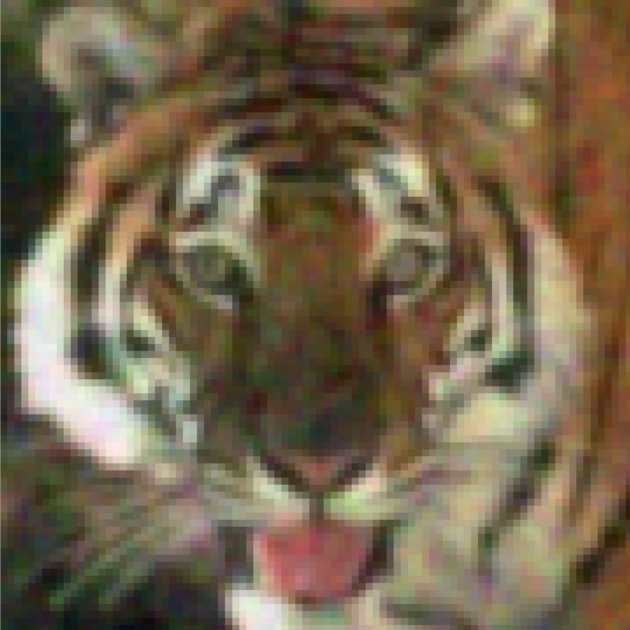}&
\includegraphics[width=4.2cm,height=3.5cm]{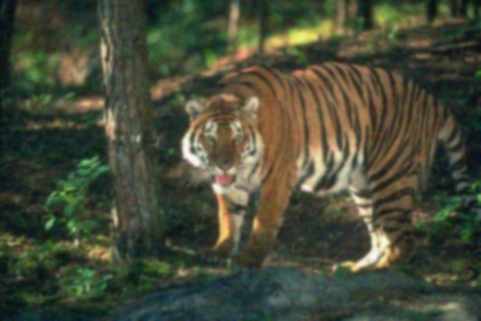}
\includegraphics[width=3.5cm,height=3.5cm]{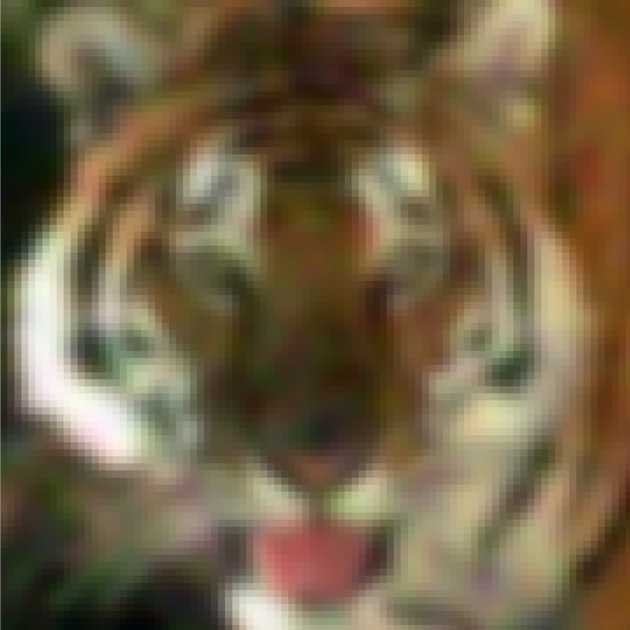}\\
\mbox{\scriptsize{(f) ~\cite{PVorotnikov12} (Adaptive)}} &
\mbox{\scriptsize{(g) ~\cite{PVorotnikov12}} (Adaptive)} \\
\mbox{\scriptsize{$(\mbox{PSNR},\mbox{MSSIM})=(26.1172,0.7540)$}} &
\mbox{\scriptsize{$(\mbox{PSNR},\mbox{MSSIM})=(24.8883,0.7057)$}}\\
 \end{array}\]
	\caption{(Color online) Better edge preserving image restoration results were obtained using our scheme in comparison with the original coupled PDE model~\cite{PVorotnikov12}.
The stopping parameters were chosen according to maximum MSSIM values, see Table~\ref{table:psnr} for the corresponding values.}\label{fig:CTE_Vs_Surya_Berkeley_1}
\end{figure*}
\begin{figure*}[t]
\centering
\[\begin{array}{cc}
\includegraphics[width=3.8cm,height=4.cm]{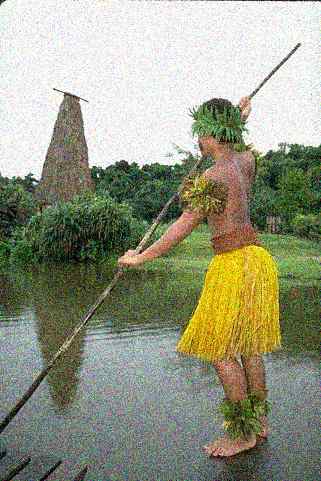}
\includegraphics[width=3.8cm,height=4.cm]{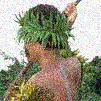}
&
\includegraphics[width=3.8cm,height=4.cm]{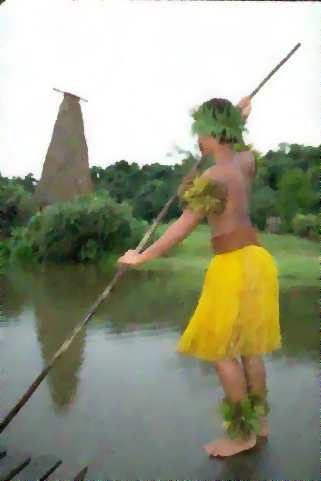}
\includegraphics[width=3.8cm,height=4.cm]{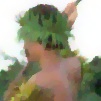}
\\
\mbox{\scriptsize{(a) Input noise ($\sigma=30$), $(\mbox{PSNR},\mbox{MSSIM})=(21.888,0.2445)$}} &
\mbox{\scriptsize{(b) Proposed ($\mu\gets$adaptive, $\lambda\gets$constant) }} \\
&
\mbox{\scriptsize{$(\mbox{PSNR},\mbox{MSSIM})=(28.3610,0.7718)$}}
\\ 
\includegraphics[width=3.8cm,height=4.cm]{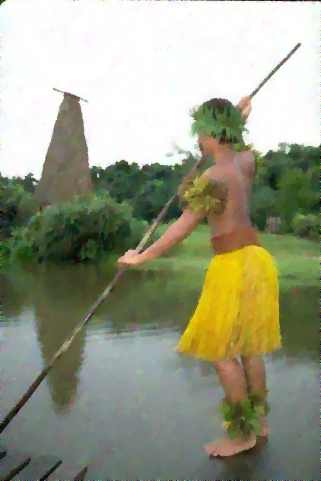}
\includegraphics[width=3.8cm,height=4.cm]{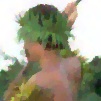}
&
\includegraphics[width=3.8cm,height=4.cm]{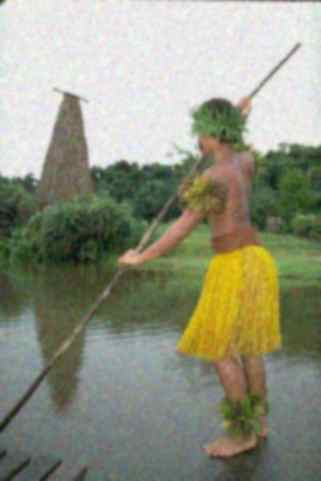}
\includegraphics[width=3.8cm,height=4.cm]{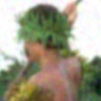}\\
\mbox{\scriptsize{(c) Proposed ($\mu\gets$constant, $\lambda\gets$adaptive) }} &
\mbox{\scriptsize{(d)~\cite{PVorotnikov12} ($\mu\gets$constant, $\lambda\gets$adaptive)}} \\
\mbox{\scriptsize{$(\mbox{PSNR},\mbox{MSSIM})=(29.2489,0.7636)$}} &
\mbox{\scriptsize{$(\mbox{PSNR},\mbox{MSSIM})=(26.4529,0.6849)$}}\\
\end{array}\]
\caption{(Color online) Better edge preserving image restoration results were obtained using our scheme in comparison with the original coupled PDE model~\cite{PVorotnikov12}.
The stopping parameters were chosen according to maximum MSSIM values, see Table~\ref{table:psnr} for the corresponding values.}\label{fig:CTE_Vs_Surya_Berkeley_2}
\end{figure*}
\begin{table*}
\centering
\scriptsize{
\begin{tabular}{ccccccccc}
\hline
\multirow{2}{*}{Method} & $\mu\gets$constant &$\mu\gets$adaptive & Figure & Noise&  Stopping& Convergence &  \multirow{2}{*}{PSNR (dB)/MSSIM}\\
&$\lambda\gets$ adaptive&$\lambda\gets$ constant&Example&Level&Parameter&Time(s)&\\
\hline
\multirow{3}{*}{CTE (Proposed)} &&\multirow{3}{*}{$\surd$}&~\ref{fig:CTE_Vs_Surya_Berkeley_1} (b)&$\sigma=20$ & $\epsilon=10^{-4}$ & 30.47 & 28.7938/0.7723\\
                                               &&&~\ref{fig:CTE_Vs_Surya_Berkeley_1} (c)&$\sigma=20$ & $\epsilon=10^{-6}$ & 38.71& 27.3510/0.7147\\
                                                &&&~\ref{fig:CTE_Vs_Surya_Berkeley_2} (b)&$\sigma=30$ & $\epsilon=10^{-6}$ & 38.61& 26.5238/0.6774\\

\hline
\multirow{3}{*}{CTE (Proposed)} &\multirow{3}{*}{$\surd$}&&~\ref{fig:CTE_Vs_Surya_Berkeley_1} (d)&$\sigma=20$& $\epsilon=10^{-4}$ & 30.96 & 28.2464/0.7237\\
                                               &&&~\ref{fig:CTE_Vs_Surya_Berkeley_1} (e)&$\sigma=20$ & $\epsilon=10^{-6}$ & 32.20& 26.1983/0.6796\\
                                                &&&~\ref{fig:CTE_Vs_Surya_Berkeley_2} (c)&$\sigma=30$ & $\epsilon=10^{-6}$ & 32.33& 26.7451/0.6792\\
\hline
\multirow{3}{*}{\cite{PVorotnikov12}} &\multirow{3}{*}{$\surd$}&&~
\ref{fig:CTE_Vs_Surya_Berkeley_1} (f)&$\sigma=20$ & $\epsilon=10^{-4}$ & 5.07  & 27.3646/0.7178\\
                                                 &&&~  \ref{fig:CTE_Vs_Surya_Berkeley_1} (g)&$\sigma=20$ & $\epsilon=10^{-6}$ &  40.57 & 25.4476/0.6970\\                                                   &&&~\ref{fig:CTE_Vs_Surya_Berkeley_2} (d)&$\sigma=30$ & $\epsilon=10^{-6}$ & 40.30  & 25.0559/0.6543\\
\hline
\end{tabular}
}
\caption{Image denoising error metrics (average) comparison using original coupled PDE scheme~\cite{PVorotnikov12} for different noise levels and parameters on the Berkeley segmentation dataset (BSDS) $500$. Some examples corresponding to the  entries are shown above in Figure~\ref{fig:CTE_Vs_Surya_Berkeley_1} and Figure~\ref{fig:CTE_Vs_Surya_Berkeley_2}.}\label{table:psnr}
\end{table*}


\subsubsection{Comparison of adaptive fidelity parameters}\label{sssec:fidelity}

\begin{figure*}
\centering
\[\begin{array}{cc}
\begin{minipage}[l]{5.0cm}
\includegraphics[width=3.4cm,height=3cm]{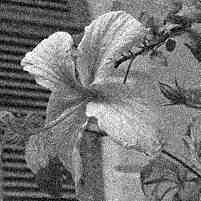}
\end{minipage}
&
\begin{minipage}[l]{5cm}
\includegraphics[width=6.5cm,height=5cm]{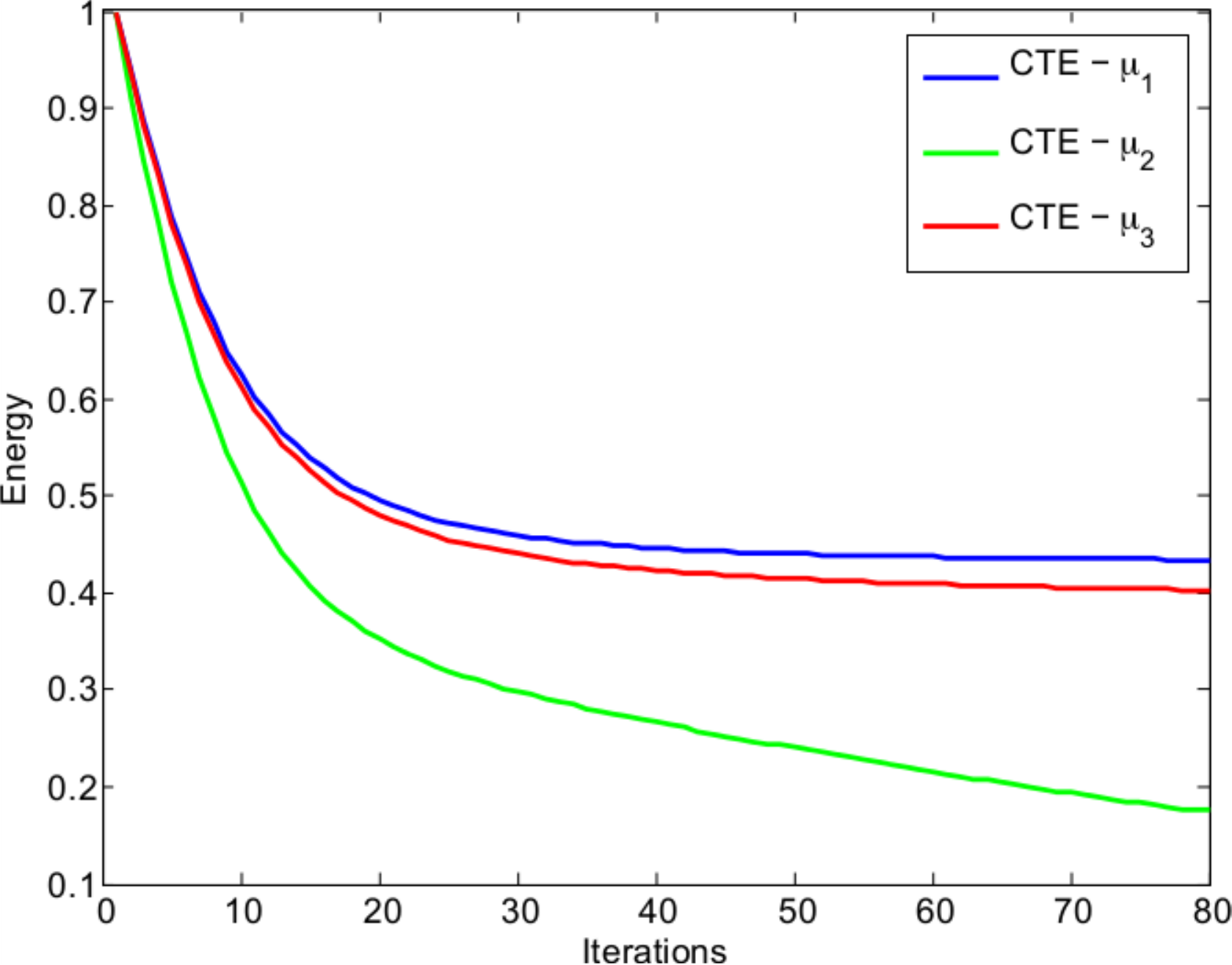}
\end{minipage}
\\
\hspace{-2cm}\mbox{\footnotesize{(a) Input noise ($\sigma=20$)}}
&
\hspace{2cm} \mbox{\footnotesize{(b) Energy vs Iteration for different $\mu$ choices}}
\end{array}\]
\[\begin{array}{c}
\mbox{\scriptsize{(b) $u,v,w,f-u$ components for our proposed scheme using $\mu_{1}$ with \mbox{$(\mbox{PSNR},\mbox{MSSIM})=(27.6623,0.8143)$}}}
\\
\includegraphics[width=3.9cm,height=3.5cm]{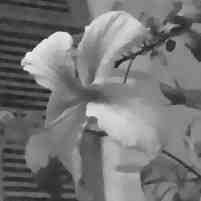}
\includegraphics[width=3.9cm,height=3.5cm]{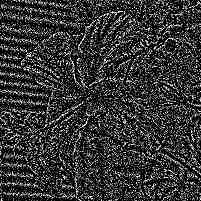}
\includegraphics[width=3.9cm,height=3.5cm]{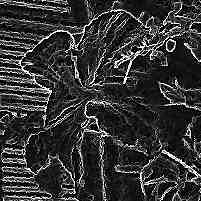}
\includegraphics[width=3.9cm,height=3.5cm]{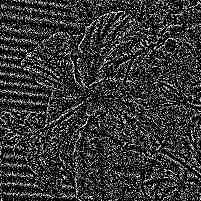}
\\
\mbox{\scriptsize{(c) $u,v,w,f-u$ components for our proposed scheme using  $\mu_{2}$ with \mbox{$(\mbox{PSNR},\mbox{MSSIM})=(24.6084,0.6823)$}}}
\\
\includegraphics[width=3.9cm,height=3.5cm]{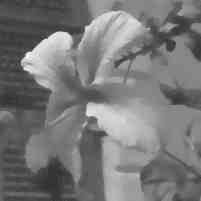}
\includegraphics[width=3.9cm,height=3.5cm]{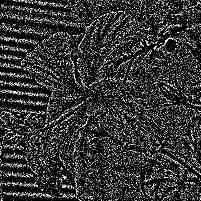}
\includegraphics[width=3.9cm,height=3.5cm]{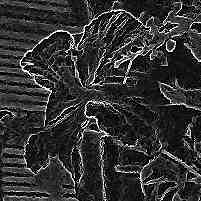}
\includegraphics[width=3.9cm,height=3.5cm]{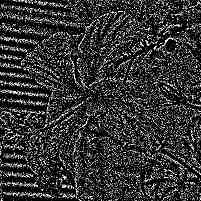}
\\
\mbox{\scriptsize{(d) $u,v,w,f-u$ components for our proposed scheme using $\mu_{3}$ with \mbox{$(\mbox{PSNR},\mbox{MSSIM})=(25.6414,0.7503)$}}}
\\
\includegraphics[width=3.9cm,height=3.5cm]{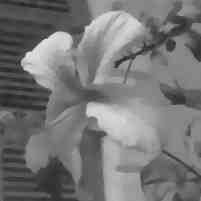}
\includegraphics[width=3.9cm,height=3.5cm]{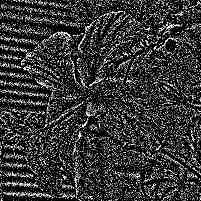}
\includegraphics[width=3.9cm,height=3.5cm]{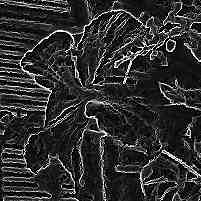}
\includegraphics[width=3.9cm,height=3.5cm]{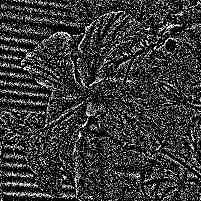}\\
\end{array}\]
\caption{Comparison of different $\mu$ functions computed using a real image corrupted with Gaussian noise level $\sigma=20$ and stopping parameter $\epsilon=10^{-4}$.
	(a) Original image.
	(b) Energy versus iteration for different adaptive $\mu$ functions based energy minimization scheme~(\ref{E:wTV}-\ref{E:constraint}).
	(c) $\mu_1$ based on local histograms  Eqn.~\eqref{E:CTElambda}.
	(d) $\mu_2$ based on the work of~\cite{PVorotnikov12} Eqn.~\eqref{E:VBSlambda}.
	(e) $\mu_3$ base on the work of~\cite{LMVese10} Eqn.~\eqref{E:IPOLlambda}.}\label{fig:mu_comp}
\end{figure*}

Figure~\ref{fig:mu_comp} gives a comparison for the implementation of our CTE model ($\mu\gets$adaptive, $\lambda\gets$constant) by using adaptive fidelity parameters $\mu_{1}$ (proposed), $\mu_{2}$ (based on the work of~\cite{PVorotnikov12}) and $\mu_{3}$ (studied in~\cite{LMVese10}) for a close-up $flower$ image corrupted with Gaussian noise level $\sigma=20$. As shown different decomposition levels based on smooth + random noise + edges components are given according to the adaptive parameter $\mu$. In terms of PSNR and MSSIM error metrics the $\mu_{1}$ choice improves better the denoising result. Note the similar convergence property of the energy value between our $\mu_{1}$ choice and the $\mu_{3}$ adaptive parameter.

\subsubsection{Brain MRI image decomposition}\label{sssec:brain}

\begin{figure*}
\centering
\includegraphics[width=3.2cm,height=3.2cm]{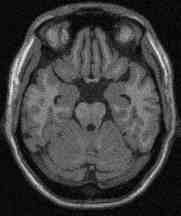}
\includegraphics[width=3.2cm,height=3.2cm]{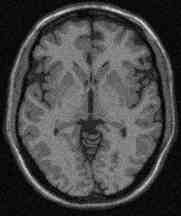}
\includegraphics[width=3.2cm,height=3.2cm]{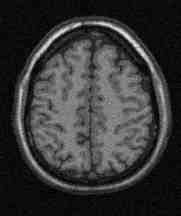}

\subfigure[Noisy brain MRI images]
{\includegraphics[width=4.cm,height=4.cm]{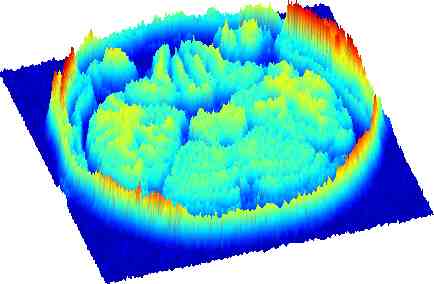}
\includegraphics[width=4.cm,height=4.cm]{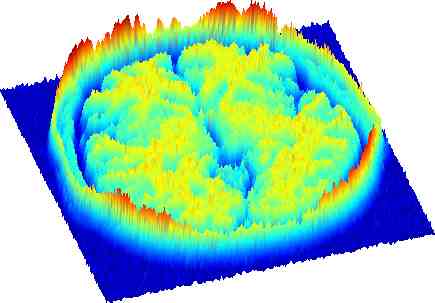}
\includegraphics[width=4.cm,height=4.cm]{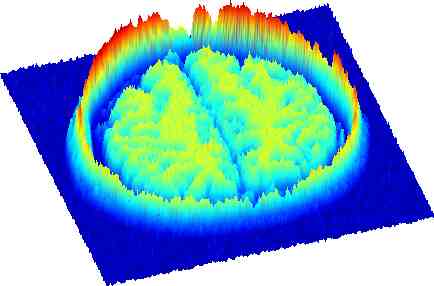}}\\

\includegraphics[width=3.2cm,height=3.2cm]{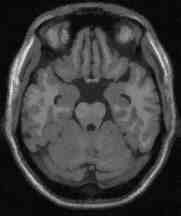}
\includegraphics[width=3.2cm,height=3.2cm]{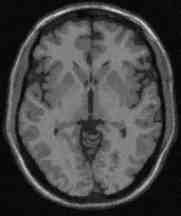}
\includegraphics[width=3.2cm,height=3.2cm]{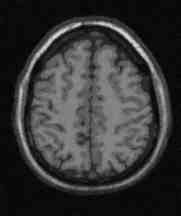}

\subfigure[Cartoon components ($u$)]{
\includegraphics[width=4.3cm,height=4.3cm]{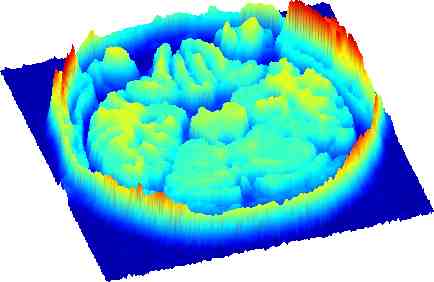}
\includegraphics[width=4.3cm,height=4.3cm]{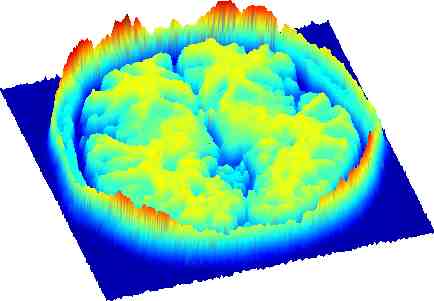}
\includegraphics[width=4.3cm,height=4.3cm]{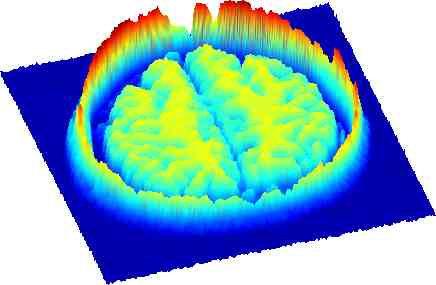}}

\subfigure[Edges components ($w$)]{
\includegraphics[width=3.2cm,height=3.2cm]{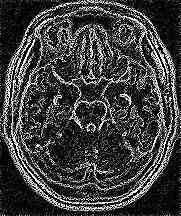}
\includegraphics[width=3.2cm,height=3.2cm]{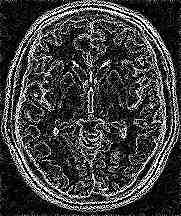}
\includegraphics[width=3.2cm,height=3.2cm]{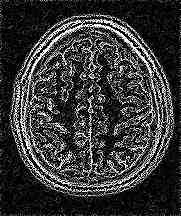}}
\caption{(Color online) $u$ \& $w$ components using our proposed model with  ($\mu\gets$adaptive, $\lambda\gets$constant) and stopping parameter $\epsilon=10^{-4}$ for different Brain MR images (slice: $50$, $70$ and $120$ respectively).} 	
\label{fig:brain}
\end{figure*}

Figure~\ref{fig:brain} shows input Brain MR images from the Simulated Brain Database and its corresponding $(u,v,w)$ functions for slices $50$, $70$ and $120$ together with their corresponding inputs. Lack of sharp edges in the denoised Brain MRI images (Third-Fourth rows) can be attributed to the spatial relaxation based for the coupling PDE scheme. Spatial smoothing based regularization introduces a slight blur on each edge map $w$ (Fifth row). Although, we gradually reduce the smoothing results by varying the stopping parameter $\epsilon$ the strong discontinuities are preserved well and noise is removed effectively within regions.

\subsection{Multi-scale decomposition}\label{ssec:mult}
\begin{figure*}
\[\begin{array}{c}
\mbox{\scriptsize{Multiscale cartoon and texture extraction using our proposed CTE method}}
\end{array}\]
\[\begin{array}{ccccc}
\includegraphics[width=2.75cm,height=2.75cm]{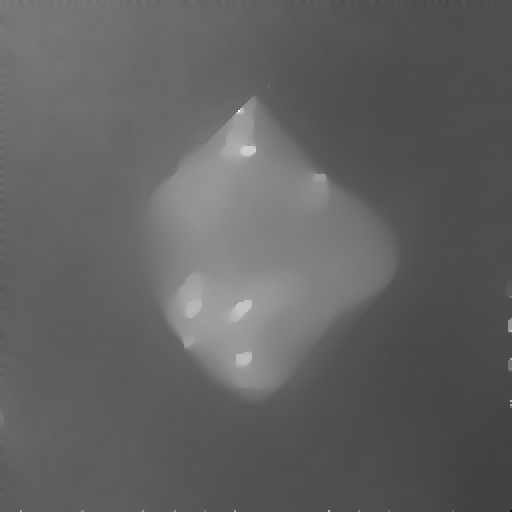}&  
\includegraphics[width=2.75cm,height=2.75cm]{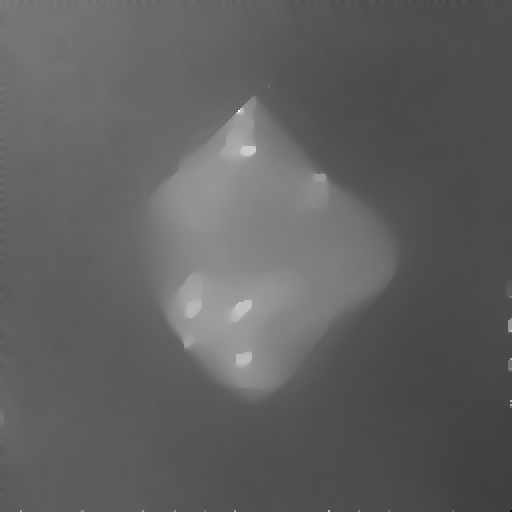}& 
\includegraphics[width=2.75cm,height=2.75cm]{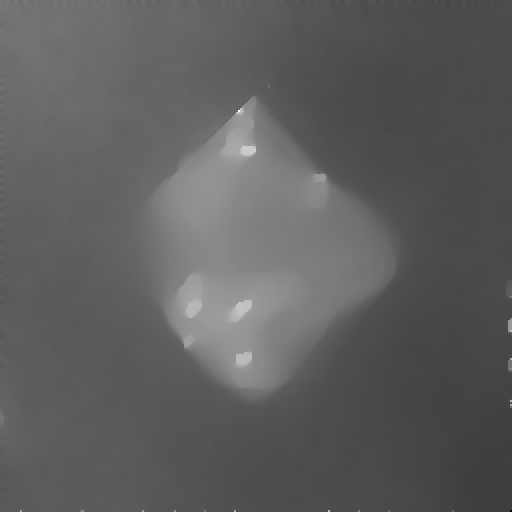}&
\includegraphics[width=2.75cm,height=2.75cm]{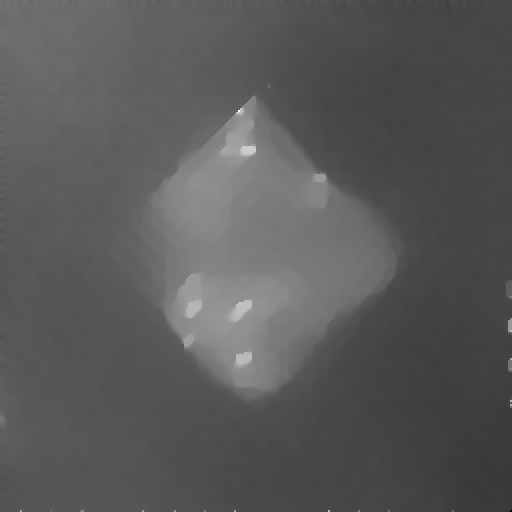}&
\includegraphics[width=2.75cm,height=2.75cm]{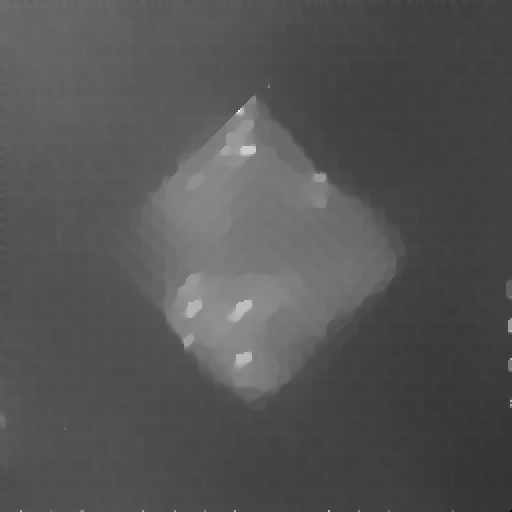}\\
\mbox{\scriptsize{(a) $u_{0}$}} &
\mbox{\scriptsize{(b) $\sum_{i=0}^{1}u_{i}$}} &
\mbox{\scriptsize{(c) $\sum_{i=0}^{2}u_{i}$}} &
\mbox{\scriptsize{(d) $\sum_{i=0}^{3}u_{i}$}} &
\mbox{\scriptsize{(f) $\sum_{i=0}^{4}u_{i}$}} \\
\includegraphics[width=2.75cm,height=2.75cm]{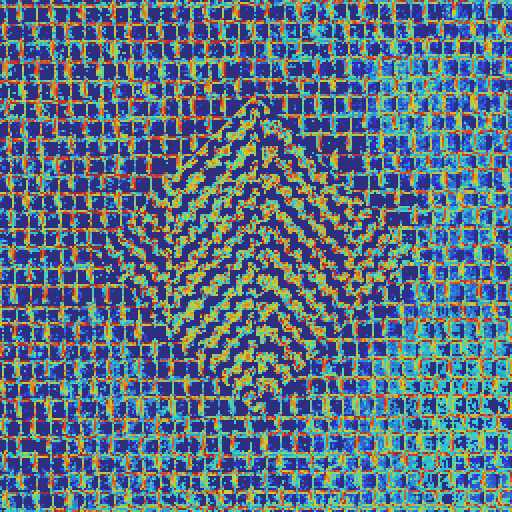}&  
\includegraphics[width=2.75cm,height=2.75cm]{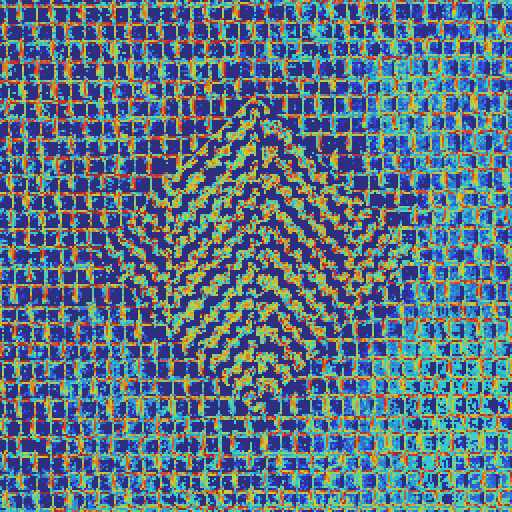}& 
\includegraphics[width=2.75cm,height=2.75cm]{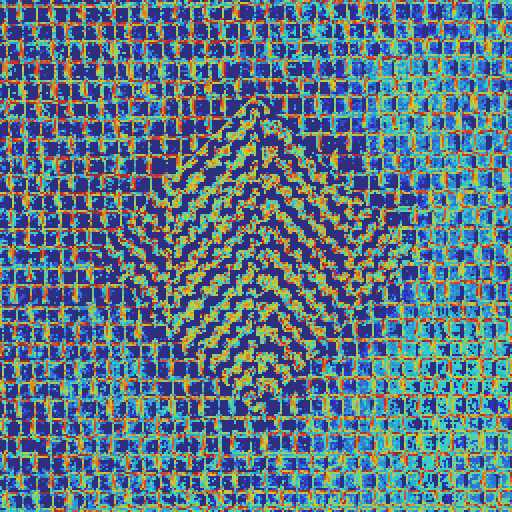}&
\includegraphics[width=2.75cm,height=2.75cm]{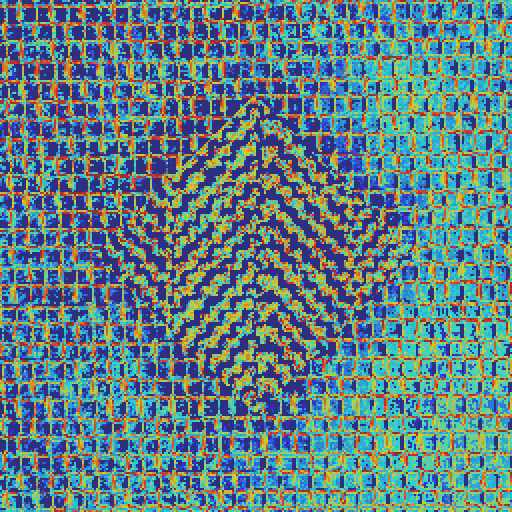}&
\includegraphics[width=2.75cm,height=2.75cm]{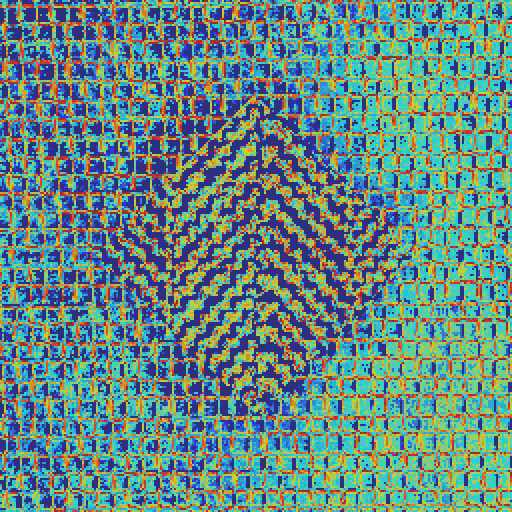}\\
\mbox{\scriptsize{(a) $v_{0}$}} &
\mbox{\scriptsize{(b) $\sum_{i=0}^{1}v_{i}$}} &
\mbox{\scriptsize{(c) $\sum_{i=0}^{2}v_{i}$}} &
\mbox{\scriptsize{(d) $\sum_{i=0}^{3}v_{i}$}} &
\mbox{\scriptsize{(f) $\sum_{i=0}^{4}v_{i}$}} \\
\end{array}\]
\[\begin{array}{c}
\mbox{\scriptsize{Multiscale cartoon and texture extraction using the hierarchical decomposition of Tang and He~\cite{THe13}}}
\end{array}\]
\[\begin{array}{ccccc}
\includegraphics[width=2.75cm,height=2.75cm]{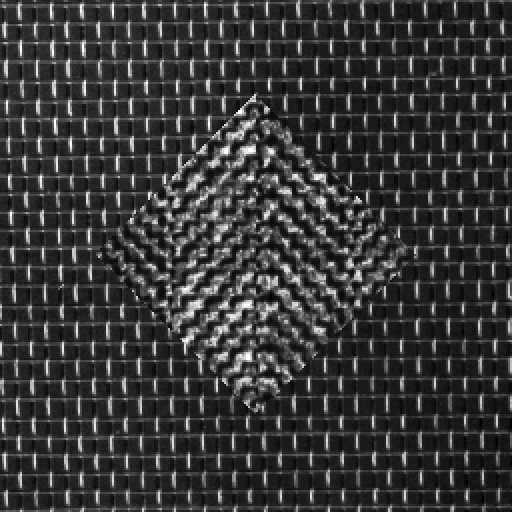}&  
\includegraphics[width=2.75cm,height=2.75cm]{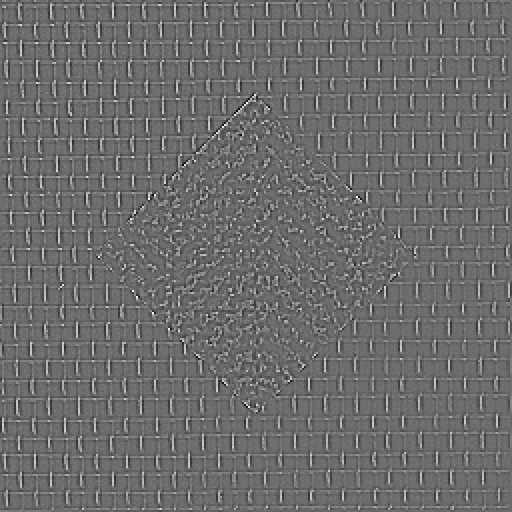}& 
\includegraphics[width=2.75cm,height=2.75cm]{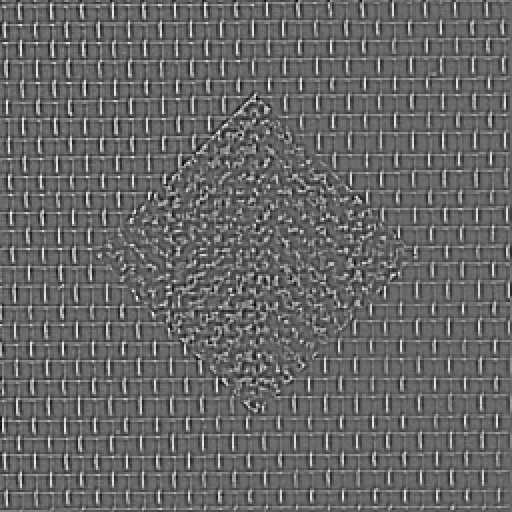}&
\includegraphics[width=2.75cm,height=2.75cm]{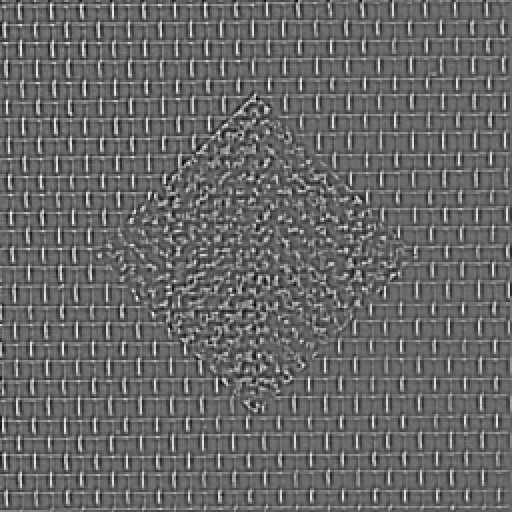}&
\includegraphics[width=2.75cm,height=2.75cm]{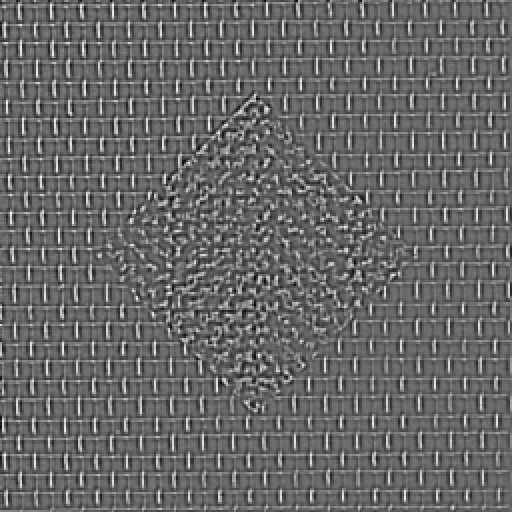}\\
\mbox{\scriptsize{(a) $u_{0}$}} &
\mbox{\scriptsize{(b) $\sum_{i=0}^{1}u_{i}$}} &
\mbox{\scriptsize{(c) $\sum_{i=0}^{2}u_{i}$}} &
\mbox{\scriptsize{(d) $\sum_{i=0}^{3}u_{i}$}} &
\mbox{\scriptsize{(f) $\sum_{i=0}^{4}u_{i}$}} \\
\includegraphics[width=2.75cm,height=2.75cm]{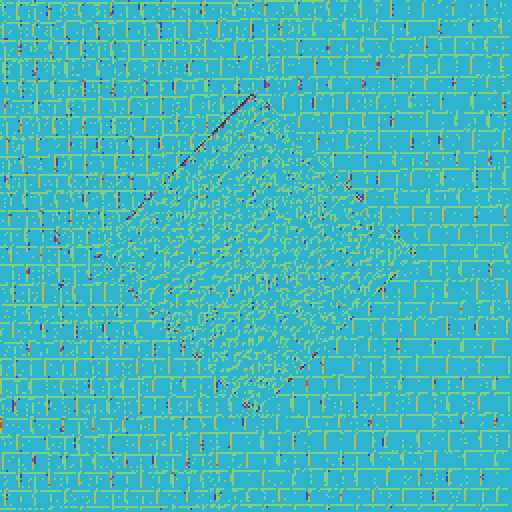}&  
\includegraphics[width=2.75cm,height=2.75cm]{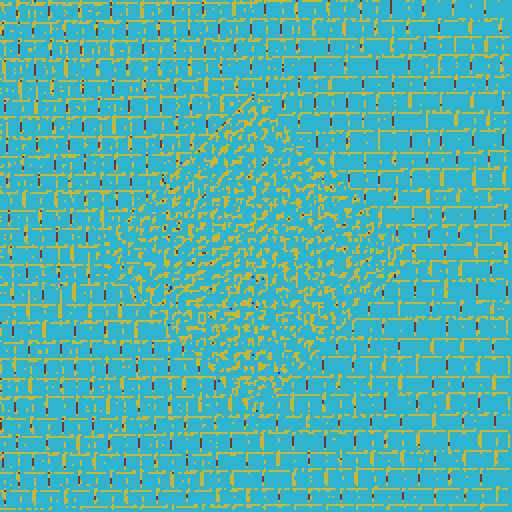}& 
\includegraphics[width=2.75cm,height=2.75cm]{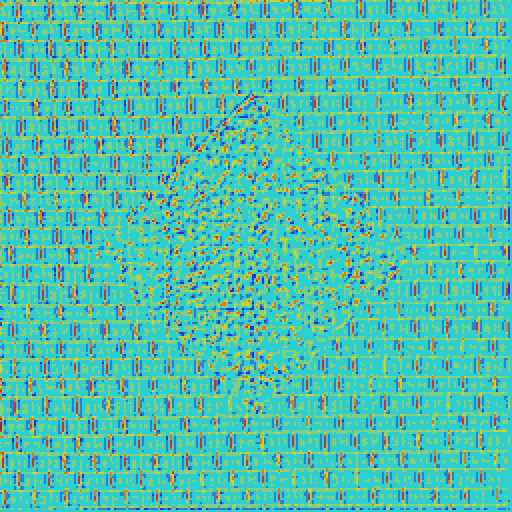}&
\includegraphics[width=2.75cm,height=2.75cm]{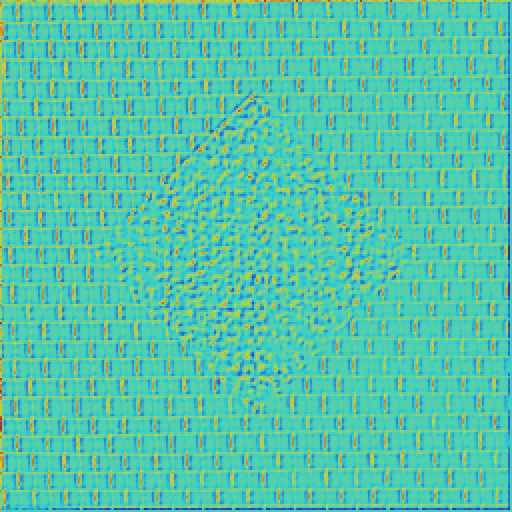}&
\includegraphics[width=2.75cm,height=2.75cm]{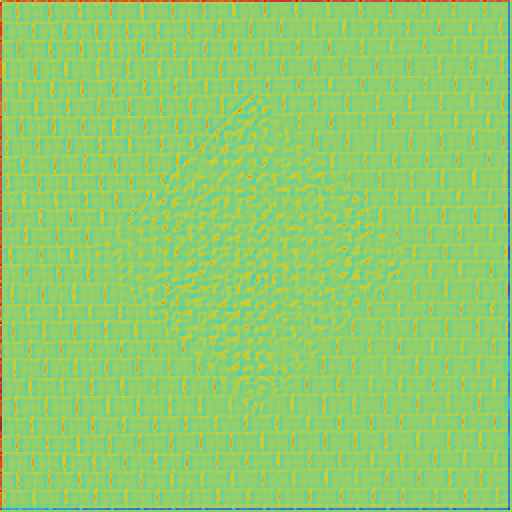}\\
\mbox{\scriptsize{(a) $v_{0}$}} &
\mbox{\scriptsize{(b) $\sum_{i=0}^{1}v_{i}$}} &
\mbox{\scriptsize{(c) $\sum_{i=0}^{2}v_{i}$}} &
\mbox{\scriptsize{(d) $\sum_{i=0}^{3}v_{i}$}} &
\mbox{\scriptsize{(f) $\sum_{i=0}^{4}v_{i}$}} \\
\end{array}\]
\caption{(Color online) Multiscale cartoon and texture decomposition of a synthetic image for $5$ steps by implementing our proposed approach (top two rows) (stooping parameter $\epsilon=10^{-4}$) and the hierarchical method from~\cite{THe13} (bottom two rows).}	
\label{fig:multiscale}
\end{figure*}
Following~\cite{THe13} we can make the weighted TV model with multi-scale parameter $\lambda$. Note that this is slightly different from the original multi scale usage in~\cite{TadmorNezzar04}, here we use it in the constraint Eqn.~\eqref{E:constraint}. Let us briefly recall the model proposed in~\cite{THe13} where the texture component $v$ is modeled using $G_p$ norm. That is, the minimization is carried out for both $u,v$,
\begin{eqnarray*}
\inf_{u,v}\Bigg\{E(f,\lambda;u,v)
=|u|_{BV(\Omega)}+\mu\|f-u-v\|^{2}_{L^{2}(\Omega)}+\lambda\|v\|_{G_{p}(\Omega)}\Bigg\},
\end{eqnarray*}
with $G_{p}$ consisting of all distributions which can be written as, 
\begin{equation*}
v=\partial_{x}g_{1}+\partial_{y}g_{2}=div(\vec{g}),\quad \mbox{$\vec{g}\in L^{p}(\Omega,\mathbb{R}^{2})$}.
\end{equation*}
The $G_p$ norm is defined as,
\begin{equation*}
\|v\|_{G_{p}(\Omega)}=\inf\left\{\|\vec{g}\|_{L^{p}(\Omega)}\,|\, v=div(\vec{g}),\,\vec{g}\in L^{p}(\Omega,\mathbb{R}^{2})\right\}.
\end{equation*}
We utilize the same modeling for the texture component $v$ in our splitting step of the proposed weighted TV model (see Eqn.~\eqref{E:bresuv}),
\begin{eqnarray}\label{E:wTVmulti}
\min\limits_{u,v}\Bigg\{\int_{\Omega}|\nabla u|\,dx  
+\mu\|f-u-v\|^{2}_{L^{2}(\Omega)}+\lambda\|v\|_{G_{p}(\Omega)}\Bigg\}
\end{eqnarray}
Finally, we compare multiscale version~\eqref{E:wTVmulti} with the multi-scale TV decomposition results of~\cite{THe13}. 
Figure~\ref{fig:multiscale} shows the comparison result on a synthetic image for $5$ steps and our scheme retains the cartoon component clearly than the Tang and He~\cite{THe13}. Moreover, the texture components show a progressive capture of small scale oscillations.

\section{Conclusion}\label{sec:disc}

We have presented a new image decomposition model coupling a variational and PDE via a weighted total variation regularization algorithm (CTE model). Our main contribution is twofold: 
\begin{itemize}
\item[1)] The proposed decomposition model gets an image decomposition into its cartoon, texture and edge components with fixed, adaptive and multicale parameters for the $L^{1}$-fitting term by utilizing a local histogram along with a diffusion equation. Extensive experiments using  a fast dual minimization based splitting implementation indicates that the proposed scheme is useful for edge preserving image decomposition on real and noisy images. Comparative results indicate the proposed scheme is useful in denoising natural images as well in multi scale image decomposition. 
\item[2)] We fashioned a new well posed scheme to transform the non-linear problem to a set of subproblems that are much easier to solve quickly via an optimization technique and a linear diffusion PDE solution.
\end{itemize}

It is well known that there is no unique decomposition of an image into three scales: cartoon, texture and edges. At close range, texture and edges may be just a set of well-structured objects that can be kept in the cartoon part or the textural part according to a scale (equivalently iteration) parameter.  We proposed an adaptive choice computed on local windows based histogram information as well as a multi-scale adaptation weight parameter for the $L^{1}$-fitting term, allowing us to get satisfactory results.  There are other adaptive parameter choices for constructing a general decomposition model and we believe the proposed method is general in the sense that other regularizers (instead of TV) can be utilized in computing the cartoon component. We also remark that our denoising results are proof-of-concept for the proposed CTE model and we do not claim it outperforms state-of-the-art methods such as the nonlocal means~\cite{BuadesColl06} or BM3D~\cite{DEgiazarian07} which are specifically designed for optimal denoising results. 

\bibliographystyle{plain}
\bibliography{texdecomp}  
\end{document}